\documentclass[lettersize,journal]{IEEEtran}
\usepackage{amsmath,amsfonts}

\usepackage{algorithmic}
\usepackage{algorithm}
\usepackage{array}
\usepackage[caption=false,font=normalsize,labelfont=sf,textfont=sf]{subfig}
\usepackage{textcomp}
\usepackage{stfloats}
\usepackage{url}
\usepackage{verbatim}
\usepackage{graphicx}
\usepackage{cite}
\usepackage{xcolor}         %

\usepackage{amsmath,amsfonts,bm}

\def\eqref#1{equation~\ref{#1}}

\def\1{\bm{1}}

\DeclareMathAlphabet{\mathsfit}{\encodingdefault}{\sfdefault}{m}{sl}
\SetMathAlphabet{\mathsfit}{bold}{\encodingdefault}{\sfdefault}{bx}{n}

\usepackage{amsthm}

\newcommand{\reals}{\mathbb{R}}
\newcommand{\I}{\mathcal{I}}
\newcommand{\J}{\mathcal{J}}

\newcommand{\T}{\mathcal{T}}
\newcommand{\Q}{\mathcal{Q}}

\newcommand{\cl}{\upsilon}

\newcommand{\we}{w}

\usepackage{hyperref}       %
\usepackage{url}            %
\usepackage{booktabs}       %
\usepackage{xcolor}         %
\usepackage{graphicx,epsfig,tikz}
\usepackage{amsmath}
\usepackage{dsfont}

\usepackage{bbm}

\newcommand{\red}[1]{\textcolor{black}{#1}}

\usepackage{algorithm}

\usepackage{tikz}

\usepackage[linesnumbered,ruled,vlined,algo2e]{algorithm2e}
\SetKwInput{kwInit}{Init}
\SetKwComment{Comment}{$\triangleright$\ }{}

\newtheorem{theorem}{Theorem}[section]
\newtheorem{lemma}[theorem]{Lemma}
\newtheorem{remark}[theorem]{Remark}

\newtheorem{assumption}[theorem]{Assumption}
\newtheorem{corollary}[theorem]{Corollary}
\newtheorem{definition}[theorem]{Definition}
\hypersetup{hidelinks}
\begin{document}

\title{Learning to Schedule in Parallel-Server Queues\\ with Stochastic Bilinear Rewards}

\author{Jung-hun Kim and Milan Vojnovi{\' c}        %
\thanks{Jung-hun Kim is with CREST, ENSAE Paris, France (email: junghun.kim@ensae.fr) }
\thanks{Milan Vojnovi{\' c} is with London School of Economics, United Kingdom (email: m.vojnovic@lse.ac.uk)}}

\maketitle

\begin{abstract}
We consider the problem of scheduling in multi-class, parallel-server queuing systems with uncertain rewards from job-server assignments. In this scenario, jobs incur holding costs while awaiting completion, and job-server assignments yield observable stochastic rewards with unknown mean values. The mean rewards for job-server assignments are assumed to follow a bilinear model with respect to features that characterize jobs and servers. Our objective is to minimize regret by maximizing the cumulative reward of job-server assignments over a time horizon, while keeping the total job holding cost bounded to ensure the stability of the queueing system. This problem is motivated by applications requiring resource allocation in network systems.

A central challenge is to control the tradeoff between reward maximization and fair allocation for the stability of the underlying queuing system (i.e., maximizing network throughput). To address this challenge, we propose a scheduling algorithm based on a weighted proportional fair criteria augmented with marginal costs for reward maximization, incorporating a bandit algorithm tailored for bilinear rewards. {
Our algorithm admits a regret--queue length tradeoff. For any fixed control parameter \(V>0\), it ensures a uniform expected queue length and time-average holding-cost bounds. For a target horizon \(T\), choosing \(V_T=\Theta(\sqrt{IT})\) at initialization yields \(\widetilde O((\sqrt I+d^2)\sqrt T+1/\delta)\) regret. Under this regret-optimized tuning, the corresponding expected queue length and time-average holding-cost bounds remain uniform over the execution time and scales as \(O(\sqrt{IT}+1/\delta)\) and \(O(\sqrt{IT}/\delta)\), respectively.}

\end{abstract}

\begin{IEEEkeywords}
Resource Allocation, Scheduling Jobs, Queuing,  Reward Maximization, Stability, Online Learning.
\end{IEEEkeywords}

\section{Introduction}

In this work, we address the problem of scheduling jobs in multi-class, parallel-server queuing systems—such as those found in data centers, edge computing infrastructures, and communication networks. In such systems, both flow types (jobs) and processing units (servers) can have heterogeneous characteristics, necessitating differentiated services. Assigning a job to a server or network function yields an observable stochastic reward—e.g., processing rate, or an application-specific quality of the job output dependent on the server assignment—with an unknown mean value that depends on the compatibility between job and server characteristics. Note that considering rewards of assignments accommodates assignment costs, treating them as negative rewards.  

Specifically, we consider the case of noisy rewards, where the rewards for job-server assignments follow a bilinear model based on the features characterizing jobs and servers. This reward model can capture complex interactions between job and server characteristics and can be leveraged to make effective scheduling decisions in uncertain environments, as demonstrated in our work. The scheduler's objective is to maximize the expected cumulative reward over
a prescribed horizon while controlling the expected job holding cost.

This problem arises in a wide range of networking systems where resource allocation decisions must be made under uncertainty. For example, in data centers and distributed cloud computing systems \cite{delay,quincy}, computational jobs composed of multiple tasks must be assigned to servers with heterogeneous processing capabilities and varying data locality preferences. A common goal is to maximize system throughput when processing dynamic workloads--a challenging task further exacerbated by uncertainty or lack of knowledge about certain system parameters. These parameters need to be inferred from noisy observations while simultaneously making effective scheduling decisions. Similar network resource allocation challenges occur in edge computing networks \cite{luo2021resource}, where tasks offloaded from user devices must be matched to nearby edge nodes under uncertain wireless conditions and variable server loads. 

Recently, such system optimization problems have gained renewed interest in the context of scheduling for Large Language Models (LLMs) \cite{LLM-LTR}. In these systems, multiple types of LLMs and diverse user query types (prompts) coexist, and assigning each query to an appropriate model is critical to maximize rewards—such as the quality of responses to user queries or query processing times—under unknown reward distributions. In these scenarios, it is essential to schedule queries in a way that balances reward maximization with system stability.

%

Learning to schedule in queuing systems has been studied due to its wide range of applications and the need to address unknown system parameters. Much of the existing research focuses on queuing system stability (i.e., bounded queue lengths) or other queuing-related performance objectives \cite{krishnasamy2016regret, KRJS21, stahlbuhk2021learning, freund2022efficient, sentenac2021decentralized, yang2023learning, huang2024lyapunov}. Recently, \cite{hsu2021integrated} considered the joint optimization problem in multi-class, parallel-server queuing systems, where the objective is to maximize rewards realized through job-server assignments while maintaining queuing system stability. They proposed an algorithm and showed that it achieves regret over a given time horizon that scales linearly with the time horizon. Notably, their work considers arbitrary mean rewards, not allowing the scheduler to leverage the observable information about the job and server features. 

In contrast, our work addresses the case where the unknown mean rewards follow a bilinear function of the job and server features, with unknown coefficient weights (parameters) that capture pairwise interactions between them. The key motivation behind our work is to propose a practical algorithm that leverages this feature information to achieve better regret scaling with respect to both the time horizon and the number of servers. Achieving this goal requires an algorithm that combines scheduling decisions with the inference of unknown reward parameters. Furthermore, we aim to bound the holding cost, which generalizes traditional queueing bounds studied in \cite{hsu2021integrated} by incorporating a weighted proportional fairness criterion.

\red{{The bilinear reward model provides a parsimonious representation of
compatibility between two types of entities, here job classes and server
classes, through their feature interactions. Such models have been used in
bilinear bandits \cite{jun2019bilinear} and related pairwise prediction
problems such as recommendation systems, inductive matrix completion, and
bioinformatics applications
\cite{jain2013provable,huang2019fibinet,natarajan2014inductive}. In our
setting, the entries of the matrix quantify how each job-side
feature interacts with each server-side feature. This captures effects such
as the compatibility between CPU demand and CPU capacity, memory demand and
memory capacity, data-locality features and server placement, or prompt/query
features and model-serving features. By sharing the same interaction matrix
across all job--server pairs, the bilinear structure allows observations from
one assignment to inform the reward estimates of other related assignments,
and is therefore more sample-efficient than an unstructured model that treats
all job--server rewards as unrelated parameters.}
}
In what follows, we provide a summary of our contributions.
\subsection{Summary of Contributions}

\begin{itemize}
    \item To address the tradeoff between reward maximization and queuing system stability, we propose a scheduling algorithm that dynamically allocates jobs to servers based on fair allocation with marginal costs. This is achieved by solving a system optimization problem that combines weighted proportional fair allocation criteria and the reward of job-server assignments. Importantly, the algorithm employs a bandit strategy for learning the unknown reward distribution.

   \item \red{We establish a regret--stability tradeoff for the proposed algorithm. For any fixed control parameter \(V>0\), the algorithm guarantees uniform expected total-queue stability: \[ \sup_{t\ge1} \mathbb E\!\left[\sum_{i\in\mathcal I}Q_i(t)\right] <\infty . \] It also guarantees uniform time-averaged expected holding-cost stability: \[ \sup_{t\ge1}\overline H_t(c)<\infty. \]  For a prescribed target horizon \(T\), choosing \(V_T=\Theta(\sqrt{IT})\) for $V$ before the algorithm starts and keeping it fixed throughout the run yields \[ R(T) = \widetilde O\!\left( (\sqrt I+d^2)\sqrt T+\frac{1}{\delta} \right), \] while the corresponding expected queue length and time-average holding-cost bounds remain uniform over the execution time and scales as \(O(\sqrt{IT}+1/\delta)\) and \(O(\sqrt{IT}/\delta)\), respectively.}
    \item Lastly, we present the results of numerical experiments to demonstrate the performance gains achieved by our algorithm over the best previously known algorithm. Additionally, we demonstrate that better mean holding costs can be achieved by using weighted proportional fair criteria. 
\end{itemize}

\subsection{Related Work}

Our work falls within the line of research on resource allocation optimization under uncertainty, with a particular focus on combining learning with optimization of resource allocation (e.g., \cite{cmu,capacity,NS19,LMS19,shah2020adaptive,KRJS21,JKK21,stahlbuhk2021learning,yang2023learning,zhao-fluct,Fu,hsu2021integrated,huang2024lyapunov,sentenac2021decentralized,freund2022efficient,yang2023learning,huang2024lyapunov}).
The reward maximization aspect of our scheduling problem formulation has connections with minimum-cost scheduling used in cluster computing systems, as discussed in \cite{quincy,firmament}, where cost-based scheduling is employed to prioritize the allocation of tasks near the data that needs processing, with fixed and known task-server allocation costs. In our work, we consider that allocation costs (or rewards) are assumed to be unknown a priori, and only noisy values are observed as tasks are assigned to servers. More importantly, the allocation policies we examine aim to maximize reward while minimizing job holding cost, ensuring queuing system stability with fair allocation.

Stability has been extensively studied in the context of network resource allocation, including research on Maxweight (Backpressure) policies \cite{tassiulas,mckeown,bramson}, proportional fair allocation \cite{kelly97,kelly,mas}, and other notions of fair allocations such as $\alpha$-fair allocations \cite{mo} and $(\alpha,g)$-switch policies \cite{walton2014concave}. 
Our work differs in that we consider allocation policies that require the estimation of unknown mean rewards and the reward maximization objective, in addition to ensuring queuing system stability. 

Some works have examined the queuing system control problem, concerned with achieving stability or minimizing total queue length under uncertain job service times, as explored in \cite{krishnasamy2016regret,KRJS21,sentenac2021decentralized,freund2022efficient,yang2023learning,huang2024lyapunov}, which is different from the reward maximization objective in our problem setting. Furthermore, we address job holding costs by considering weighted proportional fair allocations. There is work on learning proportional fair allocations \cite{propfair}. However, our work differs in several aspects: firstly, our objective combines fairness of allocation and the reward of assignments; secondly, we consider uncertainty in the reward of assignments; and thirdly, we address dynamic arrivals and departures of jobs. 
 
 The work most closely related to ours is on the queuing system control problem, as studied by \cite{hsu2021integrated}, which considered the case of unstructured rewards, making no use of job or server features. In contrast, we consider systems where the scheduler can leverage access to job or server features, under the assumption of structured rewards of job-server assignments, following a bilinear model. The structure of rewards enables us to design algorithms that extend learning to encompass job classes collectively, contrasting with the approach in \cite{hsu2021integrated} where learning is carried out for each job separately. 
As a result, we achieve a sublinear regret bound with respect to the time horizon. Moreover, our work accommodates more general, weighted proportional fairness criteria, and we derive bounds on the holding costs, allowing for discrimination between job classes.

\subsection{Organization of the Paper}

In Section~\ref{sec:pf}, we present the problem formulation and define the notation used throughout the paper. In Section~\ref{sec:alg}, we propose algorithms and analyze their regret and mean holding cost. In Section~\ref{sec:dist}, we introduce distributed iterative algorithms for computing allocations and their convergence analysis. In Section~\ref{sec:num}, we show the results of our numerical experiments. Concluding remarks are presented in Section~\ref{sec:conc}. Proofs of the theorems and additional results are included in the appendix.

\section{Problem Formulation}
\label{sec:pf}

This section introduces the resource allocation problem we address and outlines the criteria employed for performance evaluation.

\subsection{System Assumptions} We consider a multi-class, parallel-server queuing system, with $\mathcal{I} = \{1,\ldots, I\}$ and $\mathcal{J} = \{1,\ldots, J\}$ denoting the sets of job and server classes, respectively. Each server class $j\in \mathcal{J}$ has $n_j$ servers. Let $n$ denote the total number of servers, i.e., $n=\sum_{j\in \mathcal{J}}n_j$. Jobs of class $i\in \mathcal{I}$ arrive into the system at rate $\lambda_i$, have a mean service time of $1/\mu_i$, and are queued until served. Time is assumed to be discrete. Let $Q_i(t)$ denote the number of class-$i$ jobs in the system at time $t$ with $Q_i(0):=0$.

Each job in class $i\in \I$ has a feature vector $u_i\in \reals^{d_1}$ and each server in class $j\in \J$ has a feature vector $v_j\in \reals^{d_2}$, both of which are known to the scheduler. For simplicity, we assume $d_1=d_2=d$. Assigning a job of class $i$ to a server of class $j$ yields a stochastic reward observed by the scheduler. The mean reward $r_{i,j}$ for assigning a job of class $i$ to a server of class $j$ is assumed to be of the bilinear form: $r_{i,j} = u_i^\top\Theta v_j$, where $\Theta\in \mathbb{R}^{d\times d}$ is an unknown parameter matrix. At each time step, the scheduler assigns available jobs to the servers, with at most one job assigned per server. An illustration of the main components of the system is provided in Figure~\ref{fig:bilinear1}.

\begin{figure}[t]
\centering
\includegraphics[width=0.9\linewidth]{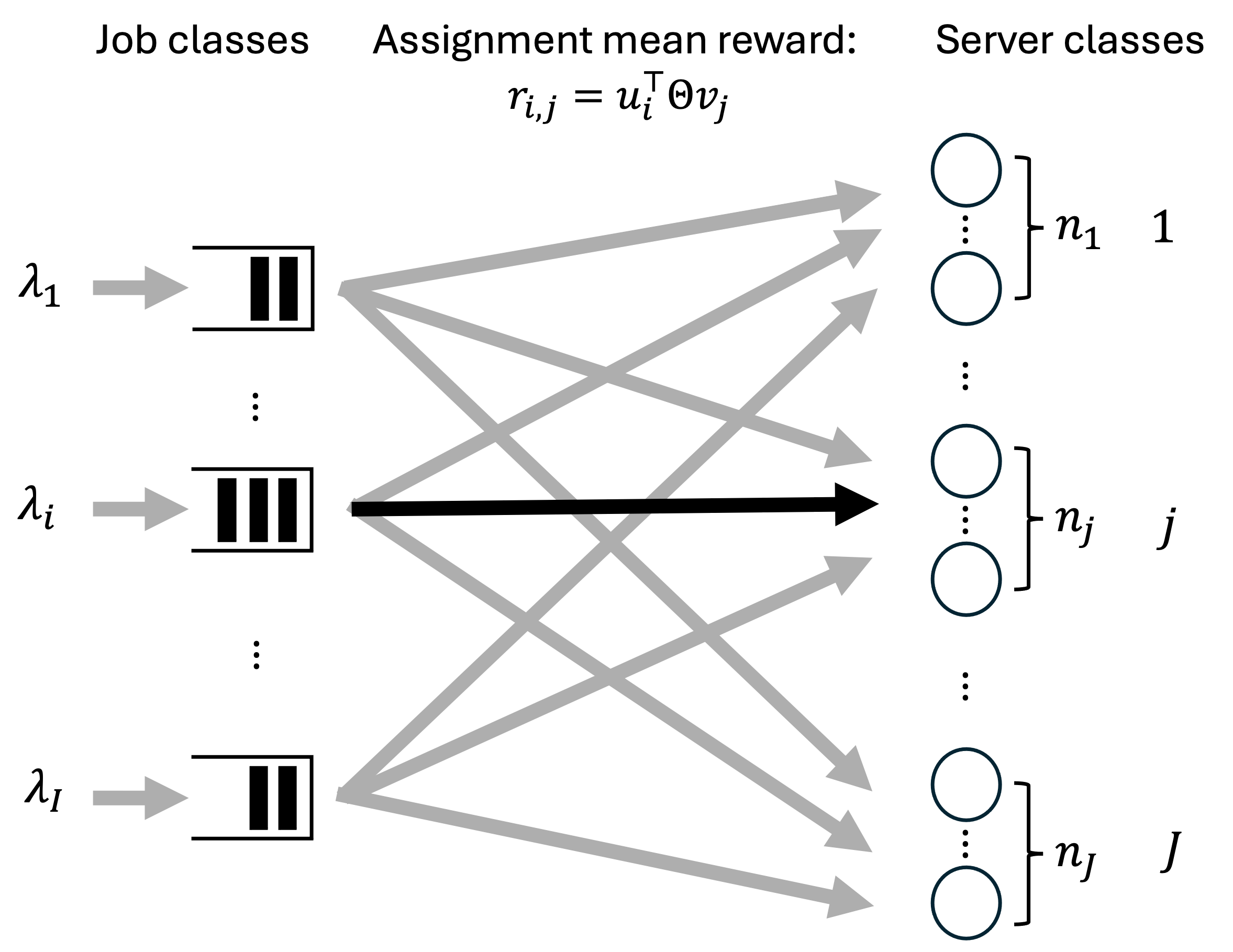}
\caption{Allocating a job of class $i$ to a server of class $j$ yields a stochastic reward according to a bilinear model, with mean value $r_{i,j} = u_i^\top \Theta v_j$, where $u_i$ and $v_j$ are known feature vectors representing job and server classes and $\Theta$ is an unknown parameter. 
}
\label{fig:bilinear1}
\end{figure}

The goal of the scheduler is to maximize the expected cumulative reward of the job-server assignments over a given time horizon while maintaining a bounded expected job holding cost at each time step. The scheduler's objectives are defined more formally below.

\subsection{Regret and Holding Costs} For evaluating the performance of a scheduling algorithm, we consider two criteria: the expected total reward of assignments over a time horizon of $T$ time steps and the holding cost of jobs at any given time step. Regarding the reward objective, we define regret as the difference between the cumulative reward of an oracle policy and the cumulative reward of the algorithm over the time horizon. 

The oracle is assumed to have knowledge of the mean rewards $r_{i,j}$ and the traffic intensity parameters $\rho_i:=\lambda_i/\mu_i$, representing the load induced by arriving jobs of class $i$. Let the oracle policy $p^*=(p^*_{i,j}:  i\in\I, j\in\J)$ be defined as an optimal fractional allocation of jobs to servers according to the following oracle optimization problem:
$$
\begin{array}{rl}
\mbox{maximize} & \sum_{i\in \I}\sum_{j\in \J}r_{i,j} \rho_i p_{i,j}\\%
\mbox{subject to} & \sum_{i\in \I}\rho_i p_{i,j}\le n_j \mbox{ for all } j\in \J\\
& \sum_{j \in \J}p_{i,j}= 1 \mbox{ for all } i\in \I\\
\mbox{over} & p_{i,j}\geq 0, \hbox{ for all } i\in \I, j\in \J.
\end{array}
$$
 The term $r_{i,j} \rho_i p_{i,j}$ represents the expected reward per unit time obtained by routing the load of arriving class-$i$ jobs to server class $j$. Let  $r^* = \sum_{i\in \I}\sum_{j \in \J}r_{i,j}\rho_i p_{i,j}^*$ be the optimal value of the oracle optimization problem.

The \emph{regret} of an algorithm with expected allocation $y(t)=(y_{i,j}(t); i\in \I(t), j\in \J)$ for all $t\in[T]$ is defined as
 \begin{align}
 R(T)\!=\! r^* T-\mathbb{E}\left[\sum_{t=1}^T\sum_{i\in \I(t)}\sum_{j \in \J}
 r_{i,j}  y_{i,j}(t)\right]
\label{equ:regdef}
 \end{align}
 where $\I(t)$ denotes the set of job classes waiting to be served at time $t$, i.e., $\I(t)=\{i\in \I: Q_i(t) > 0\}$.

The \emph{holding cost} at time step $t$ is defined as $\sum_{i\in \I} c_i Q_i(t)$, where $c = (c_1, \ldots, c_I)$ are given marginal holding cost parameters. We denote the mean holding cost as:
\begin{equation}
H(t; c) = \sum_{i\in\I}c_i\mathbb{E}[Q_i(t)].
\label{equ:hcostdef}
\end{equation} 
Specifically, when $c_i=1$ for all $i\in \I$, the holding cost corresponds to the total queue length.

The goal of the scheduler is to minimize finite-horizon regret while ensuring a time-uniform bound on the expected holding cost.

\subsection{Additional Assumptions and Notation}

In order to establish theoretical guarantees for a scheduling algorithm, we need to make assumptions about job arrivals, service times, and the uncertainty of rewards.

\paragraph{Arrivals and service times} The arrivals of jobs to the system are assumed to be according to a Bernoulli process; that is, in each time step there is either a single job arrival with probability $\lambda \in (0,1]$ or no job arrivals. The classes of jobs are assumed to be independent and identically distributed across job arrivals, with a job belonging to a class $i$ with probability $\lambda_i / \lambda$. The service times of the jobs are assumed to be independent between jobs and follow the geometric distribution with mean $1/\mu_i$ for jobs in class $i$. Following the standard terminology, we refer to $1/\mu_i$ as the \emph{mean service time} and $\mu_i\in(0,1]$ as the \emph{service rate}. The assumptions about the job arrivals and service times are standard in the analysis of queuing systems.

Following standard queuing system terminology, we refer to $\rho_i=\lambda_i/\mu_i$ as \emph{traffic intensity} of job class $i$, and let $\rho=\sum_{i\in \I}\rho_i$ denote the total traffic intensity.

For service times, we focus on the case of homogeneous service times, which is a special case such that $\mu_i = \mu$ for every $i\in \I$. This simplifies our analysis and the presentation of the main results. It is worth noting that this assumption was also made in \cite{hsu2021integrated}. However, we also discuss extensions of our results to general mean service times. For clarity, we treat $\mu$ and $\lambda_i$ as constants in the statements of the main results, while their full dependence is made explicit in the proofs.

\paragraph{Stability conditions} The queuing system stability region is the set of job arrival rates $(\lambda_1,\ldots,\lambda_I)$ for which the condition $\rho/n < 1$ holds \cite{tassiulas,mandelbaum2004scheduling,freund2022efficient,sentenac2021decentralized,gaitonde2021virtues}.  We assume that the queuing system satisfies the following stability condition:
\begin{assumption}
$\rho/n < 1$.    
\end{assumption}
 We define the \emph{traffic intensity slackness} $\delta=n-\rho$, which ensures that $\delta>0$ under the above assumption. Intuitively, a smaller value of $\delta$ indicates that the system is closer to instability (i.e., the total queue length grows to infinity). We define the following stability notions of queueing system:

\red{ 
\begin{definition}[Uniform expected queue-length stability] \label{def:queue_length_stability} A scheduling policy is said to be uniformly expected queue-length stable if there exists a finite constant \(B_Q<\infty\), independent of time \(t\), such that \[ \sup_{t\ge1} \mathbb E\left[\sum_{i\in\mathcal I}Q_i(t)\right] \le B_Q . \] \end{definition} 
\begin{definition}[Uniform time-averaged expected holding-cost stability] \label{def:holding_cost_stability} Fix a cost vector \(c\in\mathbb R_+^I\). For \(t\ge1\), define the time-averaged expected holding cost by \[ \overline H_t(c) := \frac1t\sum_{s=0}^{t-1}H(s;c) = \frac1t\sum_{s=0}^{t-1} \sum_{i\in\mathcal I}c_i\mathbb E[Q_i(s)] . \] A scheduling policy is said to be uniformly time-averaged expected holding-cost stable if there exists a finite constant \(B_c<\infty\), independent of the averaging horizon \(t\), such that \[ \sup_{t\ge1}\overline H_t(c) \le B_c . \] \end{definition}
}

\paragraph{Rewards of job-server assignments} The rewards of the job-server assignments are the sum of a mean value and a zero mean random variable according to a sub-Gaussian distribution. A random variable $\varepsilon$ is said to be sub-Gaussian with variance proxy $\sigma^2$ ($\sigma^2$-sub-Gaussian) if $\mathbb{E}[\varepsilon]=0$ and its moment generating function satisfies $\mathbb{E}[e^{s \varepsilon}]\leq e^{\sigma^2s^2/2}$ for all $s \in \mathbb{R}$. This assumption is standard in the learning literature. It includes a wide range of distributions, such as Gaussian and bounded random variables. 

We note that the bilinear model corresponds to a linear model by using the change of variables\footnote{For any given matrix $A$, $\mathrm{vec}(A)$ denotes the vector formed by stacking the rows of $A$.} $\theta=\mathrm{vec}(\Theta)\in\mathbb R^{d^2}$ and $z_{i,j}=\mathrm{vec}(u_iv_j^\top)\in\mathbb R^{d^2}$. Then, we can express the mean rewards as $r_{i,j} = u_i^\top\Theta v_j=z_{i,j}^\top\theta$, for $i\in \mathcal{I}$ and $j\in \mathcal{J}$. We assume that $\|\theta\|_2\le 1$ and $\|z_{i,j}\|_2\le 1$ for all $i\in \mathcal{I}$ and $j \in \mathcal{J}$, which ensures that $r_{i,j} \in [-1,1]$ for all $i\in \mathcal{I}$ and $j\in \mathcal{J}$.

\paragraph{Additional notation} Let $\Q_i(t)$ denote the set of jobs of class $i$ that are in the system at time $t$; let $Q_i(t)=|\Q_i(t)|$. Let $\Q(t) = \cup_{i\in \I}\Q_i(t)$. We denote by $A_i(t)$ and $D_i(t)$ the number of arrivals and departures of jobs of class $i$ at time $t$, respectively, and define $A(t)=\sum_{i\in\I}A_i(t)$ and $D(t)=\sum_{i\in\I}D_i(t)$.  Note that $Q_i(t+1)=\max\{Q_i(t)+A_i(t+1)-D_i(t),0\}$. Let $\cl:\cup_{t\geq 0} \Q(t)\rightarrow \I$ denote the mapping of jobs to the corresponding job classes. 

At each time step $t$, the scheduler assigns jobs in $\mathcal{Q}(t)$ to the servers. Let $\tilde{\mathcal{Q}}(t)$ denote the set of successfully assigned jobs. The number of successfully assigned jobs is less than or equal to the number of servers $n$. At the end of each time step $t$, for each assignment of a job $k\in\tilde{\mathcal{Q}}_i(t)$ to a server in class $j$, the scheduler observes a stochastic reward of value $\xi_{k,t} = r_{i,j} + \eta_{k,t}$, where $\eta_{k,t}$ follows a 1-sub-Gaussian distribution. 

Furthermore, we use the following additional notation: the weighted norm of a vector $x\in\reals^{d}$ with respect to a weight matrix $A\in \reals^{d\times d}$ is defined as $\|x\|_A=\sqrt{x^\top A x}$. We use the big-O notation $\tilde{O}(\cdot)$ to ignore poly-logarithmic factors.

\section{Algorithm and Theoretical Guarantees}
\label{sec:alg}

In this section, we present our main results, which comprise a scheduling algorithm and theoretical guarantees on its performance with respect to regret and mean holding cost.

\subsection{Algorithm}

\begin{algorithm}[t]
\LinesNotNumbered
\caption{Scheduling Algorithm for Bilinear Rewards}  %
\textbf{Input:} target horizon \(T\), control parameter \(V>0\), weights
\(w=(w_i:i\in\mathcal I)\), and \(\gamma>1\)\\
\textbf{Initialize: }
$\Lambda^{-1}\leftarrow (1/n)I_{d^2\times d^2}$, $b\leftarrow 0_{d^2\times 1}$

\For{$t=1,2, \ldots$}{

\tcp{Optimize allocation}

$\hat{\theta}\leftarrow \Lambda^{-1}b$

$\tilde{r}_{i,j}(t)\leftarrow \Pi_{[-1,1]}(z_{i,j}^\top \hat{\theta}+\sqrt{z_{i,j}^\top \Lambda^{-1} z_{i,j} }\beta(t))$ for $i\in \I$ and $j\in \J$ 

Set $(y_{i,j}(t): i\in \I, j\in \J)$ to the solution of (\ref{eq:weighted})

\tcp{Assign jobs to servers}

\For{$j=1, \ldots, J$}{
\For{$l=1, \ldots, n_j$}{

Choose a job $k_{t,j,l}\in \Q(t)$ randomly with probabilities $y_{\cl(k),j}/(n_jQ_{\cl(k)}(t))$ for $k\in \Q(t)$, 

 or, choose no job $k_{t,j,l}=k_0$ with probability $1-\sum_{k\in\Q(t)}y_{\cl(k),j}/(n_jQ_{\cl(k)}(t))$

\If{job $k_{t,j,l}\neq k_0$ }{

Assign job $k_{t,j,l}$ to server $l$ of class $j$  to process one service unit of this job

Observe reward $\xi_{t,j,l}$ of assigned job $k_{t,j,l}$

}}}

\tcp{Update}

\For{$j=1, \ldots, J$}{

\For{$l=1, \ldots, n_j$}{

\If{$k_{t,j,l}\neq k_0$}{
$i\leftarrow \cl(k_{t,j,l})$

$\Lambda^{-1} \leftarrow \Lambda^{-1}-\frac{\Lambda^{-1} z_{i,j}z_{i,j}^\top \Lambda^{-1}}{1+z_{i,j}^\top \Lambda^{-1} z_{i,j}}$

$b\leftarrow b + z_{i,j}\xi_{t,j,l}$
}}}}
\label{alg:Alg1}
\end{algorithm}

We introduce a scheduling algorithm that uses upper confidence bound (UCB) indices for job-server assignment rewards and weighted proportional fair allocation to account for job priorities and ensure queuing system stability, inspired by \cite{hsu2021integrated,abbasi2011improved, jun2019bilinear}. The algorithm is described in pseudocode as Algorithm~\ref{alg:Alg1}. The different steps performed by the algorithm are detailed below.

\subsubsection{Expected Allocation} 

The algorithm uses the expected allocation $y(t)=(y_{i,j}(t): i\in\I, j\in \J)$ at each time step $t$, which is the solution to the following convex optimization problem, with $\tilde{r}$ and $Q$ set to $\tilde{r}(t)$ and $Q(t)$, respectively:  
\begin{align}
    \begin{array}{rl}
\mbox{maximize} & P(y; \tilde{r},\gamma) + \frac{1}{V} F(y; \we, Q)\\
\mbox{subject to} & \sum_{i\in \I}y_{i,j}\leq n_j, \hbox{ for all } j\in \J\\
\mbox{over} & y_{i,j}\ge 0, \hbox{ for all } i\in \I ,j\in \J 
\end{array}\label{eq:weighted}
\end{align}
where 
\begin{align*}
P(y; \tilde{r}, \gamma) :=& \sum_{i\in \I}\sum_{j\in \J} (\tilde{r}_{i,j} - \gamma)y_{i,j}\\
F(y; \we, Q) :=& \sum_{i\in \I} Q_i \we_i  \log\left(\sum_{j\in \J}y_{i,j}\right),
\end{align*}
and $w\in \mathbb{R}^{I}_+$, $V> 0$, and $\gamma >1$ are tunable parameters. 

In the objective function in (\ref{eq:weighted}), $P(y; \tilde{r}, \gamma)$ accounts for the objective of maximizing the rewards of the job-server assignments, while $F(y; w, Q)$ accounts for the objective of weighted proportional fair allocation, ensuring fairness of allocation and stability of the queuing system. 

For each $i\in \I$ and $j\in \J$, we can interpret $\gamma - \tilde{r}_{i,j}(t) > 0$ as a marginal assignment cost. The parameter $\gamma$ allows us to control server utilization. Note that conditions $\gamma > 1$ and $\tilde{r}_{i,j}(t)\leq 1$ for all $i\in \I$, $j\in \J$ and $t\in[T]$ ensure the non-negativity of marginal assignment costs. 

The parameters $w$ allow for a weighted version of proportional fair allocation, for each class of jobs $i\in \I$, giving higher priority to classes with larger $w_i$. Additionally, the parameter $V$ allows us to control the trade-off between the weighted proportional fair allocation and the marginal cost in the objective. Increasing $V$ places greater emphasis on cost minimization (or reward maximization), thus reducing the influence of considerations of fairness and stability.

\subsubsection{Randomized Assignment} At each time step $t$, after computing the expected allocation $y(t)$, the algorithm utilizes a randomized allocation to assign jobs to the servers, ensuring that the expected allocation is consistent with $y(t)$. This is achieved through the following procedure. Let $\cl(k)$ denote the class of job $k$. Then, for each server class $j\in \mathcal{J}$, a job $k\in \mathcal{Q}(t)$ is selected with probability $y_{\cl(k),j}(t)/(n_jQ_{\cl(k)}(t))$. When a server of class $j\in \J$ and index $l\in [n_j]$ selects a job $k_{t,j,l}\in\mathcal{Q}(t)$, the algorithm assigns this job to the server and observes the value of the corresponding stochastic reward.

It should be noted that randomized assignment makes the scheduling policy non-work-conserving. In the case where $\sum_{i\in \mathcal{I}(t)} y_{i,j}(t) < n_j$ for some server class $j\in \mathcal{J}$ and some time step $t\geq 1$, a server of class $j\in \mathcal{J}$ may not be assigned a job at time step $t$ with probability $1-\sum_{k\in \mathcal{Q}(t)}y_{\cl(k),j}(t)/(n_jQ_{\cl(k)}(t))$. This allows the scheduler to defer assignment and avoid low-reward assignments for better decision making in future rounds. 
 
\subsubsection{Bilinear Bandit Strategy} The algorithm employs a UCB strategy for bilinear bandits, which estimates the unknown parameter $\theta$ of the bilinear stochastic reward model from the observed partial feedback and balances the trade-off between exploration and exploitation. In Algorithm~\ref{alg:Alg1}, for an assigned job $k_{t,j,l}$ in the $l$-th selection by server class $j$ at time $t$, the algorithm observes the stochastic reward value $\xi_{l,j}(t)$. In the following, we simplify the notation by denoting the feature information for $k_{t,j,l}$ as $z_{l,j}(t)=z_{\cl(k_{t,j,l}), j}$, where $\cl(\cdot)$ maps a job to its class. Also, for clarity in analysis, we use $\hat{\theta}(t), \Lambda(t), b(t)$ for $\hat{\theta},\Lambda, b$ in time step $t$ of Algorithm~\ref{alg:Alg1}. At each time step $t$, the algorithm utilizes an estimator of $\theta$ defined as:
\[
\hat{\theta}(t)=\Lambda(t)^{-1}b(t)
\]
where
\[
\Lambda(t)= nI_{d^2\times d^2}+\sum_{s=1}^{t-1}\sum_{j\in \J}\sum_{l\in[n_j]}z_{l,j}(s)z_{l,j}(s)^\top
\]
and
\[
b(t)=\sum_{s=1}^{t-1}\sum_{j\in \J}\sum_{l=1}^{n_j} z_{l,j}(s)\xi_{l,j}(s).
\]
The UCB indices $\tilde{r}(t)$ are defined as, for $i\in \I$ and $j\in \J$:
\[
\tilde{r}_{i,j}(t)=\Pi_{[-1,1]}\left(\max_{\theta'\in {\mathcal C}(t)}\{z_{i,j}^\top \theta'\}\right),
\]
where $\Pi_{[-1,1]}(x):=\max\{-1,\min\{x,1\}\}$ and ${\mathcal C}(t)$ is the confidence set defined as:
\[
{\mathcal C}(t)=\left\{\theta'\in \reals^{d^2}:\|\hat{\theta}(t)-\theta'\|_{\Lambda(t)}\le{\beta(t)}\right\}
\]
with ${\beta(t)}= \sqrt{d^2\log(tT)}+\sqrt{n}$. %
It can be easily shown that
$$
\tilde{r}_{i,j}(t) =\Pi_{[-1,1]}\left( z_{i,j}^\top \hat{\theta}(t) + \sqrt{z_{i,j}^\top \Lambda(t)^{-1}z_{i,j}} {\beta(t)}\right).
$$
The algorithm iteratively computes the inverse matrix $\Lambda(t)^{-1}$ using the Sherman-Morrison formula \cite{H89}, exploiting the fact that $\Lambda(t)$ is a weighted sum of an identity matrix and rank-1 matrices. This has a computational cost of $O(d^4+IJd^4)$ per round for computing the mean reward estimators, where the first term accounts for computing the inverse matrix $\Lambda(t)^{-1}$ and the second term accounts for computing the mean reward estimators.

\subsection{Regret, Queue Length, and Holding Cost Bounds}
\label{sec:Alg1}

In this section, we provide regret and holding cost bounds for Algorithm~\ref{alg:Alg1}. 

\subsubsection{Regret Bound} We present a regret bound for Algorithm~\ref{alg:Alg1} in the following theorem. The proof is provided in Appendix~\ref{sec:regb1}. 

\begin{theorem} Given a target horizon $T\ge 1$,
for any $V > 0$ and constant $\gamma > 1$, the regret of Algorithm~\ref{alg:Alg1} is bounded as

\begin{equation}
R(T)
=\widetilde O\!\left(
   \bar\alpha_1V
   +\bar\alpha_2\frac1\delta
   +\bar\alpha_3\frac{T}{V}
   +\bar\alpha_4\sqrt T
\right),
\label{eq:regret_final_scalar_queue}
\end{equation}
where 
\[
\bar\alpha_1
  =\frac{n}{\mu w_{\min}},
\qquad
\bar\alpha_2
  =\frac{1}{\mu}\left(
      n+\rho+2c_\gamma n^3\frac{w_{\max}}{w_{\min}}
    \right),
\]
\[
\bar\alpha_3
  =\frac{n^2}{2}(1+c_\gamma^2\mu)
      \max_{i\in\I}\frac{w_i}{\rho_i}
    +\frac12\sum_{i\in\I}w_i,
\qquad
\bar\alpha_4=d^2\sqrt n+dn.
\]

Here, $w_{\min} = \min_{i \in \mathcal{I}} w_i$ and $w_{\max} = \max_{i \in \mathcal{I}} w_i$.
\label{thm:regbd1}
\end{theorem}

To highlight the key dependencies in the regret bound, focusing on $T$, $V$, $I$, $d$, and $\delta$, while treating other terms as constants, we can simplify it as 
\begin{equation}
R(T)=  \tilde{O}\left(V+\frac{1}{\delta}+ \frac{1}{V}IT+d^2\sqrt{T}\right).
\label{equ:regt}
\end{equation}
The terms in the regret bound in (\ref{equ:regt}) originate from three sources. The first two terms, $V$ and $1/\delta$, are derived from the bounding of the expected queue length at the time step $T$. The third term of $\frac{1}{V}IT$ arises from the stochasticity of job departures due to the use of randomized job-server assignments. The last term, $d^2\sqrt{T}$, comes from the bandit algorithm used to learn the mean rewards of the assignments. 
{
\begin{corollary}[Regret-optimized tuning]
\label{cor:regret-optimized-v}
Under the assumptions of Theorem~\ref{thm:regbd1}, choosing
\[
  V_T=\Theta(\sqrt{IT})
\]
\red{before the start of the algorithm} gives
\[
  R(T)
  =
  \widetilde O\!\left(
    (\sqrt I+d^2)\sqrt T+\frac{1}{\delta}
  \right),
\]
up to constants depending on \(n,\gamma,w_{\min},w_{\max}\).
\end{corollary}}

\red{
\begin{remark}[On the horizon-dependent choice of \(V\)]
Theorem~\ref{thm:regbd1} is proved for an arbitrary fixed
deterministic value of \(V>0\). Once the target horizon \(T\) is fixed, we choose
\(V_T\) before the algorithm starts and keep it constant throughout the run. Therefore, applying Theorem~\ref{thm:regbd1} with
\(V=V_T\) is legitimate and does not introduce any additional
terms.
\end{remark}}

\red{\paragraph{Worst-case regret guarantee} The regret bound in Corollary~\ref{cor:regret-optimized-v} should be interpreted as a problem-instance-independent, or gap-free, guarantee. After vectorizing the bilinear model, the reward estimation problem becomes a structured linear-bandit problem in dimension \(d^2\), and the confidence-set analysis yields the standard gap-free contribution \(\widetilde O(d^2\sqrt T)\), as in worst-case guarantees for linear bandits \cite{abbasi2011improved}, logistic bandits \cite{faury2020improved}, and bilinear bandits \cite{jun2019bilinear}. Moreover, the regret in our scheduling problem is not only a reward-estimation error. It also contains queue-control and randomized-service terms, \(O(V_T)\) and \(O(IT/V_T)\), which arise from the drift-plus-penalty analysis and from stochastic departures induced by randomized assignments. Balancing these terms gives \(V_T=\Theta(\sqrt{IT})\), leading to the finite-horizon regret--stability tradeoff in Corollary~\ref{cor:regret-optimized-v}.}

\red{\paragraph{Comparison with the unstructured-reward setting of \cite{hsu2021integrated}}
The queue-control part of our algorithm is inspired by the utility-guided
dynamic assignment framework of \cite{hsu2021integrated}. The main difference
is in how rewards are learned. In \cite{hsu2021integrated}, the mean rewards
are unstructured and must be learned separately for different job--server
combinations. The regret bound of their algorithm (\texttt{UGDA-OL}) is stated for a time-averaged performance
loss. Converting it to cumulative regret over horizon \(T\) and using the
optimized balancing choice \(V_T=\Theta(\sqrt{IT})\) gives
\[
  R_{\mathrm{UGDA\text{-}OL}}(T)
  =
  \widetilde O\!\left(
    \sqrt{IT}+JT+\frac{1}{\delta}
  \right).
\]
The linear term \(JT\) is the learning cost of treating the rewards as
unstructured. In contrast, our bilinear model writes
\[
  r_{i,j}=u_i^\top\Theta v_j
  =
  z_{i,j}^\top\theta,
\]
so every observed reward updates a shared \(d^2\)-dimensional parameter
\(\theta\). This parameter sharing reduces the learning contribution to
\[
  \widetilde O(d^2\sqrt T),
\]
which is sublinear in \(T\) whenever \(d\) is fixed. The queue-length
guarantees are qualitatively comparable for fixed \(V_T\). The detailed conversion of
the regret bound is provided in Appendix~\ref{app:com}.}

\red{Beyond the regret improvement, our formulation also explicitly controls
general weighted holding costs. The weighted proportional-fair term in our
algorithm allows the weights \(w_i\) to be chosen according to the marginal
holding costs \(c_i\). In particular, choosing \(w_i=c_i\) yields a
cost-aware holding-cost bound, while the usual total queue-length guarantee is
recovered as the special case \(c_i=1\). Thus, our weighted algorithm provides
a more general cost-sensitive queue-control guarantee, allowing high-cost job
classes to be prioritized.}
%
%
%

%
%
%

%

%

\subsubsection{Queue Length and Holding Cost Bounds} We next consider the mean queue and holding cost of Algorithm~\ref{alg:Alg1}. Let $c_{\max}=\max_{i\in\I}c_i$ and recall that $w_{\max}=\max_{i\in\I}w_i$. In the following theorem, we provide a bound on the mean  queue length. The proof is provided in Appendix~\ref{app:q_len_bound}.

\red{
\begin{theorem}[Uniform expected queue length bound] \label{thm:Q_len_bound}  Algorithm~\ref{alg:Alg1} guarantees that, for any \(V>0\) 
 \[ \sup_{t\ge1} \mathbb E\left[\sum_{i\in\mathcal I}Q_i(t)\right] = O\left(V+\frac1\delta\right), \]
  when \(n,\gamma,w_{\min},w_{\max}\) are treated as fixed system parameters.
 \end{theorem}}

 \red{
Theorem~\ref{thm:Q_len_bound} shows that the algorithm guarantees a time-uniform bound on the expected total queue length, thereby establishing stability in the sense of Definition~\ref{def:queue_length_stability}. The higher the value of $V$, the less weight is placed on the fairness term in the objective function of the optimization problem in (\ref{eq:weighted}). The queue length bound has an inverse dependency on $\delta$, which is typical for mean queue length bounds \cite{freund2022efficient,yang2023learning,huang2024lyapunov}.}

 \red{In particular, with the regret-optimized choice \(V_T=\Theta(\sqrt{IT})\), chosen before the algorithm starts and kept fixed during execution,  Theorem~\ref{thm:Q_len_bound} gives \[ \sup_{t\ge1} \mathbb E\left[\sum_{i\in\mathcal I}Q_i(t)\right] = O\left( \sqrt{IT} + \frac1\delta \right). \] 
Thus, for each prescribed target horizon \(T\), the resulting policy admits a uniform expected total-queue bound over all execution times \(t\ge1\).}

\red{Now we provide a theorem for the holding cost bound for $H(t;c)$ of Algorithm~\ref{alg:Alg1}, defined in (\ref{equ:hcostdef}) for marginal holding costs $c_1, \ldots, c_I$. The proof is provided in Appendix~\ref{app:holding_bound}.}

\red{\begin{theorem}[Uniform time-averaged epxected holding-cost stability] \label{thm:Q_len_bound_cost}  For any averaging horizon \(t\ge1\), define \[ \overline H_t(c) := \frac1t\sum_{s=0}^{t-1}H(s;c). \] Algorithm~\ref{alg:Alg1} is uniformly time-averaged expected holding-cost stable for any fixed \(V>0\) satisfying \[ \sup_{t\ge 1}\overline H_t(c) = O\left( \max_{i\in \I}\frac{c_i}{w_i}\frac{1}{\delta}(V+Iw_{\max}) \right). \] In particular, if \(w_i=c_i\) for all \(i\in\mathcal I\), then \[ \sup_{t\ge 1}\overline H_t(c) = O\left( \frac{V+Ic_{\max}}{\delta} \right). \] \end{theorem}}

\red{Theorem~\ref{thm:Q_len_bound_cost} shows that the algorithm ensures a uniform bound on the time-averaged expected holding cost, thereby establishing stability in the sense of Definition~\ref{def:holding_cost_stability}. Moreover,
 the weights \(w_i\) control the dependence of the holding cost on the cost vector \(c\) through the mismatch factor \(\max_i c_i/w_i\). Hence, choosing \(w_i=c_i\) removes this multiplicative mismatch factor from the leading control term. With the regret-optimized choice \(V_T=\Theta(\sqrt{IT})\), this gives \[ \sup_{t\ge 1} \overline H_t(c) = O\left( \frac{\sqrt{IT}+Ic_{\max}}{\delta} \right), \] up to fixed system-dependent constants. Thus, for each prescribed target horizon \(T\), the resulting policy admits a uniform time-average expected holding cost bound over all execution times \(t\ge1\).}

\red{The \(1/\delta\) factor in the time-average holding-cost bound comes from
the weighted Lyapunov drift argument: the negative drift coefficient is
proportional to the traffic slackness \(\delta=n-\rho\). This bound is
therefore slightly weaker than the scalar total-queue bound, which is proved
by a sharper one-dimensional coupling argument. Nevertheless, the
cost-aware choice \(w_i=c_i\) still eliminates the multiplicative
cost-mismatch factor \(\max_i c_i/w_i\) from the leading term.}

\red{
\begin{remark}[Regret--stability tradeoff]
The parameter \(V\) controls the regret--stability tradeoff. If \(V\) is
fixed independently of the target horizon \(T\), then the holding-cost bound
is uniform in time and independent of \(T\), but the regret bound contains the
term \(IT/V\) and is linear in \(T\). Balancing the terms \(V\) and \(IT/V\)
in the regret upper bound gives the horizon-dependent tuning
\(V_T=\Theta(\sqrt{IT})\). For each prescribed target horizon \(T\), this
value is chosen before the algorithm starts and is kept fixed during
execution. Hence the resulting policy achieves sublinear regret over
\([1,T]\), while the
holding-cost bound remains uniform over the execution time \(t\ge 1\).
\end{remark}
}

\red{
\begin{remark}[Target horizon and execution horizon]
The target horizon \(T\) is used only to tune the parameter \(V_T\) and
the confidence radius. The policy can be run for all
\(t\ge 1\). The regret guarantee is evaluated over the first \(T\) time slots,
whereas the stability guarantee is evaluated over the infinite time axis of
the same fixed policy.
\end{remark}
}

%
%
%
%
%
%
%
%
%

%
%

%

%
%
%

%
%
%
%
%
%
%
%
%

%
 
 %
 
%

%
%
%
%
%

%

%

%

\section{Distributed Allocation Algorithms}
\label{sec:dist}

\begin{figure}[t]
\centering
\includegraphics[width=\linewidth]{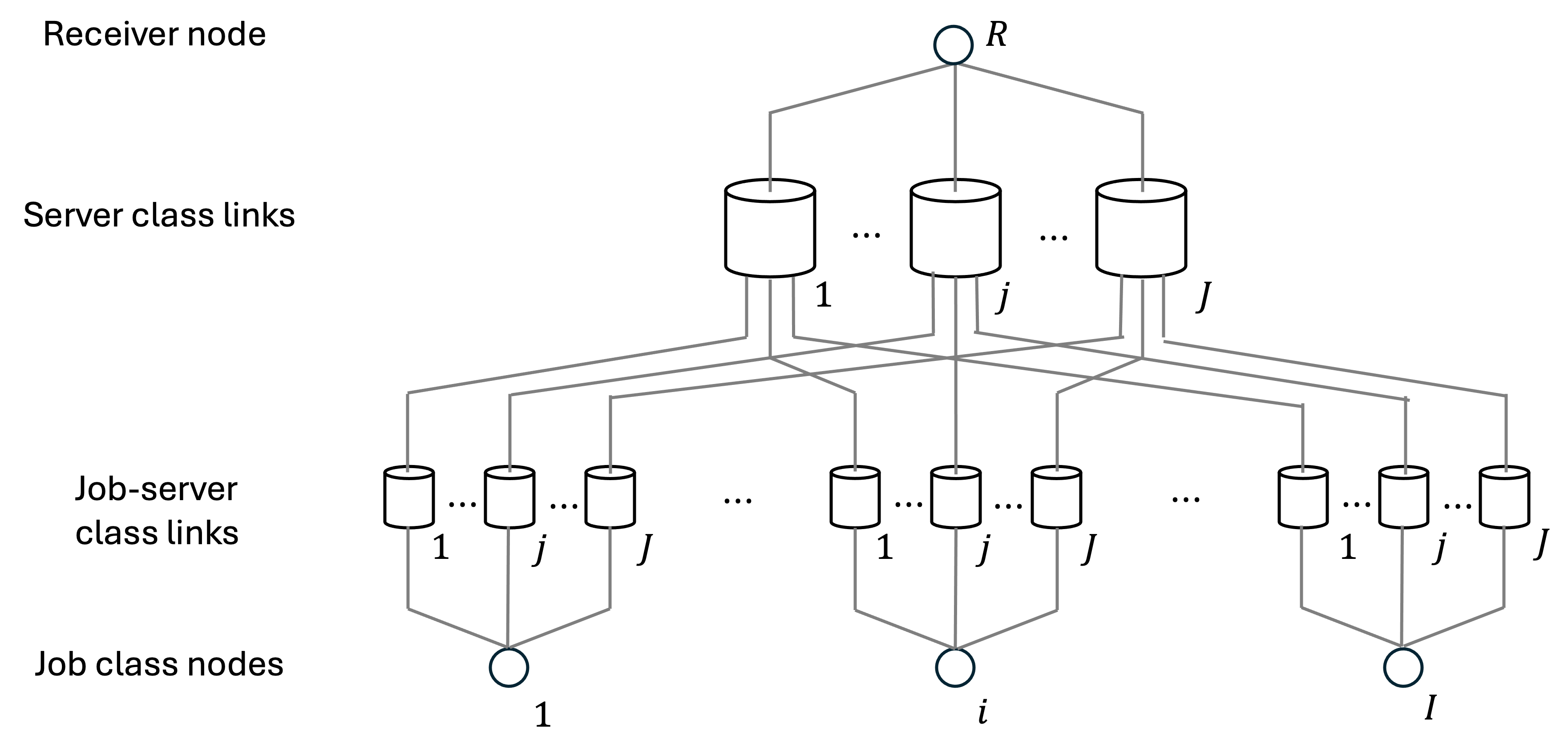}
\caption{Distributed computation of the expected allocation corresponds to a joint routing and rate control problem with a job class node as a source and a virtual receiver node. A routing path $(i,j)$ originates at job node $i$, traverses a link used exclusively by this path with a marginal price $\gamma - \tilde{r}_{i,j}$ and then traverses link $j$ corresponding to a job server node, and finally terminates at the receiver node.} 
\label{fig:bilinear2}
\end{figure}

The scheduling algorithm defined in Algorithm~\ref{alg:Alg1} can be run by a dedicated compute node in a centralized computation implementation. For large-scale systems with many jobs and servers, it is of interest to consider distributed scheduling algorithms, where computations are performed by compute nodes representing jobs or servers. The part of the algorithm concerned with the computation of mean reward estimators can be easily distributed. However, the part concerned with the computation of expected allocations of jobs to servers requires solving the convex optimization problem (\ref{eq:weighted}) in each time step.

In this section, we discuss distributed iterative algorithms for approximately computing these expected allocations, where iterative updates are performed by compute nodes representing jobs or servers. These iterative updates follow certain projected gradient descent-type algorithms with feedback delays due to distributed computation. We provide sufficient conditions for the exponential-rate convergence of these iterative algorithms. The definition of the algorithms and their convergence analysis exploits a relation with the joint routing and rate control problem addressed in \cite{kv05}.

In what follows, we consider an arbitrary time step $t$ and omit reference to $t$ in our notation. For simplicity, we consider $w_i=1$ for $i\in\I$. We write $\Q$ in lieu of $\Q(t)$ and $\I_+$ in lieu of $\I(t)$.  With a slight abuse of notation, we let $y_{i,j}(r)$ denote the value of allocation for a job-server class combination $(i,j)$ at iteration $r$, for $i\in \I_+$ and $j\in \J$, and let $y(r) = (y_{i,j}(r): i\in \I_+, j\in \J)$. We consider iterative updates of allocations under the assumption that the set of jobs and mean reward estimators are fixed. This is a standard assumption when studying such iterative allocation updates in network resource allocation problems. It is a limit case of an operational regime where such iterative updates are run at a faster timescale than the timescale at which the set of jobs and mean reward estimates change. 

We refer to the node maintaining state for a job class $i\in \I$ as \emph{job-node} $i$, and the node maintaining state for a server class $j\in \J$ as \emph{server-node} $j$. Let $\tau_{(i,j)}$ denote the round-trip delay for the $(i,j)$ job-server class, defined as the sum of the delay for communicating information from job-node $i$ to server-node $j$, denoted by $\tau_{i,j}$, and the delay for communicating information in the reverse direction, denoted by $\tau_{j,i}$.

\subsection{Allocation Computed by Job Nodes} We consider distributed computation where each job-node $i\in \I_+$ computes $y_i(t) :=  (y_{i,j}(r): j\in \J)$ for $r\geq 0$ by using the following iterative updates:
\begin{equation}
y_{i,j}(r+1) = y_{i,j}(r) + \alpha_{i,j}\left(1-\frac{\lambda_{i,j}(r)}{u'_i(y^\dag_i(r))}\right)_{y_{i,j}(r)}^+ 
\label{equ:xdiff0}
\end{equation}
where 
\begin{align*}
\lambda_{i,j}(r):= & p_j\left(y^\S_j(r-\tau_{j,i})\right) + \gamma - \tilde{r}_{i,j},\\
y^\dag_i(r):= & \sum_{j\in \J} y_{i,j}(r-\tau_{(i,j)}),\\ 
y^\S_j(r):= & \sum_{i'\in \I_+}y_{i',j}(r-\tau_{i',j}),\\
u_i(y) :=& \frac{1}{V}|\Q_i|\log(y)
\end{align*}
with $\alpha_{i,j} > 0$ being a step size parameter, $p_j$ being a non-negative continuously differentiable function with strictly positive derivative, and $(b)_a^+ = b$ if $a> 0$ and $(b)_a^+ = \max\{b,0\}$ if $a = 0$. Note that $y_i(r)$ can be interpreted as allocation to job-node $i$ acknowledged via feedback from server-nodes $j\in \J$.

We study the convergence properties of the above iterative method in continuous time by considering the system of delay differential equations, for $i\in \I_+$ and $j\in \J$,
\begin{equation}
\frac{d}{dr}y_{i,j}(r) = \alpha_{i,j}\left(1-\frac{\lambda_{i,j}(r)}{u_i'(y^\dag_i(r))}\right)_{y_{i,j}(r)}^+.
\label{equ:xdiff}
\end{equation}
Let $C_j(z) = \int_0^z p_j(u)du$ for $j\in \J$. 
An allocation $y= (y_{i,j} : i\in \I_+, j\in \J)$ is said to be an \emph{equilibrium point} of (\ref{equ:xdiff}) if it is a solution of the convex optimization problem:
\begin{equation}
\hspace{-2mm}\begin{array}{rl}
\!\!\mathrm{maximize}\!\! & \!\! \!\!\sum_{i\in \I_+}\!\! u_i(y^\dag_i) \!- \!\sum_{j\in \J} (C_j(y^\S_j)\!+\!\sum_{i\in \I_+}\!\!(\gamma \!- \! \tilde{r}_{i,j}) y_{i,j})\\
\mathrm{subject\: to}\!\! &\!\! y^\dag_i = \sum_{j\in \J} y_{i,j}, \hbox{ for all } i\in \I_+\\
&\!\!  y^\S_j = \sum_{i\in \I_+} y_{i,j}, \hbox{ for all } j\in \J\\
\mathrm{over} \!\! &\!\!  y_{i,j}\geq 0, \hbox{ for all } i\in \I_+, j\in \J.
\end{array}
\label{equ:opt2}
\end{equation}

The objective function of (\ref{equ:opt2}) is maximized at $y$ if, and only if, for all $i\in \I_+$ and $j\in \J$,
\begin{eqnarray}
&& y_{i,j}\geq 0,\label{equ:oc1}\\
&& u_i'(y^\dag_i)-p_j(y^\S_j) - (\gamma - \tilde{r}_{i,j}) \geq 0, \hbox{ and }\label{equ:oc2}\\ 
&& y_{i,j}\left(u_i'(y^\dag_i)-p_j(y^\S_j) - (\gamma - \tilde{r}_{i,j})\right) = 0.\label{equ:oc3}
\end{eqnarray}

\begin{figure*}[ht]%
\centering
\includegraphics[width=0.4\linewidth]{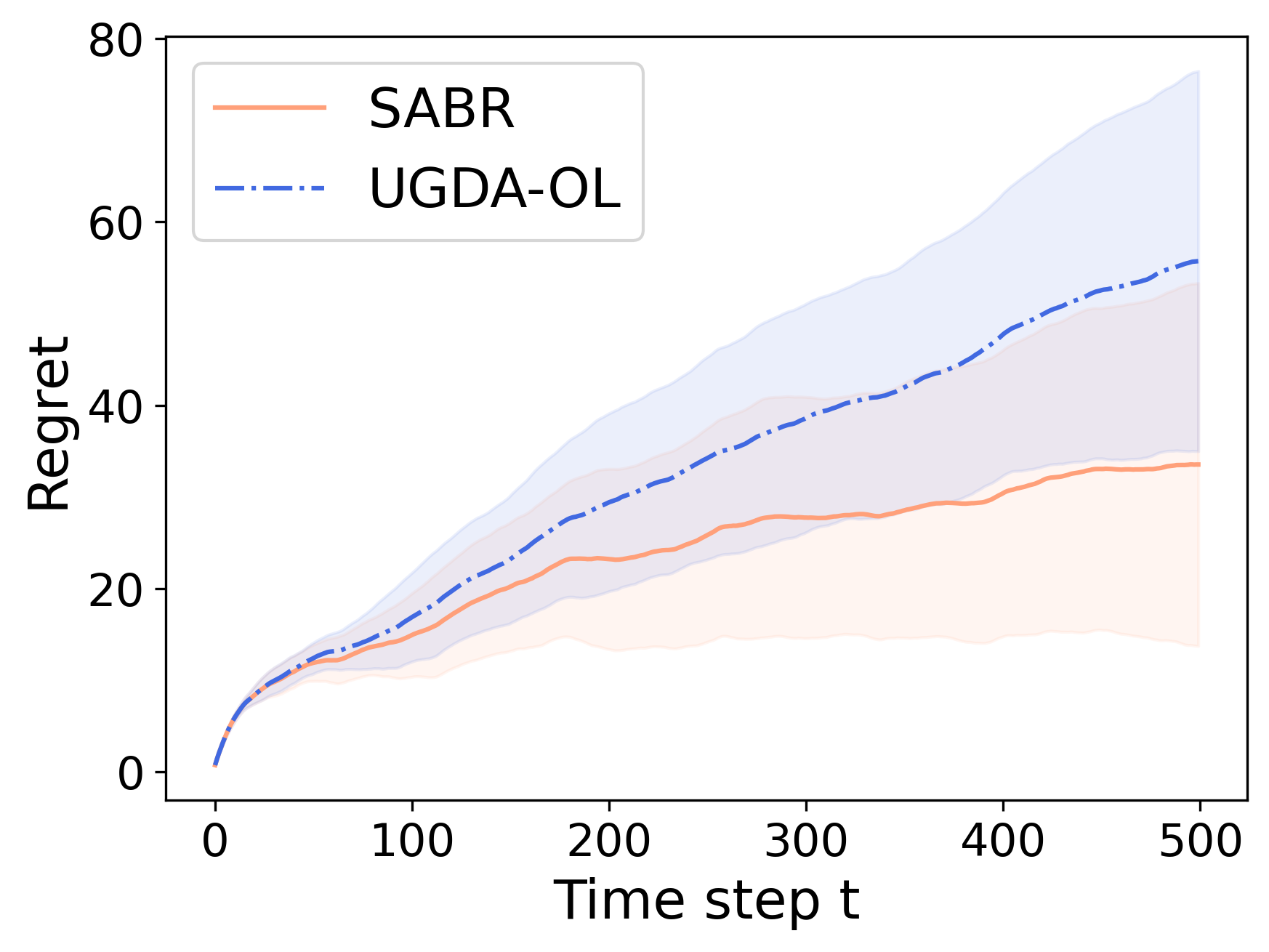}\hspace*{1.5cm}
\includegraphics[width=0.41\linewidth]{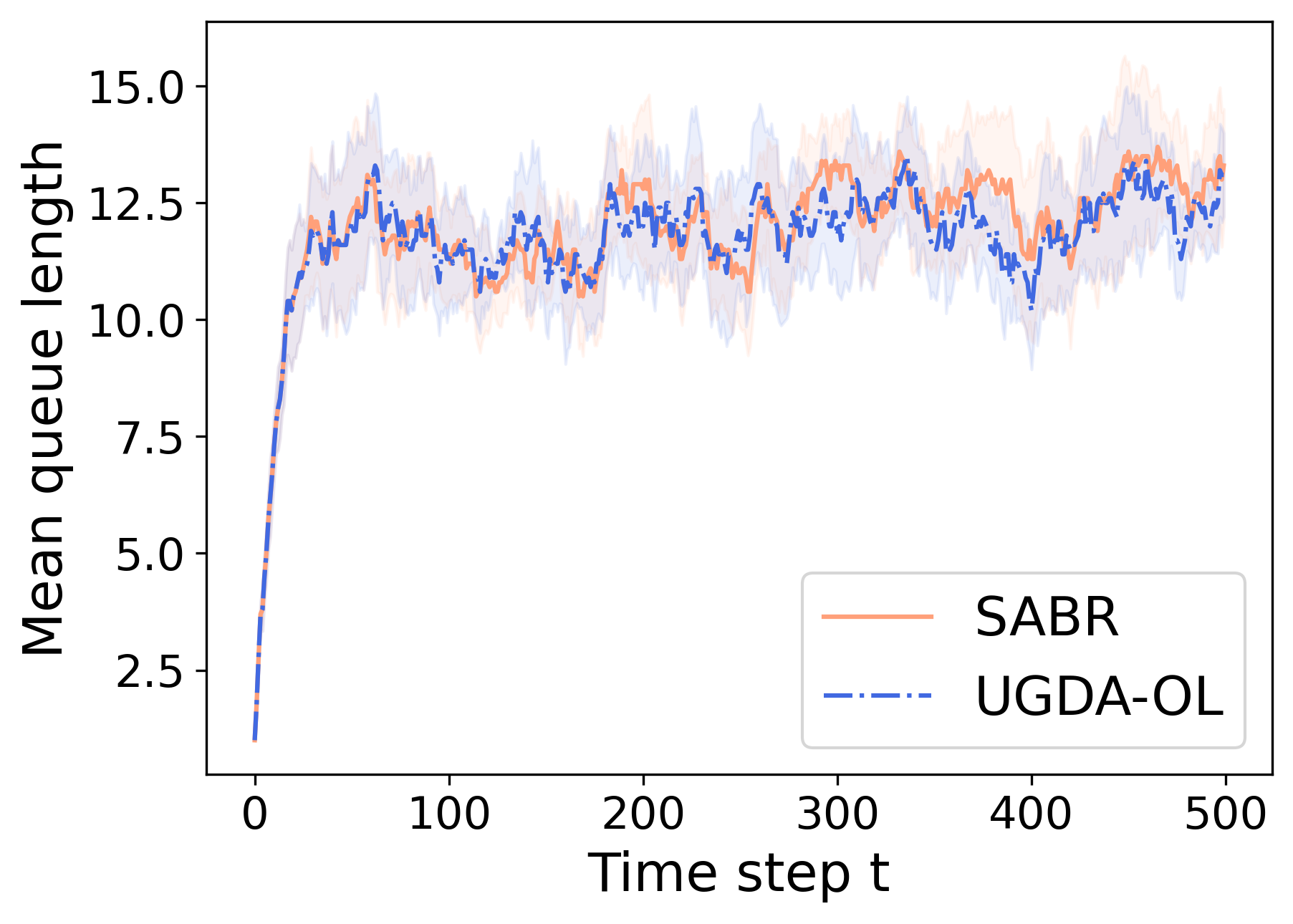}
\caption{Performance of \texttt{SABR} and \texttt{UGDA-OL} over time steps: (left) regret and (right) mean queue length.
}
\label{fig:com}
\end{figure*}
\begin{figure*}[ht]
\centering

\begin{minipage}[ht]{0.43\linewidth}
    \centering
    \includegraphics[width=\linewidth]{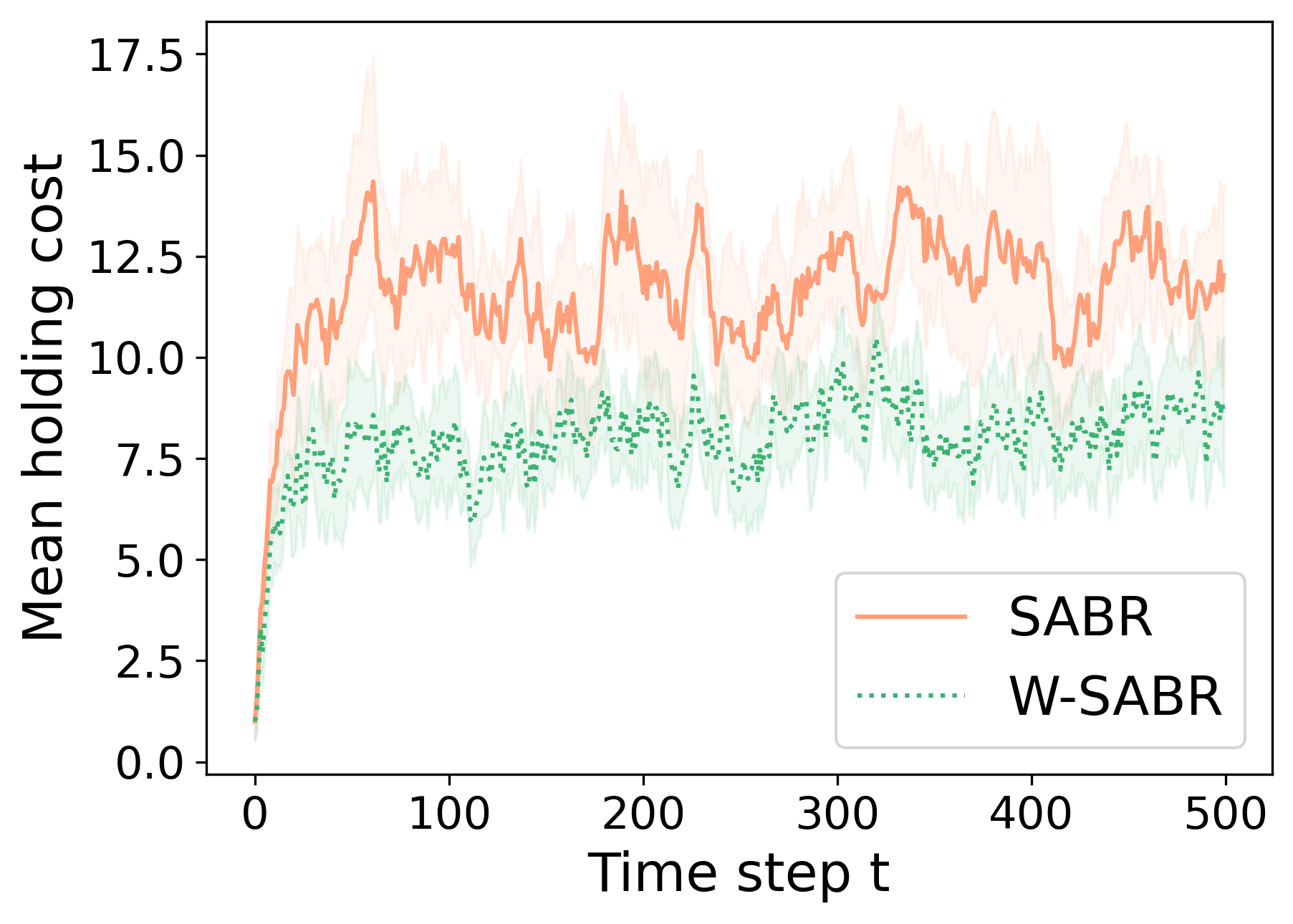}
     \vspace{1mm}
    (a)
\end{minipage}\hspace*{1.3cm}
\begin{minipage}[ht]{0.42\linewidth}
    \centering
    \includegraphics[width=\linewidth]{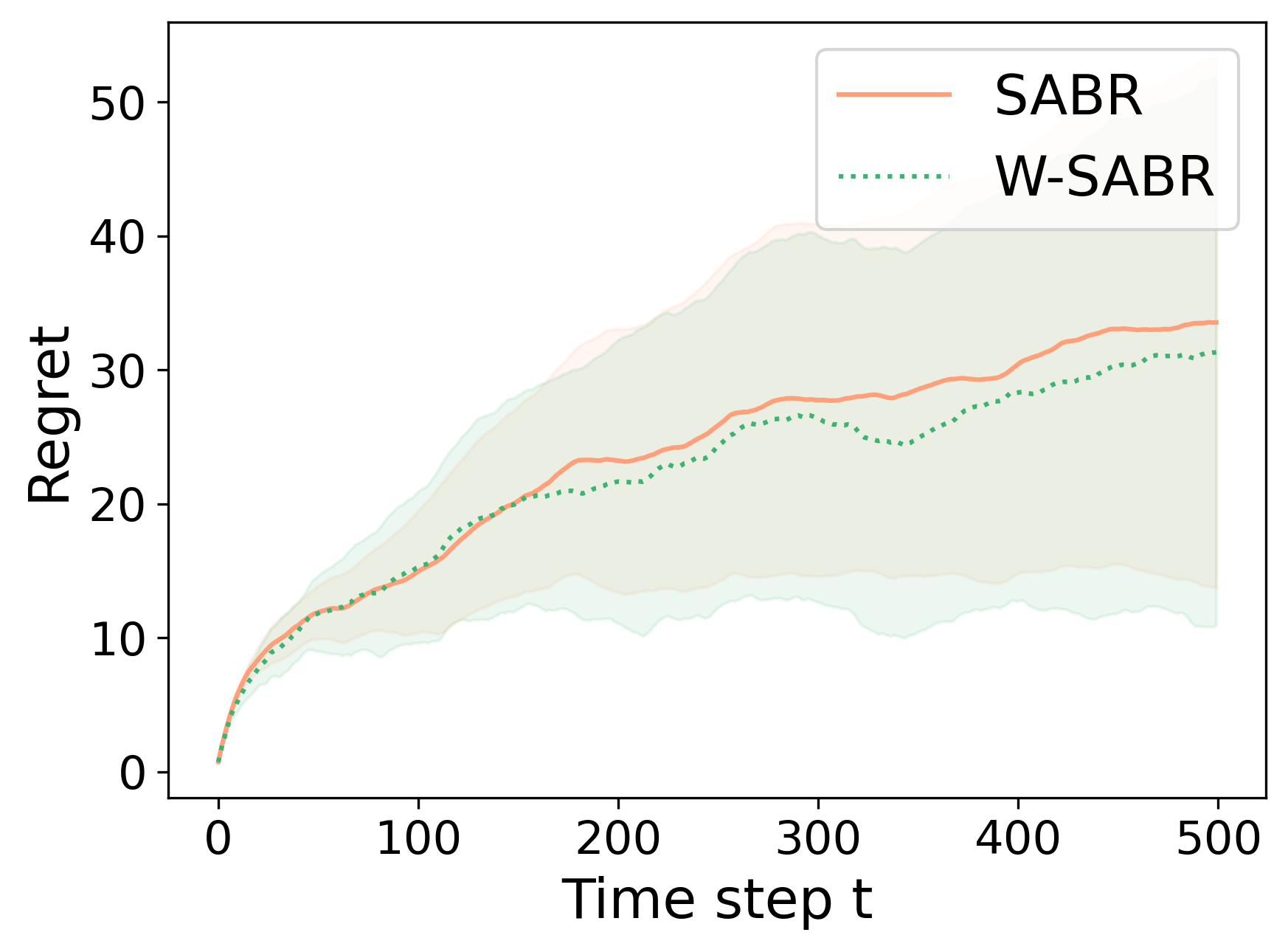}
     \vspace{1mm}
    (b)
\end{minipage}

\vspace{0.4cm}

\begin{minipage}[ht]{0.4\linewidth}
    \centering
    \includegraphics[width=\linewidth]{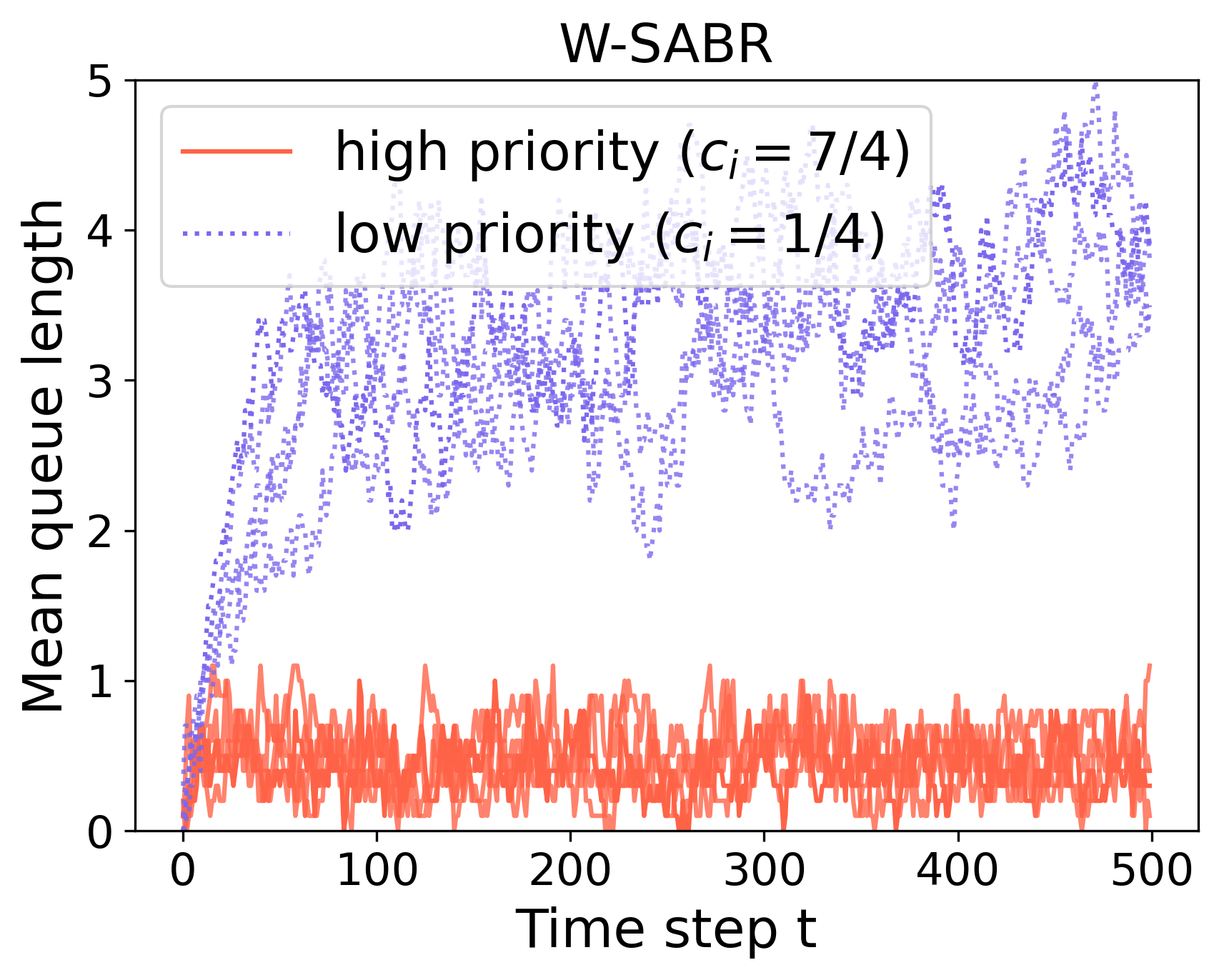}
     \vspace{1mm}
    (c)
\end{minipage}\hspace*{1.5cm}
\begin{minipage}[ht]{0.4\linewidth}
    \centering
    \includegraphics[width=\linewidth]{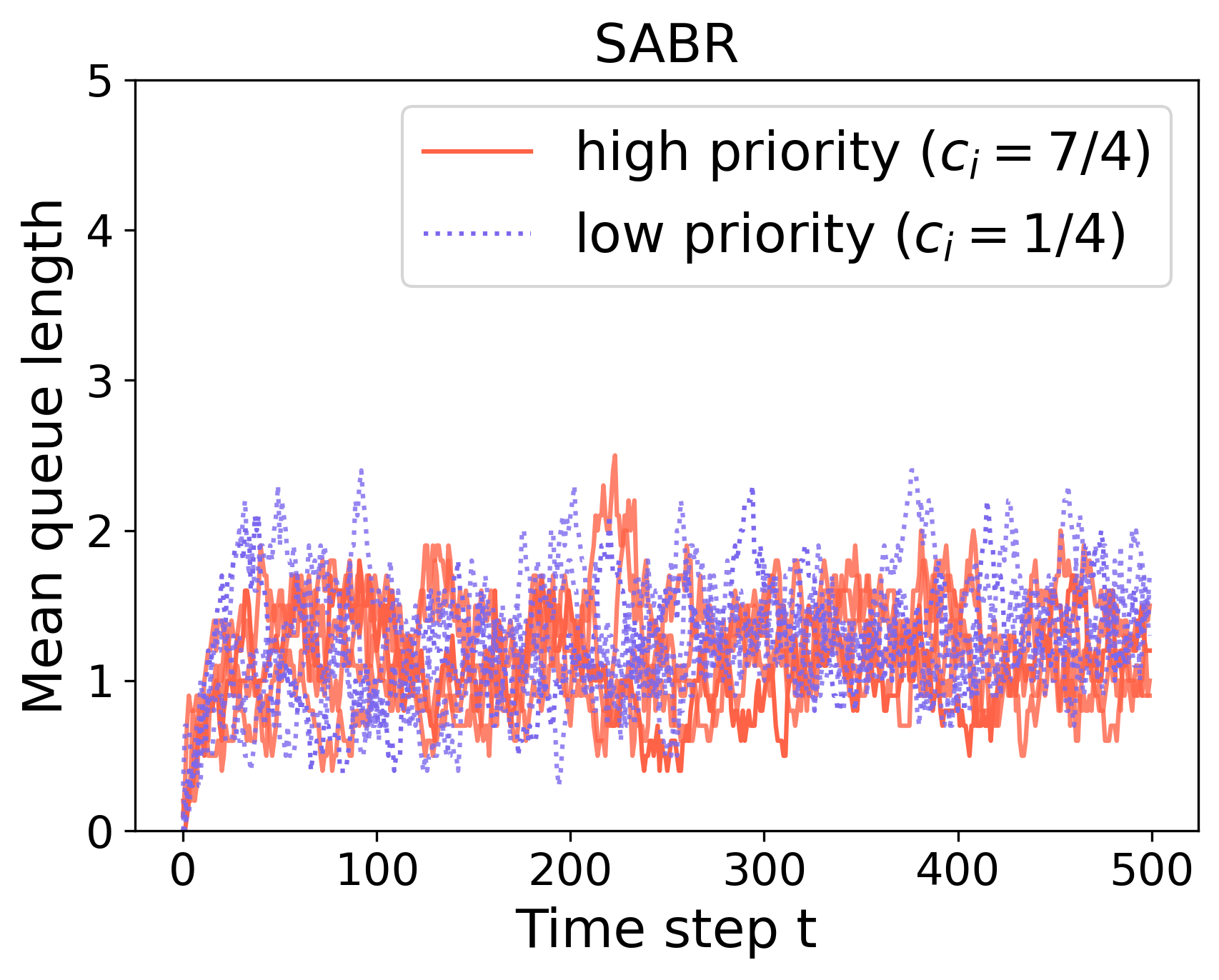}
      \vspace{1mm}
    (d)
\end{minipage}

\caption{Performance of \texttt{SABR} and/or \texttt{W-SABR} over time steps: (a) mean holding cost, (b) regret, (c) mean queue length for each job class under \texttt{W-SABR}, and (d) mean queue length for each job class under \texttt{SABR}.}
\label{fig:hold_cost}
\end{figure*}
A point $y$ satisfying (\ref{equ:oc1})-(\ref{equ:oc3}) is said to be an \emph{interior point} if either (\ref{equ:oc1}) or (\ref{equ:oc2}) holds with strict inequalities. 

The iterative updates (\ref{equ:xdiff0}) are of gradient descent type as, when all feedback communication delays are equal to zero, we have $1-\lambda_{i,j}(r)/u_i'(y^\dag_i(r)) = (1/u_i'(y^\dag_i(r)))\partial f(y(r))/\partial y_{i,j}(r)$ where $f$ is the objective of the optimization problem (\ref{equ:opt2}).  

The system of delay differential equations (\ref{equ:xdiff}) and the optimization problem (\ref{equ:opt2}) correspond to a special instance of a joint routing and rate control problem formulation studied in \cite{kv05}, where $(i,j)$ is the index of a route, $i$ is the index of a source, $j$ is the index of a link, and each route $(i,j)$ passes through link $j$ with a cost function $C_j$. Additionally, there is a link with a fixed price per unit flow $\gamma - \tilde{r}_{i,j}\geq 0$ that is used exclusively by route $(i,j)$. See Figure~\ref{fig:bilinear2} for a graphical illustration.

\subsection{Allocation Computed by Server Nodes} Another distributed algorithm for computing a maximizer $y$ of the optimization problem (\ref{equ:opt2}) is defined by letting each server-node $j\in \J$ compute values $(y_{i,j}(r): i\in \I_+, r\geq 0)$ using iterative updates with the associated system of delay differential equations, for $i\in \I_+$ and $j\in \J$,
\begin{equation}
\frac{d}{dr}y_{i,j}(r) = \alpha_{i,j}\left(1-\frac{\lambda_{i,j}(r)}{u'_i(y^\dag_i(r-\tau_{i,j}))}\right)^+_{y_{i,j}(r)}, 
\label{equ:xdiff2}
\end{equation}
where $
\lambda_{i,j}(r)= p_j\left(y^{\S}_j(r)\right) + \gamma - \tilde{r}_{i,j}$,
$y^\dag_i(r) = \sum_{j'\in \J} y_{i,j'}(r-\tau_{j',i}),$ and $y^\S_j(r) = \sum_{i'\in \I_+}y_{i',j}(r-\tau_{(i',j)}).$ Here, $y^\S_j(r)$ represents the total allocation given by server class $j$, acknowledged to be received by job-nodes $i\in \I_+$ via feedback sent to server-node $j$.

\subsection{Stability Condition} 
We provide a condition that ensures convergence of $(y^\dag(r), y^\S(r))$ to a unique point corresponding to the solution of optimization problem (\ref{equ:opt2}) as $r$ approaches infinity. Here, $(y^\dag(r), y^\S(r))$ evolves according to either the system of delay differential equations (\ref{equ:xdiff}) or (\ref{equ:xdiff2}). We define $\tau_{\max}$ as an upper bound on the round-trip delay for each job-server class combination, i.e., $\tau_{(i,j)}\leq \tau_{\max}$ for all $i\in\I_+$ and $j\in \J$.

\begin{theorem} Assume that $y^*$ is an interior equilibrium point, ${y^*}^\dag :=( \sum_{j\in \J}y_{i,j}^*: i\in \I_+)$, ${y^*}^\S := (\sum_{i\in \I_+}y_{i,j}^*: j\in \J)$, and that the following condition holds: for all $i\in \I_+$ and $j\in \J$,
\begin{equation}
\alpha_{i,j}\tau_{(i,j)}\left(1+\frac{p_j'({y_j^*}^\S){y_j^*}^\S}{p_j({y_j^*}^\S) + \gamma - \tilde{r}_{i,j}}\right)<\frac{\pi}{2}.
\label{equ:cond}
\end{equation}
Then, there exists a neighborhood $\mathcal{N}$ of $y^*$ such that for any initial trajectory $y(-\tau_{\max}),\ldots, y(0)$ lying within $\mathcal{N}$, $(y_i^\dag (r): i\in \I_+)$ and $(y_j^\S(r): j\in\J)$, evolving according to (\ref{equ:xdiff}) (resp. according to (\ref{equ:xdiff2})), converge exponentially fast, as $r$ goes to infinity, to the unique points ${y^*}^\dag$ and ${y^*}^\S$, respectively. 
\label{prop:cond}
\end{theorem}

For the system of delay differential equations (\ref{equ:xdiff}), the result in Theorem~\ref{prop:cond} follows from Theorem~2 in \cite{kv05}, as it corresponds to a special instance of the joint routing and rate control problem considered therein. The proof relies on the application of a generalized Nyquist stability criterion to a linearized system of delay differential equations. Similarly, for the system of delay differential equations (\ref{equ:xdiff2}), the proof follows the same approach, as detailed in the Appendix. 

Theorem~\ref{prop:cond} provides a sufficient condition for the exponentially fast convergence of $(y^\dag(r), y^\S(r))$ as $r$ approaches infinity. Intuitively, this condition requires the step size $\alpha_{i,j}$ to be sufficiently small relative to the reciprocal of the round-trip delay $\tau_{(i,j)}$, and involves terms related to the function $p_j$ and its derivative, as well as the marginal price $\gamma - \tilde{r}_{i,j}$, for each job-server class pair $(i,j)$. Since $\gamma - \tilde{r}_{i,j} \geq 0$, we can strengthen the sufficient condition by replacing $\gamma - \tilde{r}_{i,j}$ in (\ref{equ:cond}) with $0$. 

For concreteness, we discuss the sufficient condition from Theorem~\ref{prop:cond} for specific classes of penalty functions $C_j$.

First, consider the penalty functions $C_j$ such that $C_j'(z) = p_j(z) = (z/n_j)^{\beta_j}$, where $\beta_j >0$ for all $j\in \J$. Intuitively, the larger the value of parameter $\beta_j$, the closer the penalty function $C_j$ to the hard capacity constraint. From Theorem~\ref{prop:cond}, we derive the following sufficient stability condition:
$$
\alpha_{i,j}\tau_{(i,j)} < \frac{\pi}{2}\frac{1}{1+\beta_j}, \quad \text{for all } i\in \I_+, j\in \J.
$$
This condition implies that the step size $\alpha_{i,j}$ must be smaller than a constant factor of $1/\tau_{(i,j)}$, with the magnitude of this factor decreasing as $\beta_j$ increases.

Second, we can accommodate a broader set of penalty functions $C_j$ such that $p_j'(z)z \leq \beta_j p_j(z)$ and $p_j(z)\leq \gamma_j$ for all $z\geq 0$, where $\beta_j > 0$, $\gamma_j > 0$, and $\gamma - \tilde{r}_{i,j} \geq c$ for all $i\in \I_+$, $j\in \J$, where $c\geq 0$. Then, we have the sufficient stability condition: for all $i\in \I_+$ and $j\in \J$,
$$
\alpha_{i,j}\tau_{(i,j)} < \frac{\pi}{2}\frac{\gamma_j + c}{\gamma_j + \gamma_j \beta_j + c}. 
$$

Finally, consider the penalty function $C_j$ with $p_j(z) = (z/n_j)/(1-z/n_j)$, for $z\geq 0$, for all $j\in \J$. This function has a vertical asymptote at $z=n_j$. %
Assume $y^*$ satisfies ${y_j^*}^\S/n_j \leq 1-\epsilon$ for all $j\in \J$, where $\epsilon \in (0,1)$. Then, the following sufficient stability condition holds: for all $i\in \I_+$ and $j\in \J$,
$$
\alpha_{i,j}\tau_{(i,j)} < \frac{\pi}{4}\epsilon.
$$

\section{Numerical Results}
\label{sec:num}

In this section, we present the results of our numerical experiments. The aim of these experiments is to demonstrate the performance of our proposed algorithm and compare it with that achieved by the algorithm proposed by \cite{hsu2021integrated}, referred to as \texttt{UGDA-OL} (Utility-Guided Dynamic Assignment with Online Learning). We refer to Algorithm~\ref{alg:Alg1} as \texttt{SABR} (Scheduling Algorithm for Bilinear Rewards) when the weight parameters of the weighted proportional fair allocation criteria are identical and set to the value $1$, and as \texttt{W-SABR} (Weighted Scheduling Algorithm for Bilinear Rewards) when the weight parameters are specified to take some other values. As we will see, the experimental results validate our theoretical findings. 

\subsection{Setup of Experiments}

We consider randomly generated problem instances, enabling us to vary certain parameters to evaluate their effects on the regret and the mean queue length achieved by a scheduling algorithm. Each experiment was conducted with 10 independent repetitions. Additionally, we conducted experiments using problem instances generated from a real-world data trace obtained from a large-scale cluster computing system; these results are presented in the Appendix~\ref{app:real}.

Our basic setup of synthetic experiments is as follows. We consider identical traffic intensities over job classes, $\rho_i=\rho/I$ for all $i\in\I$, and identical number of servers over server classes, $n_j = n/J$ for all $j\in \J$, with the total number of servers $n$. Specifically, we assume $\rho = 1$ and $n=4$, resulting in the system load of $\rho/n = 0.25$. 
We set $T=500$, $1/\mu=1$, $I=10$, $J=2$, and $d=2$. The mean rewards follow the bilinear model with feature vectors ${u_i,\ i\in \I}$ and ${v_j,\ j\in \J}$, with the values of features set to independent samples from uniform distribution on $[0,1]$, and then normalized such that $\|u_i\|_2=1$ for all $i\in \I$ and $\|v_j\|_2=1$ for all $j\in \J$. The elements of the unknown parameter $\theta$ are set to independent samples from uniform distribution on $[0,1]$, and then normalized such that $\|\theta\|_2=1$. Stochastic rewards have independent additive Gaussian noises with mean zero and variance $0.01$. We set the value of parameter $\gamma$ to $1.2$. The value of parameter $V$ is chosen to minimize the regret bound for a given time horizon $T$.

\subsection{Results}

\paragraph{Comparison with UGDA-OL.}

Figure~\ref{fig:com} compares \texttt{SABR} (Algorithm~\ref{alg:Alg1}) with \texttt{UGDA-OL}, an algorithm based on \cite{hsu2021integrated} that achieves a regret bound of $\tilde{O}(\sqrt{IT} + JT + 1/\delta)$.  
Our results show that \texttt{SABR} consistently achieves lower regret. This improvement aligns with our theoretical analysis, where the regret of \texttt{SABR} scales as $\tilde{O}(\sqrt{IT}+d^2\sqrt{T}+1/\delta)$, which is sublinear in $T$, in contrast to the linear-in-$T$ term $\tilde{O}(JT)$ in \texttt{UGDA-OL}.

\paragraph{Impact of Weighted Scheduling.}
Figure~\ref{fig:hold_cost} compares the performance of \texttt{SABR} and \texttt{W-SABR} under heterogeneous holding costs. The costs are set to $7/4$ for half of the job classes (high-priority) and $1/4$ for the remaining classes (low-priority). In \texttt{W-SABR}, the weight parameters are matched to the cost parameters, i.e., $w_i = c_i$.

Figure~\ref{fig:hold_cost}(a) shows that \texttt{W-SABR} achieves lower overall holding costs compared to \texttt{SABR}, consistent with the bound in Theorem~\ref{thm:Q_len_bound_cost}, which suggests that properly tuning the weights can mitigate the cost dependence in the long-term holding cost. Meanwhile, Figure~\ref{fig:hold_cost}(b) demonstrates that both algorithms have similar regret performance, indicating that incorporating cost-aware weighting does not adversely affect learning efficiency.

Figures~\ref{fig:hold_cost}(c) and (d) depict the mean queue lengths for each job class under \texttt{W-SABR} and \texttt{SABR}, respectively. As expected, \texttt{W-SABR} prioritizes high-cost jobs by maintaining shorter queue lengths for those classes, thereby reducing the overall holding cost.  In contrast, \texttt{SABR}, which uses uniform weights, balances the queue lengths
more evenly across job classes without accounting for holding-cost
differences.

\section{Conclusion}
\label{sec:conc}
{
We studied scheduling in parallel-server queueing systems with stochastic
bilinear rewards. The proposed algorithm combines weighted proportional-fair
allocation with bilinear bandit learning. We established a regret--stability
tradeoff for the proposed algorithm. For any fixed control parameter \(V>0\),
the algorithm admits a uniform expected queue length and uniform time-averaging holding-cost bounds. When a finite target horizon \(T\) is specified in advance, choosing
\(V_T=\Theta(\sqrt{IT})\) yields sublinear regret, while maintaining uniform expected queue length and uniform time-averaged expected holding-cost bounds. Thus, the bilinear reward structure improves the
learning component of regret while preserving queueing and holding cost stability.
}

\section*{Acknowledgment}
This work was supported by the Systems for ML research award from Meta Research.
JK was supported in part by the French National Research
Agency (ANR) through the PEPR AI FOUNDRY project
(ANR-23PEIA-0003) and the DOOM project (ANR-23-CE23-0002),
and by the European Research Council (ERC) through the OCEAN
project (ERC-2022-SYG-OCEAN-101071601).

\bibliography{mybib}
\bibliographystyle{ieeetr}

\newpage
\onecolumn
\appendix
\section{Appendix}
\subsection{Proof of Theorem~\ref{thm:regbd1}}
\label{sec:regb1}

We first provide an outline of the proof to highlight the main steps, followed by the proof of the theorem. Our proof is adapted from the framework of \cite{hsu2021integrated}. For simplicity, we use the notation $c_\gamma=(\gamma+1)/(\gamma-1)$. \red{Recall that $\I(t)=\{i\in \I: Q_i(t) > 0\}$.}

\subsubsection{Proof outline}

The proof is based on decomposing the regret into different components, resulting in the following regret bound:
\begin{align}
R(T)
&\le \gamma \frac{1}{\mu}\mathbb{E}[Q(T)]+\frac{1}{V} \sum_{t=1}^T\mathbb{E}\left[G(t)\right]+\frac{1}{V}\left(\sum_{t=1}^T \mathbb{E}\left[H(t)\right]+\frac{1}{2}\left(\sum_{i\in \I}\we_i\right)(T+1)\right),
\label{eq:R_bound2}
\end{align}
where $$
G(t)=\sum_{i\in \I(t)}\sum_{j\in \J}\left(\frac{\we_i}{\rho_i}Q_i(t)+V(r_{i,j}-\gamma)\right)\left(\rho p_{i,j}^*-y_{i,j}(t)\right)
$$ 
and 
$$
H(t)=\frac{n^2}{2}\left(1+c_{\gamma}^2\mu\right)\left(\max_{i\in \I(t)}\frac{\we_i}{\rho_i}\right).
$$

To prove \eqref{eq:R_bound2}, we utilize the drift-plus-penalty method with the Lyapunov function defined as:
$$
L({q})=\frac{1}{2}\sum_{i\in \I}\frac{\we_i}{\rho_i} q_i^2.
$$
Let
\begin{align}
  \Delta(t)&=\sum_{i\in \I}\rho_i\sum_{j\in \J} p_{i,j}^*(r_{i,j}-\gamma)-\sum_{i\in \I(t)}\sum_{j\in \J}y_{i,j}(t)(r_{i,j}-\gamma).
\end{align}
By analyzing the drift-plus-penalty function, $L({Q}(t+1))-L({Q}(t))+V\mu\Delta(t)$, we obtain
\begin{align}
\mathbb{E}\left[L({Q}(t+1))-L({Q}(t))+V\mu\Delta(t)\right]\le \mu\mathbb{E}\left[G(t)\right]+\mu\mathbb{E}[H(t)]+\frac{\mu }{2}\sum_{i\in\I}\we_i,
\label{eq:LD2b}
\end{align} 
from which \eqref{eq:R_bound2} easily follows. 

The regret bound in \eqref{eq:R_bound2} comprises three components: the first component is proportional to the mean queue length, the second arises from the bandit learning algorithm, and the third stems from the stochastic nature of job arrivals and departures.
 
The term $G(t)$ is pivotal in bounding the effect of the bandit learning algorithm on regret. Let $(y^*_{i,j}(t): i\in\I(t), j\in \J)$ denote the solution of the optimization problem \eqref{eq:weighted} with parameters $\hat{r}_{i,j}$ replaced with the true mean values $r_{i,j}$. Then, we have
$$
G(t)\le G_1(t)+G_2(t),
$$
 where 
 $
 G_1(t)=V\sum_{i\in\I(t)}\sum_{j\in \J}(\tilde{r}_{i,j}(t)-r_{i,j})y_{i,j}(t)
 $
 and 
 $
 G_2(t)=V\sum_{i\in\I(t)}\sum_{ j\in \J}(r_{i,j}-\tilde{r}_{i,j}(t))y^*_{i,j}(t).
 $ 
It is noteworthy that $G_1(t)$ and $G_2(t)$ represent weighted sums of mean reward estimation errors.

To bound the weighted sums of mean reward estimation errors, we evaluate the error of the estimator of $\theta$ using a weighted norm. Let $x_{\cl(k),j}(t)=y_{\cl(k),j}(t)/Q_{\cl(k)}(t)$ and $\tilde{x}_{\cl(k),j}(t)$ represent the actual number of servers of class $j$ assigned to job $k$ at time step $t$, such that $\mathbb{E}[\tilde{x}_{\cl(k),j}(t)\mid x_{\cl(k),j}(t)]=x_{\cl(k),j}(t)$, and $\tilde{\Q}(t)$ denote the set of assigned jobs in $\Q(t)$ to servers at time $t$. Then, by defining $\tilde{\theta}_{i,j}(t)=\arg\max_{\theta'\in \mathcal{C}(t)}z_{i,j}^\top \theta'$, we can establish the following:
\begin{align}
    \frac{1}{V}\sum_{t=1}^T\mathbb{E}\left[G_1(t)\right]&=\sum_{t=1}^T\mathbb{E}\left[\sum_{i\in\I(t),j\in J}(\tilde{r}_{i,j}(t)-r_{i,j})y_{i,j}(t)\right]\cr
    &=\sum_{t=1}^T\mathbb{E}\left[\sum_{k\in\Q(t),j\in J}(\tilde{r}_{\cl(k),j}(t)-r_{\cl(k),j})x_{\cl(k),j}(t)\right]
    \cr&=\sum_{t=1}^T\mathbb{E}\left[\sum_{k\in \tilde{\Q}(t), j \in \J}(\tilde{r}_{\cl(k),j}(t)-r_{\cl(k),j})\tilde{x}_{\cl(k),j}(t)\right].
\label{eq:G1_bd1_1}
\end{align}
By conditioning on the event $\{\theta\in {\mathcal C}(t) $ for $ t\in \mathcal{T}\}$, which holds with high probability, we have:
\begin{align}
    &\sum_{t=1}^T\sum_{k\in \tilde{\Q}(t), j \in \J}(\tilde{r}_{\cl(k),j}(t)-r_{\cl(k),j})\tilde{x}_{\cl(k),j}(t) \cr
    &\le \sum_{t=1}^T\sum_{k \in \tilde{\Q}(t),j\in  \J}\|z_{\cl(k),j}\|_{\Lambda(t)^{-1}}\cdot\|\tilde{\theta}_{\cl(k),j}(t)-\hat{\theta}_{\cl(k),j}(t)+\hat{\theta}_{\cl(k),j}(t)-\theta\|_{\Lambda(t)}\tilde{x}_{\cl(k),j}(t)\cr 
    &\le\sum_{t=1}^T \sum_{k\in \tilde{\Q}(t), j \in  \J} 2\|z_{\cl(k),j}\|_{\Lambda(t)^{-1}}{\beta(t)}\tilde{x}_{\cl(k),j}(t),\label{eq:G1_bd1_2}
\end{align}
where the second inequality is obtained by using the fact $\theta\in {\mathcal C}(t)$. 
We can also establish:
\begin{align}
    \sum_{t=1}^T\sum_{k\in\tilde{\Q}(t),j\in\J}\tilde{x}_{\cl(k),j}(t)\|z_{\cl(k),j}\|_{\Lambda(t)^{-1}}^2\le 2d^2\log\left(n+\frac{n}{d^2}T\right).
    \label{eq:sum_w_bd}
\end{align}

 Recalling that $\beta(t)= \sqrt{d^2\log(tT)}+\sqrt{n}$, we utilize the above inequalities, the Cauchy-Schwarz inequality, and \eqref{eq:sum_w_bd} to derive:
$$
\frac{1}{V}\sum_{t=1}^T\mathbb{E}\left[G_1(t)\right]\le \mathbb{E}\left[\beta(T)\sqrt{nT\sum_{t=1}^T\sum_{k\in\tilde{\Q}(t),j\in\J}\tilde{x}_{\cl(k),j}(t)\|z_{\cl(k),j}\|_{\Lambda(t)^{-1}}^2}\right]+2 n=\tilde{O}((d^2\sqrt{n}+dn)\sqrt{T}).
$$  
Additionally, it can be shown that $(1/V)\sum_{t=1}^T\mathbb{E}[G_2(t)]=O(1)$. Thus, combining these results yields: 
 \begin{align}
     \frac{1}{V}\sum_{t=1}^T\mathbb{E}[G_1(t)]+\frac{1}{V}\sum_{t=1}^T\mathbb{E}[G_2(t)]=\tilde{O}((d^2\sqrt{n}+dn)\sqrt{T}).
     \label{eq:G_bound}
 \end{align}
Moreover, as $(1/V)\sum_{t=1}^T G_2(t)$ is negative with high probability, $(1/V)\sum_{t=1}^T\mathbb{E}[G_1(t)]$ dominates $(1/V)\sum_{t=1}^T\mathbb{E}[G_2(t)]$.

Finally, by utilizing \eqref{eq:R_bound2}, \eqref{eq:G_bound}, and the mean queue length bound derived from Theorem~\ref{thm:Q_len_bound}, we obtain the regret bound as asserted in Theorem~\ref{thm:regbd1}.

\subsubsection{Proof of the theorem}

The proof uses a regret bound that has three components and then proceeds with separately bounding these components.
The first component is proportional to the mean queue length. The second component is due to the bandit learning algorithm. This term is bounded by leveraging the bilinear structure of rewards. The third component is due to randomness of job arrivals and departures. In the following lemma, we provide a regret bound that consists of the three aforementioned components.

\begin{lemma} The regret is bounded as follows:
\begin{align*}
R(T)%
&\le \gamma \frac{1}{\mu}\mathbb{E}[Q(T)]+\frac{1}{V} \sum_{t=1}^T\mathbb{E}\left[G(t)\right]+\frac{1}{V}\left(\sum_{t=1}^T\mathbb{E}\left[H(t)\right]+\frac{1}{2}\left(\sum_{i\in\I} \we_i\right)(T+1)\right),
\end{align*} \label{lem:R_bd_lam}
\end{lemma} 
where 
$$
G(t)=\sum_{i\in \I(t)}\sum_{j\in \J}\left(\frac{\we_i}{\rho_i}Q_i(t)+V(r_{i,j}-\gamma)\right)\left(\rho_i p_{i,j}^*-y_{i,j}(t)\right)
$$ 
and 
$$
H(t)=\frac{n^2}{2}\left(1+c_{\gamma}^2\mu\right)\left(\max_{i\in \I(t)}\frac{\we_i}{\rho_i}\right).
$$

\begin{proof}
The queues of different job classes $i\in \I$ evolve as
\begin{align*}
Q_i(t+1)=Q_i(t)+A_i(t+1)-D_i(t),
\end{align*}
where $A_i(t+1)$ and $D_i(t)$ are the number of class-$i$ job arrivals at the beginning of time step $t+1$ and the number of class-$i$ job departures at the end of time step $t$, respectively. Let $A(t)=\sum_{i\in \I}A_i(t)$ and $D(t)=\sum_{i\in \I}D_i(t).$

We use the Lyapunov function defined as
$$
L({q})=\frac{1}{2}\sum_{i\in \I}\frac{\we_i}{\rho_{i}}q_i^2.
$$

The following conditional expected drift equations hold for queues of different job classes: if $i\notin \I(t)$,
$$
\mathbb{E}[Q_i(t+1)^2-Q_i(t)^2\mid Q_i(t)=0]=\mathbb{E}[A_i(t+1)^2]=\lambda_i,
$$
and, otherwise, if $i\in \I(t)$,
\begin{align}
    &\mathbb{E}[Q_i(t+1)^2-Q_i(t)^2\mid \Q(t), {x}(t)]\cr
    &\le \mathbb{E}[2Q_i(t)(A_i(t+1)-D_i(t))+(A_i(t+1)-D_i(t))^2)\mid \Q(t), {x}(t)]
    \cr
    &\le 2Q_i(t)\left(\lambda_i-\mathbb{E}[D_i(t)\mid  \Q(t), {x}(t)]\right)+\lambda_i+\mathbb{E}[D_i(t)^2\mid \Q(t), {x}(t)].
    \label{eq:Q^2_gap_bd_lam}
\end{align}

We next derive bounds for $\mathbb{E}[D_i(t)\mid \Q(t),{x}(t)]$ and $\mathbb{E}[D_i(t)^2\mid \Q(t),{x}(t)]$ in the following lemma. 

\begin{lemma} For any $i\in \I(t)$, we have
\begin{align*}
    \mathbb{E}[D_i(t)\mid \Q(t),y(t)]\ge \mu\sum_{j\in \J}y_{i,j}(t)-\frac{\mu^2n^2(\gamma+1)^2}{2(\gamma-1)^2}\frac{\we_{i}^{2}Q_i(t)}{\sum_{i^\prime\in \I(t)}\we_{i^\prime}^{2}Q_{i^\prime}(t)^2}.
\end{align*} and 
\begin{align*}
    \sum_{i^\prime\in\I}\mathbb{E}[D_{i^\prime}(t)^2\mid \Q(t),y(t)]\le n^2\mu.
\end{align*}
\label{lem:D_bd}
\end{lemma}

\begin{proof}
Let $E_i(t)$ be the event that job $k\in \Q_i(t)$ for $i\in \I(t)$ is not completed at the end of time step $t$. A server of class $j$ is assigned job $k$ with probability $y_{i,j}(t)/(n_jQ_{i}(t))$, and this job is completed with probability $\mu$ by the memory-less property of geometric distribution. Therefore, we have
\begin{align}
    \mathbb{P}[E_i(t)\mid\Q(t),y(t)]=\prod_{j\in \J}\left(1-\mu\frac{y_{i,j}(t)}{n_jQ_{i}(t)}\right)^{n_j},  \label{eq:E}
\end{align}
 and $$ \mathbb{E}[D_i(t)\mid {\Q}(t),y(t)]=\sum_{k\in \Q_i(t)}\left(1-\mathbb{P}[E_i(t)\mid \Q(t),y(t)]\right)=Q_i(t)\left(1-\mathbb{P}[E_i(t)\mid \Q(t),y(t)]\right).$$

Using $1-x\le e^{-x}\le 1-x+x^2/2$ for $x\ge 0$, we have 
\begin{eqnarray*}
    1-\mathbb{P}[E_i(t)\mid \Q(t),y(t)] &\ge & 1-\exp\left(-\sum_{j\in \J}\mu \frac{y_{i,j}(t)}{Q_i(t)}\right) \\ & \ge & 
    \mu\sum_{j\in\J}\frac{y_{i,j}(t)}{Q_i(t)} -\frac{\mu^2}{2}\left(\sum_{j \in \J}\frac{y_{i,j}(t)}{Q_i(t)}\right)^2.
\end{eqnarray*}
Hence, for any $i\in\I(t)$ we have \begin{align}
\mathbb{E}[D_i(t)\mid \Q(t),y(t)]\ge \mu \sum_{j \in \J}y_{i,j}(t)-\frac{\mu^2}{2}\frac{1}{Q_i(t)}\left(\sum_{j\in \J}y_{i,j}(t)\right)^2. \label{eq:D_lowbd_1}  
\end{align}

Let $q(t)=(q_j(t):j\in\J)\in \mathbb{R}_+^J$ be the Lagrange multipliers for the constraints $\sum_{i\in\I(t)}y_{i,j}(t)\le n_j$ for all $j\in\J$ and $h(t)=(h_{i,j}(t):i\in\I(t),j\in\J)\in \mathbb{R}_+^{|\I(t)|\times J}$ be the Lagrange multipliers for the constraints $y_{i,j}(t)\ge 0$ for all $i\in \I(t)$ and $j\in \J$ in $\eqref{eq:weighted}$. Then, we have the Lagrangian function for the optimization problem \eqref{eq:weighted} given as
\begin{align*}
\mathcal{L}(y(t), q(t), h(t))&=\sum_{i\in \I(t)}\left[\frac{1}{V} Q_i(t) \we_i\log\left(\sum_{j\in \J}y_{i,j}(t)\right)+\sum_{j\in\J}y_{i,j}(t)(\tilde{r}_{i,j}(t)-\gamma)\right]
    \cr &\qquad +\sum_{j\in \J}q_j(t)\left(n_j-\sum_{i\in\I(t)}y_{i,j}(t)\right)+\sum_{i\in\I(t), j\in\J}h_{i,j}(t)y_{i,j}(t).
\end{align*} 

If $y(t)$ is a solution of \eqref{eq:weighted}, then $y(t)$ satisfies the following stationarity conditions,
\begin{equation}
    \sum_{j^\prime\in\J}y_{i,j^\prime}(t)=\frac{1}{V}\frac{\we_i Q_i(t)}{q_j(t)-h_{i,j}(t)+\gamma-\tilde{r}_{i,j}(t)}, \hbox{ for all } i \in \I(t),
    \label{eq:L2_stationary}
\end{equation}
and the following complementary slackness conditions, 
\begin{equation}
q_j(t)\left(n_j-\sum_{i\in\I(t)}y_{i,j}(t)\right)=0, \hbox{ for all } j \in \J \label{eq:L2_com1} 
\end{equation}
and
\begin{equation}
h_{i,j}(t)y_{i,j}(t)=0, \hbox{ for all } i\in \I(t), j\in \J.   \label{eq:L2_com2} 
\end{equation}

The convex optimization problem in \eqref{eq:weighted}, with $n_j>0$ for all $j\in \J$, always has a feasible solution. This implies that for any $i\in \I(t)$, there exists $j_0\in \J$ such that $y_{i,j_0}(t)>0$ since $\sum_{j\in \J}y_{i,j}(t)>0$ from $\log(\sum_{j\in \J}y_{i,j}(t))$ in \eqref{eq:weighted}. This implies that $h_{i,j_0}(t)=0$ from the complementary slackness conditions \eqref{eq:L2_com2}. 
 Consider any two job classes $i$ and $i^\prime$ in $\I(t)$. We have
\begin{align}
q_j(t)+\gamma-\tilde{r}_{i,j}(t)\ge q_j(t)+\gamma-1\ge \frac{\gamma-1}{\gamma+1}(q_j(t)+\gamma-\hat{r}_{i^\prime,j}(t)).
\label{eq:Lag_Mul_lowbd}
\end{align}
From \eqref{eq:L2_stationary}, \eqref{eq:L2_com2}, and \eqref{eq:Lag_Mul_lowbd}, we have
\begin{align}
    \sum_{j\in \J}y_{i,j}(t)&=\frac{\we_i Q_i(t)}{V(q_{j_0}(t)-h_{i,j_0}(t)+\gamma-\hat{r}_{i,j_0}(t))}\cr
    &=\frac{\we_i Q_i(t)}{V(q_{j_0}(t)+\gamma-\hat{r}_{i,j_0}(t))}\cr
    &\le\frac{\we_i Q_i(t)(\gamma+1)/(\gamma-1)}{V(q_{j_0}(t)+\gamma-\hat{r}_{i^\prime,j_0}(t))}\cr
    &\le\frac{\we_i Q_i(t)}{\we_{i^\prime}Q_{i^\prime}(t)}\frac{\we_{i^\prime} Q_{i^\prime}(t)(\gamma+1)/(\gamma-1)}{V(q_{j_0}(t)-h_{i^\prime,j_0}(t)+\gamma-\hat{r}_{i^\prime,j_0}(t))}\cr
    &\le\frac{\we_i Q_i(t)}{\we_{i^\prime} Q_{i^\prime}(t)}\frac{\gamma+1}{\gamma-1}\sum_{j\in \J}y_{i^\prime,j}(t).\label{eq:x_sim_lam}
\end{align}
From \eqref{eq:x_sim_lam}, we have 
\begin{align}
    n&\ge 
    \sum_{i^\prime\in \I(t)}\sum_{j\in \J}y_{i^\prime,j}(t)\cr &\ge ((\gamma-1)/(\gamma+1))\sum_{i^\prime \in \I(t)}\frac{\we_{i^\prime} Q_{i^\prime}(t)}{\we_{i} Q_i(t)}\sum_{j\in \J}y_{i,j}(t).\label{eq:x_upbd_traffic}
\end{align}
Then, from  \eqref{eq:D_lowbd_1} and \eqref{eq:x_upbd_traffic}, we obtain
\begin{align*}
\mathbb{E}[D_i(t)\mid \Q(t),y(t)]&\ge \mu\sum_{j \in \J}y_{i,j}(t)-\frac{\mu^2}{2}\frac{1}{Q_i(t)}\left(\sum_{j\in \J}y_{i,j}(t)\right)^2 \cr &\ge \mu\sum_{j\in \J}y_{i,j}(t)-\frac{\mu^2n^2(\gamma+1)^2}{2(\gamma-1)^2}\frac{\we_{i}^{2} Q_i(t)}{\sum_{i^\prime\in \I(t)}\we_{i^\prime}^{2}Q_{i^\prime}(t)^2}.
\end{align*}

Applying $(1-x)(1-y)\ge 1-(x+y)$ for all $x,y\geq 0$, from \eqref{eq:E} we obtain
\begin{align}
    \mathbb{E}[D_i(t)\mid \Q(t),y(t)]=Q_i(t)\left(1-\mathbb{P}[E_i(t)\mid \Q(t),y(t)]\right)\le \mu \sum_{j \in J}y_{i,j}(t).\label{eq:D_bd_x}
\end{align}

From \eqref{eq:D_bd_x} and $D_i(t)\le\sum_{j\in \J}n_j=n$, we also have
$$
\sum_{i\in\I}\mathbb{E}[D_i(t)^2\mid \Q(t), {x}(t)]\le \sum_{i\in\I}\mathbb{E}[D_i(t)\mid  \Q(t), {x}(t)]n\le n^2\mu.
$$
\end{proof}

From \eqref{eq:Q^2_gap_bd_lam}, it follows 
 \begin{align*}
& \mathbb{E}[L({Q}(t+1))-L({Q}(t))\mid \Q(t), {x}(t)]\cr
&\le \sum_{i\in \I(t)}\left(\frac{\we_i}{\rho_i}Q_i(t) \mu\sum_{j\in \J}\left(\rho_i p_{i,j}^*-y_{i,j}(t)\right)+\frac{\we_i\lambda_i}{2\rho_i}+\frac{\we_i}{2\rho_i}\mathbb{E}[D_i(t)^2\mid \Q(t), {x}(t)]\right)\cr
&\qquad+\frac{(\gamma+1)^2n^2\mu^2 }{2(\gamma-1)^2  }\left(\max_{i\in \I(t)}\frac{\we_i}{\rho_i}\right)+\sum_{i\notin \I(t)}\frac{\we_i\lambda_i}{2\rho_{i}}\cr
&\le
\sum_{i\in \I(t)}\sum_{j\in \J}\left(\frac{\we_i}{\rho_i}Q_i(t)\mu\left(\rho_i p_{i,j}^*-y_{i,j}(t)\right)\right)+\frac{\mu\sum_{i\in\I}\we_i}{2}+\left(\frac{(\gamma+1)^2n^2\mu^2}{2(\gamma-1)^2}+\frac{n^2\mu}{2}\right)\left(\max_{i\in \I(t)}\frac{\we_i}{\rho_i}\right).
 \end{align*}

Note that, from \red{$r_{i,j}-\gamma<0$},
\begin{align*}
   \Delta(t)
   &=\sum_{i\in \I}\rho_i\sum_{j\in \J} p_{i,j}^*(r_{i,j}-\gamma)-\sum_{i\in \I(t)}\sum_{j\in \J}y_{i,j}(t)(r_{i,j}-\gamma)\cr 
   &\le\sum_{i\in \I(t)}\rho_i\sum_{j\in \J}p_{i,j}^*(r_{i,j}-\gamma)-\sum_{i\in \I(t)}\sum_{j\in \J} y_{i,j}(t)(r_{i,j}-\gamma)\cr
   &= \sum_{i\in \I(t)}\sum_{j\in \J} (r_{i,j}-\gamma)(\rho_i p_{i,j}^* - y_{i,j}(t)). 
\end{align*}

It follows that
\begin{align}
\mathbb{E}\left[L({Q}(t+1))-L({Q}(t))+V\mu\Delta(t)\right]\le \mu \mathbb{E}\left[G(t)\right]+\mu \mathbb{E}\left[H(t)\right]+\mu \frac{1}{2}\sum_{i\in \I}\we_i,
\label{eq:LD_lam}
\end{align} 
where
\begin{align*}
    G(t)&=\sum_{i\in \I(t)}\sum_{j\in \J}\left(\frac{\we_i}{\rho_i}Q_i(t)+V(r_{i,j}-\gamma)\right)\left(\rho_i p_{i,j}^*-y_{i,j}(t)\right) 
\end{align*} 
and 
$$
H(t)=\frac{n^2}{2}\left(1+\left(\frac{\gamma+1}{\gamma-1}\right)^2\mu\right)\left(\max_{i\in \I(t)}\frac{\we_i}{\rho_i}\right).
$$
Now, note
\begin{eqnarray*}
\sum_{t=1}^T \left(\sum_{i\in \I}\rho_i-\mathbb{E}\left[\sum_{i\in \I(t)}\sum_{j\in \J} y_{i,j}(t)\right]\right) 
& \le & \sum_{t=1}^T\left(\rho- \frac{1}{\mu}\mathbb{E}[D(t)]\right)\\
&=& \frac{1}{\mu}\sum_{t=1}^T\mathbb{E}[A(t+1)-D(t)]\\
&=& \frac{1}{\mu}\sum_{t=1}^T\mathbb{E}[Q(t+1)-Q(t)]\\
&=& \frac{1}{\mu}(\mathbb{E}[Q(T+1)]-\mathbb{E}[Q(1)])\\
&=& \frac{1}{\mu}\mathbb{E}[Q(T)-D(T)+A(T+1)-A(1)]\\
&\le& \frac{1}{\mu}\mathbb{E}[Q(T)] 
\end{eqnarray*} 
where the first inequality is from \eqref{eq:D_bd_x}. Note also that 

\begin{align}
\sum_{t=1}^T \mathbb{E}[\Delta(t)]
&=\sum_{t=1}^T\left(\mathbb{E}\left[\sum_{i\in \I}\sum_{j \in \J}\rho_i r_{i,j}  p_{i,j}^*-\sum_{i\in \I(t)}\sum_{j \in \J} r_{i,j}y_{i,j}(t)\right]\right)-\gamma\sum_{t=1}^T \left(\rho-\mathbb{E}\left[\sum_{i\in \I(t)}\sum_{j \in \J} y_{i,j}(t)\right]\right)\cr 
&\ge\sum_{t=1}^T\left(\sum_{i\in \I}\sum_{j \in \J}\rho_i r_{i,j}  p_{i,j}^*-\mathbb{E}\left[\sum_{i\in \I(t)}\sum_{j \in \J} r_{i,j}y_{i,j}(t))\right]\right)-\gamma \frac{1}{\mu}\mathbb{E}[Q(T)]\cr 
&= R(T)-\gamma \frac{1}{\mu}\mathbb{E}[Q(T)].
\label{eq:LD_sub_lam}
\end{align}

From \eqref{eq:LD_lam} and \eqref{eq:LD_sub_lam}, and using the facts 
\begin{align*}
        \mathbb{E}[L({Q}(1))]=\frac{1}{2}\mathbb{E}\left[\sum_{i\in\I}\frac{\we_i}{\rho_i}Q_i(1)^2\right]
        =\frac{1}{2}\mathbb{E}\left[\sum_{i\in \I}\frac{\we_i}{\rho_i}A_i(1)^2\right]
        =\frac{\mu \sum_{i\in\I}\we_i}{2}
\end{align*}
and $L({Q}(T+1))\ge 0$, we have 
\begin{align*}
R(T)&=\sum_{t=1}^T\left(\sum_{i\in \I}\sum_{j \in \J}\rho_i r_{i,j}  p_{i,j}^*-\mathbb{E}\left[\sum_{i\in \I(t)}\sum_{j \in \J} r_{i,j}y_{i,j}(t)\right]\right)\cr 
&\le \gamma \frac{1}{\mu}\mathbb{E}[Q(T)]+\frac{1}{V}\sum_{t=1}^T \mathbb{E}[G(t)]+\frac{1}{V}\left(\sum_{t=1}^T\mathbb{E}[ H(t)]+(T+1)\frac{1}{2}\sum_{i\in\I}\we_i\right).
\end{align*}
\end{proof}

In what follows, we focus on bounding the regret component attributed to the bandit learning algorithm. Denote by $y^*(t)=(y^*_{i,j}(t):i\in\I(t), j\in \J)$ the solution of the optimization problem \eqref{eq:weighted} with the true mean rewards $r_{i,j}$ in the place of the mean reward estimates $\tilde{r}_{i,j}$. We employ a bound for $G(t)$ in terms of two variables quantifying the gap between true and estimated mean rewards, as provided in the following lemma. 

\begin{lemma} The following bound holds for all $t\ge 1$, %
$$
G(t)\le G_1(t)+G_2(t),
$$
 where 
 $$
 G_1(t)=V\sum_{i\in \I(t)}\sum_{j\in \J}(\tilde{r}_{i,j}(t)-r_{i,j})y_{i,j}(t),
 $$
 and 
 $$
 G_2(t)=V\sum_{i\in \I(t)}\sum_{ j\in \J}(r_{i,j}-\tilde{r}_{i,j}(t))y^*_{i,j}(t).
 $$\label{lem:G_G1G2_bd}
\end{lemma}
\begin{proof}
Denote ${y}(t)=({y}_{i,j}(t):j\in\J, i\in\I(t))$, ${\tilde{r}}(t)=({\tilde{r}}_{i,j}(t):j\in \J, i\in \I(t))$, $r=(r_{i,j}:i\in\I,j\in\J)$, and $\tilde{y}(t)=(\tilde{y}_{i,j}(t):i\in\I(t),j\in\J)$ where $\tilde{y}_{i,j}(t)=\rho_ip_{i,j}^*$. Let $h({y}(t)\mid \Q(t),r)=\sum_{i\in\I(t)}(V\sum_{j\in\J}(r-\gamma){y}_{i,j}(t)+Q_i(t)\we_i\log(\sum_{j\in\J}{y}_{i,j}(t)))$. Then from Lemma EC.1 in \cite{hsu2021integrated} we have
$$
\sum_{i\in\I(t)}\sum_{j\in \J}\left(\frac{\we_i}{\rho_i}Q_i(t)+V(r_{i,j}-\gamma)\right)\left(\rho_ip_{i,j}^*-{y}_{i,j}\right)\le h(\tilde{y}(t)\mid \Q(t),r)-h({y}(t)\mid \Q(t),r).
$$ 
From $h(\tilde{y}(t)\mid \Q(t),r)\le h(y^*(t)\mid \Q(t),r),$ we have 
$$
G(t)\le h(y^*(t)\mid \Q(t),r)-h(y(t)\mid \Q(t),r).$$ Then from $h(y^*(t)\mid \Q(t),r)=h(y^*(t)\mid \Q(t),\tilde{r}(t))+G_2(t)\le h(y(t)\mid \Q(t),\tilde{r}(t))+G_2(t)= h(y(t)\mid \Q(t),r)+G_1(t)+G_2(t),$ we have $G(t)\le G_1(t)+G_2(t),$ which concludes the proof.
\end{proof}

We will now present a key lemma for bounding the regret component attributed to the bilinear bandit learning algorithm. 

\begin{lemma} For any constant $\gamma>1$, we have
$$
\frac{1}{V}\sum_{t=1}^T\mathbb{E}[G(t)]=\tilde{O}((d^2\sqrt{n}+dn)\sqrt{T}).
$$
\label{lem:A_1}
\end{lemma}

\begin{proof} 
Recall that our bandit learning algorithm uses the confidence set ${\mathcal C}(t)$ for the parameter of the bilinear model in time step $t$, which is defined as follows
$$
{\mathcal C}(t)=\left\{\theta'\in\reals^{d^2}:\|\hat{\theta}(t)-\theta'\|_{\Lambda(t)}\le \beta(t)\right\},
$$
where $\beta(t) = \sqrt{d^2\log\left(tT\right)}+\sqrt{n}$.

It is known that ${\mathcal C}(t)$ has a good property for estimating the unknown parameter of the linear model, which is stated in the following lemma.

\begin{lemma}[Theorem 4.2. in \cite{qin2014contextual}]
 The true parameter value $\theta$ lies in the set ${\mathcal C}(t)$ for all $t\in \mathcal{T}$, with probability at least $1-1/T$.
\label{lem:confidence}
\end{lemma}

In the following two lemmas, we provide bounds for $(1/V)\sum_{t=1}^T\mathbb{E}[G_1(t)]$ and $(1/V)\sum_{t=1}^T\mathbb{E}[G_2(t)]$, respectively, from which the bound in Lemma~\ref{lem:G_G1G2_bd} follows.

\begin{lemma} The following bound holds
$$
\frac{1}{V}\sum_{t=1}^T\mathbb{E}[G_1(t)]=\tilde{O}((d^2\sqrt{n}+dn)\sqrt{T}).
$$ 
\label{lem:A_2}
\end{lemma}

\begin{proof}
 Recall that for $k\in Q(t)$,  $x_{\cl(k),j}(t)=y_{\cl(k),j}(t)/Q_{\cl(k)}(t)$, $\tilde{x}_{\cl(k),j}(t)$ is the actual number of servers of class $j$ assigned to job $k$ at time $t$ such that $\mathbb{E}[\tilde{x}_{\cl(k),j}(t)\mid x_{\cl(k),j}(t)]=x_{\cl(k),j}(t)$, and $\tilde{\Q}(t)$ is the set of assigned jobs in $\Q(t)$ to servers at time $t$. Let filtration $\mathcal{F}_{t-1}$ be the $\sigma$-algebra generated by random variables before time $t$. Then, we have 
\begin{align*}
    \frac{1}{V}\sum_{t=1}^T \mathbb{E}[G_1(t)]&=\sum_{t=1}^T\mathbb{E}\left[\sum_{i\in \I(t), j \in \J}(\tilde{r}_{i,j}(t)-r_{i,j})y_{i,j}(t)\right]\cr
    &=\sum_{t=1}^T\mathbb{E}\left[\sum_{i\in \I(t), j \in \J}(\tilde{r}_{i,j}(t)-r_{i,j})Q_i(t)x_{i,j}(t)\right]\cr
    &=\sum_{t=1}^T\mathbb{E}\left[\mathbb{E}\left[\sum_{k\in \Q(t), j \in \J}(\tilde{r}_{\cl(k),j}(t)-r_{\cl(k),j})x_{\cl(k),j}(t)\mid\mathcal{F}_{t-1}\right]\right]\cr
    &=\sum_{t=1}^T\mathbb{E}\left[\mathbb{E}\left[\sum_{k\in \Q(t), j \in \J}(\tilde{r}_{\cl(k),j}(t)-r_{\cl(k),j})\tilde{x}_{\cl(k),j}(t)\mid\mathcal{F}_{t-1}\right]\right]\cr
    &=\sum_{t=1}^T\mathbb{E}\left[\sum_{k\in \tilde{\Q}(t), j \in \J}(\tilde{r}_{\cl(k),j}(t)-r_{\cl(k),j})\tilde{x}_{\cl(k),j}(t)\right],
\end{align*}
where the second last equation holds because $\tilde{x}_{\cl(k),j}=0$ for all $k\notin \tilde{\Q}(t)$.

By Lemma~\ref{lem:confidence}, conditioning on the event $E=\{\theta\in {\mathcal C}(t) $ for all $ t\in \T\}$, which holds with probability at least $1-1/T$, we have 
\begin{align}
    &\sum_{t=1}^T\sum_{k\in \tilde{\Q}(t), j \in \J}(\tilde{r}_{\cl(k),j}(t)-r_{\cl(k),j})\tilde{x}_{\cl(k),j}(t)\cr
    &\le 2\sum_{t=1}^T\sum_{k\in \tilde{\Q}(t), j \in  \J} \|z_{\cl(k),j}\|_{\Lambda(t)^{-1}}{\beta(t)}\tilde{x}_{\cl(k),j}(t).\label{eq:G1_bd_step1}
\end{align}

 Conditioning on $E^c$, which holds with probability at most $1/T$, we have
\begin{align}
     \sum_{t=1}^T\sum_{k\in \tilde{\Q}(t), j \in \J}(\hat{r}_{\cl(k),j}(t)-r_{\cl(k),j})\tilde{x}_{\cl(k),j}(t)\le 2nT.
    \label{eq:G1_bd_low_prob}
\end{align}

We next show the following lemma. 

\begin{lemma} The following inequality holds
\begin{align*}
    \sum_{t=1}^T\sum_{k\in\tilde{\Q}(t),j\in\J}\tilde{x}_{\cl(k),j}(t)\|z_{\cl(k),j}\|_{\Lambda(t)^{-1}}^2\le 2d^2\log\left(n+\frac{Tn}{d^2}\right).
\end{align*}\label{lemma:w_norm}
\end{lemma}

\begin{proof} Let \[ s_t := \sum_{k\in \widetilde Q(t),\,j\in\mathcal J} \widetilde x_{\upsilon(k),j}(t) \|z_{\upsilon(k),j}\|_{\Lambda(t)^{-1}}^2 . \] We first show that \(s_t\le 1\). Since \(\Lambda(t)\succeq nI_{d^2}\) and \(\|z_{i,j}\|_2\le 1\), we have \[ \|z_{\upsilon(k),j}\|_{\Lambda(t)^{-1}}^2 \le \frac{1}{n}. \] Moreover, at most \(n\) servers are assigned in one time slot, and hence \[ \sum_{k\in \widetilde Q(t),\,j\in\mathcal J} \widetilde x_{\upsilon(k),j}(t) \le n . \] Therefore \(s_t\le 1\). Now define \[ A_t := \sum_{k\in \widetilde Q(t),\,j\in\mathcal J} \widetilde x_{\upsilon(k),j}(t) \Lambda(t)^{-1/2} z_{\upsilon(k),j}z_{\upsilon(k),j}^{\top} \Lambda(t)^{-1/2}. \] Then \(A_t\succeq 0\) and \(\operatorname{tr}(A_t)=s_t\). Hence, \begin{align} \frac{\det(\Lambda(t+1))}{\det(\Lambda(t))} &= \det(I_{d^2}+A_t) \nonumber\\ &\ge 1+\operatorname{tr}(A_t) \nonumber\\ &= 1+s_t . \end{align} Since \(s_t\le 1\), we have \(s_t\le 2\log(1+s_t)\). Thus, \begin{align} \sum_{t=1}^T s_t &\le 2\sum_{t=1}^T \log(1+s_t) \nonumber\\ &\le 2\sum_{t=1}^T \log\frac{\det(\Lambda(t+1))}{\det(\Lambda(t))} \nonumber\\ &= 2\log\frac{\det(\Lambda(T+1))}{\det(\Lambda(1))}. \end{align} Finally, since \(\Lambda(1)=nI_{d^2}\), we have \(\det(\Lambda(1))=n^{d^2}\). Also, \[ \operatorname{tr}(\Lambda(T+1)) \le nd^2+nT, \] and therefore, by the trace-determinant inequality, \[ \det(\Lambda(T+1)) \le \left(n+\frac{nT}{d^2}\right)^{d^2}. \] Combining the above bounds gives \begin{align} \sum_{t=1}^T \sum_{k\in \widetilde Q(t),\,j\in\mathcal J} \widetilde x_{\upsilon(k),j}(t) \|z_{\upsilon(k),j}\|_{\Lambda(t)^{-1}}^2 &\le 2d^2 \log\left( \frac{n+nT/d^2}{n} \right) \nonumber\\ &= 2d^2\log\left(1+\frac{T}{d^2}\right) \nonumber\\ &\le 2d^2\log\left(n+\frac{nT}{d^2}\right). \end{align} This proves the lemma. \end{proof}

 Finally, combining \eqref{eq:G1_bd_step1}, \eqref{eq:G1_bd_low_prob}, and Lemma~\ref{lemma:w_norm}, we obtain %
\begin{align*}
    \frac{1}{V}\sum_{t=1}^T\mathbb{E}[G_1(t)]&\le
\mathbb{E}\left[\sum_{t=1}^T\sum_{k\in \tilde{\Q}(t), j \in  \J} \|z_{\cl(k),j}\|_{\Lambda(t)^{-1}}\beta(t)\tilde{x}_{\cl(k),j}(t)\right]+2 n\cr
&=
\mathbb{E}\left[\sum_{t=1}^T\sum_{k\in \tilde{\Q}(t), j \in  \J} \sum_{s=1}^{\tilde{x}_{\cl(k),j}(t)}\|z_{\cl(k),j}\|_{\Lambda(t)^{-1}}\beta(t)\right]+2 n
\cr&\le \mathbb{E}\left[\beta(T)\sqrt{T\left(\sum_{k\in\tilde{\Q}(t),j\in \J}\tilde{x}_{\cl(k),j}(t)\right)\sum_{t=1}^T\sum_{k\in\tilde{\Q}(t),j\in\J}\sum_{s=1}^{\tilde{x}_{\cl(k),j}(t)}\|z_{\cl(k),j}\|_{\Lambda(t)^{-1}}^2}\right]+2 n\cr
\cr&\le \mathbb{E}\left[\beta(T)\sqrt{nT\sum_{t=1}^T\sum_{k\in\tilde{\Q}(t),j\in\J}\tilde{x}_{\cl(k),j}(t)\|z_{\cl(k),j}\|_{\Lambda(t)^{-1}}^2}\right]+2 n\cr
& =  \tilde{O}( (d^2\sqrt{n}+dn)\sqrt{T}),
\end{align*}
\end{proof}
where the second inequality holds by the Cauchy-Schwarz inequality $\sum_{i=1}^Na_i\le \sqrt{N\sum_{i=1}^N a_i^2}$ and the last inequality holds by $\sum_{k\in\tilde{\Q}(t)j\in \J}\tilde{x}_{\cl(k),j}(t)\le n$ and $\beta(t)\le \beta(T)$ for all $1\le t\le T$ where recall ${\beta(t)} = \sqrt{d^2\log\left(tT\right)}+\sqrt{n}$.

 \begin{lemma} The following bound holds
 $$
 \frac{1}{V}\sum_{t=1}^T \mathbb{E}[G_2(t)]\le 2 n.
 $$
 \label{lem:A_3}

 \end{lemma}

 \begin{proof}

 If $\theta\in {\mathcal C}(t)$, from $r_{i.j}-\tilde{r}_{i,j}(t)\ge 0$, we have 
 $$
 G_2(t)\le0.
 $$  

 Therefore, we only need to consider the case when $\theta\notin {\mathcal C}(t)$ for some $t\in\T$, which holds with probability at most $1/T$. We obtain
$$
\frac{1}{V}\sum_{t=1}^T \mathbb{E}[G_2(t)]\le \frac{1}{T}\sum_{t=1}^T\sum_{i\in \I(t),j\in \J}2 y^*_{i,j}(t)\le 2 n.
$$
 \end{proof}

The bound for $(1/V)\sum_{t=1}^T \mathbb{E}[G(t)]$ follows from Lemma~\ref{lem:A_2} and Lemma~\ref{lem:A_3}.
\end{proof}

From Lemma~\ref{lem:R_bd_lam}, Lemma~\ref{lem:A_1}, and
Lemma~\ref{lem:scalar_total_queue_bound},
\begin{align}
R(T)
&\le \frac{\gamma}{\mu}\mathbb E[Q(t)]
   +\frac1V\sum_{t=1}^T\mathbb E[G(t)]
   +\frac1V\left(
       \sum_{t=1}^T\mathbb E[H(t)]
       +\frac{T+1}{2}\sum_{i\in\I}w_i
     \right)\\
&\le
\frac{\gamma}{\mu}\left(
   \frac{n+\rho}{\delta}
   +\max\left\{
       \frac{(\gamma+1)Vn}{w_{\min}},
       2c_\gamma\frac{n^3}{\delta}
          \frac{w_{\max}}{w_{\min}}
     \right\}
   +1
\right)\nonumber\\
&\quad
+\widetilde O\!\left((d^2\sqrt n+dn)\sqrt T\right)
+\frac{n^2T}{2V}(1+c_\gamma^2\mu)
   \max_{i\in\I}\frac{w_i}{\rho_i}
+\frac{T+1}{2V}\sum_{i\in\I}w_i.
\label{eq:regret_after_scalar_queue}
\end{align}
Using \(\max\{a,b\}\le a+b\), for \(T\ge1\) this yields
\begin{equation}
R(T)
=\widetilde O\!\left(
   \bar\alpha_1V
   +\bar\alpha_2\frac1\delta
   +\bar\alpha_3\frac{T}{V}
   +\bar\alpha_4\sqrt T
\right),
\end{equation}
where 
\[
\bar\alpha_1
  =\frac{\gamma(\gamma+1)n}{\mu w_{\min}},
\qquad
\bar\alpha_2
  =\frac{\gamma}{\mu}\left(
      n+\rho+2c_\gamma n^3\frac{w_{\max}}{w_{\min}}
    \right),
\]
\[
\bar\alpha_3
  =\frac{n^2}{2}(1+c_\gamma^2\mu)
      \max_{i\in\I}\frac{w_i}{\rho_i}
    +\frac12\sum_{i\in\I}w_i,
\qquad
\bar\alpha_4=d^2\sqrt n+dn.
\]
If \(\gamma,\mu\), and the traffic parameters \(\rho_i\) are treated
as fixed constants, the same bound can be summarized as
\[
R(T)=\widetilde O\!\left(
  \frac{n}{w_{\min}}V
  +n^3\frac{w_{\max}}{w_{\min}}\frac1\delta
  +\left(n^2w_{\max}+\sum_{i\in\I}w_i\right)\frac{T}{V}
  +(d^2\sqrt n+dn)\sqrt T
\right).
\]

From Lemma~\ref{lem:R_bd_lam}, Lemma~\ref{lem:A_1} and queue length bound obtained from Theorem~\ref{thm:Q_len_bound}, we have  
\begin{align*}
    R(T)&\le \gamma \frac{1}{\mu}\mathbb{E}[Q(T)]+\frac{1}{V}\sum_{t=1}^T\mathbb{E}[G(t)]+\frac{1}{V}\left(\sum_{t=1}^T\mathbb{E}[H(t)]+(T+1)\frac{1}{2}\sum_{i\in\I}\we_i\right)\cr
    &=\tilde{O}\left(\alpha_1 V + \alpha_2 \frac{1}{\delta} + \alpha_3 \frac{T}{V} + \alpha_4 \sqrt{T}\right),
\end{align*}
where 
\[
\alpha_1 = \frac{1}{w_{\min}}, \quad 
\alpha_2 = n^3 \frac{w_{\max}}{w_{\min}},
\]
\[
\alpha_3 = n^2 w_{\max} + \sum_{i \in \mathcal{I}} w_i, \quad \hbox{ and }\quad 
\alpha_4 = d^2\sqrt{n}+dn.
\]
\subsection{Proof of Theorem~\ref{thm:Q_len_bound}} \label{app:q_len_bound}

\subsubsection{A scalar total-queue bound for the regret analysis}
\label{sec:scalar_queue_regret}

The regret decomposition in Lemma~\ref{lem:R_bd_lam} contains the
pointwise term \(\mathbb E[Q(t)]\), where
\[
    Q(t):=\sum_{i\in\I}Q_i(t)
\]
is the total number of jobs in the system.  Let \(\mathcal F_t\) denote
the system history up to just before the randomized job selections in slot
\(t\), and let \(\mathcal F_{t,r}\) additionally include the outcomes of the
first \(r\) selections in that slot.  To bound \(\mathbb E[Q(t)]\), it
is sufficient to use a scalar queue comparison; no class-wise decomposition
of the comparison queue is needed.

Define
\begin{equation}
    \tau_{Q,1}:=\frac{(\gamma+1)Vn}{w_{\min}},
    \qquad
    \tau_{Q,2}:=2c_\gamma\frac{n^3}{\delta}
                  \frac{w_{\max}}{w_{\min}},
    \qquad
    B_Q:=\max\{\tau_{Q,1},\tau_{Q,2}\},
    \label{eq:scalar_queue_thresholds}
\end{equation}
where \(w_{\min}:=\min_{i\in\I}w_i\),
\(w_{\max}:=\max_{i\in\I}w_i\), and
\(c_\gamma=(\gamma+1)/(\gamma-1)\).

\begin{lemma}[Full utilization under a large total queue]
\label{lem:saturation_total_queue}
If \(Q(t)\ge \tau_{Q,1}\), then
\[
    \sum_{i\in\I(t)}y_{i,j}(t)=n_j,
    \qquad j\in\J.
\]
\end{lemma}

\begin{proof}
Suppose, to the contrary, that
\(\sum_{i\in\I(t)}y_{i,j_0}(t)<n_{j_0}\) for some
\(j_0\in\J\).  By complementary slackness,
\(q_{j_0}(t)=0\).  The stationarity condition
\eqref{eq:L2_stationary}, together with
\(h_{i,j_0}(t)\ge0\) and
\(\widetilde r_{i,j_0}(t)\in[-1,1]\), gives, for every
\(i\in\I(t)\),
\[
\begin{aligned}
    \sum_{j\in\J}y_{i,j}(t)
    &=\frac{w_iQ_i(t)}
    {V\bigl(q_{j_0}(t)-h_{i,j_0}(t)
        +\gamma-\widetilde r_{i,j_0}(t)\bigr)}\\
    &\ge \frac{w_iQ_i(t)}{V(\gamma+1)}.
\end{aligned}
\]
Consequently,
\[
    \sum_{i\in\I(t)}\sum_{j\in\J}y_{i,j}(t)
    \ge
    \frac{w_{\min}Q(t)}{V(\gamma+1)}
    \ge n.
\]
On the other hand, the capacity constraints imply
\(\sum_{i,j}y_{i,j}(t)\le\sum_jn_j=n\), while the strict
slackness of server class \(j_0\) would imply
\(\sum_{i,j}y_{i,j}(t)<n\), a contradiction.
\end{proof}

\begin{lemma}[Conditional scalar departure domination]
\label{lem:departure_dom_total_queue}
On the event \(\{Q(t)\ge B_Q\}\), conditional on the
system history before the randomized selections in slot \(t\),
\[
    \mathbb P\!\left(D(t)\ge x\mid\mathcal F_t\right)
    \ge
    \mathbb P(W\ge x),
    \qquad x\ge0,
\]
where
\[
    W\sim\mathrm{Binom}\!\left(
        n,\frac{1+\rho/n}{2}\mu
    \right).
\]
\end{lemma}

\begin{proof}
By Lemma~\ref{lem:saturation_total_queue}, all server capacities are
fully utilized when \(Q(t)\ge B_Q\).  Let \(S_r(t)\) be the
set of jobs selected before round \(r\), and let \(X_r(t)=1\) if the
job selected in round \(r\) is not in \(S_r(t)\), and
\(X_r(t)=0\) otherwise.  If the server considered in round \(r\)
belongs to class \(j\), then
\[
\begin{aligned}
\mathbb P\!\left(X_r(t)=1\mid\mathcal F_{t,r-1}\right)
&=1-\frac1{n_j}\sum_{k\in\Q(t)}
   \frac{y_{\cl(k),j}(t)}{Q_{\cl(k)}(t)}
   \mathbbm 1\{k\in S_r(t)\}\\
&\ge 1-n\max_{i\in\I(t)}
       \frac{y_{i,j}(t)}{Q_i(t)}.
\end{aligned}
\]
By \eqref{eq:x_upbd_traffic}, for every \(i\in\I(t)\),
\[
    y_{i,j}(t)
    \le \sum_{j'\in\J}y_{i,j'}(t)
    \le c_\gamma
       \frac{n w_iQ_i(t)}
       {\sum_{i'\in\I(t)}w_{i'}Q_{i'}(t)}.
\]
Hence, on \(\{Q(t)\ge B_Q\}\),
\[
\begin{aligned}
\mathbb P\!\left(X_r(t)=1\mid\mathcal F_{t,r-1}\right)
&\ge
1-c_\gamma\frac{n^2w_{\max}}
 {\sum_{i'\in\I(t)}w_{i'}Q_{i'}(t)}\\
&\ge
1-c_\gamma\frac{n^2w_{\max}}
 {w_{\min}Q(t)}\\
&\ge 1-\frac{\delta}{2n}
 =\frac{1+\rho/n}{2}.
\end{aligned}
\]

\red{
\begin{lemma}[Bernoulli thinning coupling] \label{lem:bernoulli-thinning} Let \((\mathcal F_r)_{r\ge 0}\) be a filtration and let \(X_r\in\{0,1\}\) be \(\mathcal F_r\)-measurable. Suppose that, for some \(p\in[0,1]\), \[ \mathbb P(X_r=1\mid \mathcal F_{r-1})\ge p \qquad\text{for all } r=1,\ldots,m . \] Then, on an enlarged probability space, there exist independent random variables \(Y_1,\ldots,Y_m\) such that \[ Y_r\sim \mathrm{Ber}(p) \] and \[ Y_r\le X_r \qquad\text{almost surely for all } r=1,\ldots,m . \] \end{lemma} \begin{proof} If \(p=0\), the claim is trivial by taking \(Y_r=0\) for all \(r\). Hence, assume \(p>0\). Let \[ p_r:=\mathbb P(X_r=1\mid \mathcal F_{r-1}), \] so that \(p_r\ge p>0\) almost surely. On an enlarged probability space, let \(U_1,\ldots,U_m\) be i.i.d. uniform random variables on \([0,1]\), independent of the original system randomness. Define \[ Y_r := X_r\,\mathbf 1\!\left\{ U_r\le \frac{p}{p_r} \right\}, \qquad r=1,\ldots,m . \] Since \(p/p_r\le 1\), this is well defined, and clearly \(Y_r\le X_r\) almost surely. Let \[ \mathcal G_r := \mathcal F_r\vee \sigma(U_1,\ldots,U_r) \] be the enlarged filtration. Since \(U_r\) is independent of \(\mathcal G_{r-1}\) and of the original system randomness, and since \(p_r\) is \(\mathcal F_{r-1}\)-measurable, we have \[ \begin{aligned} \mathbb P(Y_r=1\mid \mathcal G_{r-1}) &= \mathbb E\!\left[ X_r\,\mathbf 1\!\left\{ U_r\le \frac{p}{p_r} \right\} \,\middle|\, \mathcal G_{r-1} \right] \\ &= \frac{p}{p_r}\, \mathbb P(X_r=1\mid \mathcal G_{r-1}) \\ &= \frac{p}{p_r}\, \mathbb P(X_r=1\mid \mathcal F_{r-1}) \\ &= p . \end{aligned} \] Thus, \[ \mathbb P(Y_r=1\mid \mathcal G_{r-1})=p, \qquad \mathbb P(Y_r=0\mid \mathcal G_{r-1})=1-p . \] We now prove independence. For any \(a_1,\ldots,a_m\in\{0,1\}\), using that \(\{Y_1=a_1,\ldots,Y_{r-1}=a_{r-1}\}\in\mathcal G_{r-1}\), we obtain iteratively \[ \begin{aligned} &\mathbb P(Y_1=a_1,\ldots,Y_m=a_m) \\ &\qquad = \mathbb E\!\left[ \mathbf 1\{Y_1=a_1,\ldots,Y_{m-1}=a_{m-1}\} \mathbb P(Y_m=a_m\mid \mathcal G_{m-1}) \right] \\ &\qquad = p^{a_m}(1-p)^{1-a_m} \mathbb P(Y_1=a_1,\ldots,Y_{m-1}=a_{m-1}) \\ &\qquad = \prod_{r=1}^m p^{a_r}(1-p)^{1-a_r}. \end{aligned} \] Therefore \(Y_1,\ldots,Y_m\) are independent Bernoulli random variables with parameter \(p\). This completes the proof. \end{proof}}

Applying Lemma~\ref{lem:bernoulli-thinning} with
\(p=(1+\rho/n)/2\), on an enlarged probability space there are
i.i.d. variables \(Y_1(t),\ldots,Y_n(t)\) such that
\[
    Y_r(t)\sim\mathrm{Ber}(p),
    \qquad Y_r(t)\le X_r(t)\quad\text{a.s.}
\]
Whenever \(Y_r(t)=1\), the selected job is distinct from all jobs
selected in preceding rounds.  By the memoryless service assumption,
the completion indicators of these distinct jobs are independent
Bernoulli random variables with parameter \(\mu\), independently of
the selection history.  Thus we may take i.i.d.
\(Z_r(t)\sim\mathrm{Ber}(\mu)\), independent of the \(Y_r(t)\)'s,
such that
\[
    D(t)\ge\sum_{r=1}^nY_r(t)Z_r(t).
\]
The sum on the right has distribution
\(\mathrm{Binom}(n,p\mu)\), which proves the claim.
\end{proof}

\begin{lemma}[Uniform scalar total-queue bound]
\label{lem:scalar_total_queue_bound}
For every \(t\ge0\),
\begin{equation}
    \mathbb E[Q(t)]
    \le
    \frac{n+\rho}{\delta}+B_Q+1.
    \label{eq:scalar_total_queue_bound}
\end{equation}
In particular,
\[
    \sup_{t\ge0}\mathbb E[Q(t)]
    =O\!\left(
       \frac{nV}{w_{\min}}
       +\frac{n+n^3w_{\max}/w_{\min}}{\delta}
    \right),
\]
where the hidden constant depends only on \(\gamma\) and \(\mu\).
\end{lemma}

\begin{proof}
Let \(Q^0(0)=0\), and define the scalar comparison queue
\begin{equation}
    Q^0(t+1)
    =\bigl[Q^0(t)+A(t+1)-D^0(t)\bigr]^+,
    \label{eq:scalar_comparison_queue}
\end{equation}
where
\[
    D^0(t)\sim\mathrm{Binom}\!\left(
       n,\frac{1+\rho/n}{2}\mu
    \right).
\]
Using fresh auxiliary randomness at each slot and the conditional
stochastic domination in Lemma~\ref{lem:departure_dom_total_queue},
we may construct the processes on a common probability space so that
\begin{equation}
    D^0(t)\le D(t)
    \qquad\text{whenever }Q(t)>B_Q,
    \label{eq:coupled_departures_scalar}
\end{equation}
while the variables \(D^0(t)\) retain the prescribed i.i.d. binomial
law and are independent of the arrival process.  The comparison queue
therefore has the same law as the \(\mathrm{Geom}/\mathrm{Geom}/n\)
queue in Lemma EC.5 of \cite{hsu2021integrated}, and hence
\begin{equation}
    \mathbb E[Q^0(t)]\le\frac{n+\rho}{\delta}.
    \label{eq:q0_scalar_mean}
\end{equation}

We next prove, pathwise, that
\begin{equation}
    Q(t)\le Q^0(t)+B_Q+1,
    \qquad t\ge0.
    \label{eq:scalar_pathwise_comparison}
\end{equation}
The claim is immediate at \(t=0\).  Suppose it holds at time \(t\).
If \(Q(t)\le B_Q\), then, because at most one job arrives in
a slot,
\[
    Q_\Sigma(t+1)
    \le Q(t)+A(t+1)
    \le B_Q+1
    \le Q^0(t+1)+B_Q+1.
\]
If \(Q(t)>B_Q\), then \eqref{eq:coupled_departures_scalar}
and the induction hypothesis give
\[
\begin{aligned}
    Q_\Sigma(t+1)
    &=Q(t)+A(t+1)-D(t)\\
    &\le Q^0(t)+B_Q+1+A(t+1)-D^0(t)\\
    &\le Q^0(t+1)+B_Q+1,
\end{aligned}
\]
where the last inequality follows from
\([x]^+\ge x\).  This proves \eqref{eq:scalar_pathwise_comparison}.
Taking expectations and using \eqref{eq:q0_scalar_mean} proves
\eqref{eq:scalar_total_queue_bound}.
\end{proof}

%
%
%

%
%
%
%
%
%

%
%
%
%

\subsection{Proof of Theorem~\ref{thm:Q_len_bound_cost}}\label{app:holding_bound}

We separate the argument into two parts.  First, we give a scalar
comparison for the total queue length.  This comparison is used in the
finite-horizon regret proof and avoids introducing class-indexed
virtual queues.  Second, we prove the general weighted holding-cost
result by a weighted Lyapunov-drift argument.  Throughout, let
$\mathcal F_t$ denote the information available immediately before the
arrivals and service completions in slot $t$, including $Q(t)$ and the
allocation $y(t)$.

\subsubsection{A scalar comparison for the total queue length}

Let
\[
    Q(t):=\sum_{i\in\I}Q_i(t),
    \qquad
    w_{\min}:=\min_{i\in\I}w_i,
    \qquad
    w_{\max}:=\max_{i\in\I}w_i,
\]
and define
\begin{equation}
\label{eq:scalar-threshold-1}
    b_1:=\frac{(\gamma+1)Vn}{w_{\min}},
\end{equation}
\begin{equation}
\label{eq:scalar-threshold-2}
    b_2:=2c_\gamma\frac{n^3}{\delta}
          \frac{w_{\max}}{w_{\min}},
    \qquad
    B_Q:=\max\{b_1,b_2\}.
\end{equation}

\begin{lemma}
\label{lem:sat-total}
If $Q(t)\ge b_1$, then
\[
    \sum_{i\in\I(t)}y_{i,j}(t)=n_j,
    \qquad j\in\J.
\]
\end{lemma}

\begin{proof}
Suppose, to obtain a contradiction, that $Q(t)\ge b_1$ and
that there exists $j_0\in\J$ such that
$\sum_{i\in\I(t)}y_{i,j_0}(t)<n_{j_0}$.  By complementary slackness,
the multiplier $q_{j_0}(t)$ associated with the capacity constraint of
server class $j_0$ is zero.  The stationarity condition for the
optimization problem defining $y(t)$ gives, for every $i\in\I(t)$,
\[
\begin{split}
    \sum_{j\in\J}y_{i,j}(t)
    &=\frac{w_iQ_i(t)}
       {V\bigl(q_{j_0}(t)-h_{i,j_0}(t)+\gamma-
       \widetilde r_{i,j_0}(t)\bigr)} \\
    &\ge \frac{w_iQ_i(t)}{V(\gamma+1)},
\end{split}
\]
where we used $\widetilde r_{i,j}(t)\in[-1,1]$.  Therefore,
\[
\begin{split}
    \sum_{i\in\I(t)}\sum_{j\in\J}y_{i,j}(t)
    &\ge \frac{\sum_{i\in\I(t)}w_iQ_i(t)}{V(\gamma+1)} \\
    &\ge \frac{w_{\min}Q(t)}{V(\gamma+1)}
    \ge n.
\end{split}
\]
On the other hand, the server-capacity constraints imply
$\sum_{i\in\I(t)}\sum_{j\in\J}y_{i,j}(t)\le n$.  Hence equality must
hold for every server-class capacity constraint, contradicting the
strict inequality for $j_0$.  This proves the claim.
\end{proof}

\begin{lemma}
\label{lem:stoch-dom-total}
On the event $\{Q(t)\ge B_Q\}$, conditional on
$\mathcal F_t$,
\[
    \mathbb P(D(t)\ge x\mid\mathcal F_t)
    \ge
    \mathbb P(W\ge x),
    \qquad x\ge0,
\]
where
\[
    W\sim\mathrm{Binom}\!\left(
        n,\frac{1+\rho/n}{2}\mu
    \right).
\]
\end{lemma}

\begin{proof}
By Lemma~\ref{lem:sat-total}, if $Q(t)\ge b_1$, then
\begin{equation}
\label{eq:all-capacities-saturated-total}
    \sum_{i\in\I(t)}y_{i,j}(t)=n_j,
    \qquad j\in\J.
\end{equation}
At time $t$, enumerate the $n$ server selections in their sequential
order.  Let $S_r(t)$ be the set of jobs selected before round $r$, and
let $X_r(t)=1$ if the job selected in round $r$ does not belong to
$S_r(t)$, and $X_r(t)=0$ otherwise.  If round $r$ corresponds to a
server of class $j$, then
\[
\begin{split}
    \mathbb P(X_r(t)=1\mid S_r(t))
    &=\frac{1}{n_j}\sum_{k\in\Q(t)}
      \frac{y_{\cl(k),j}(t)}{Q_{\cl(k)}(t)}
      \mathbf 1\{k\notin S_r(t)\} \\
    &=1-\frac{1}{n_j}\sum_{k\in\Q(t)}
      \frac{y_{\cl(k),j}(t)}{Q_{\cl(k)}(t)}
      \mathbf 1\{k\in S_r(t)\} \\
    &\ge 1-n\max_{i\in\I(t)}
       \frac{y_{i,j}(t)}{Q_i(t)},
\end{split}
\]
where the second equality follows from
\eqref{eq:all-capacities-saturated-total}.

From the allocation-ratio bound established in
\eqref{eq:x_upbd_traffic},
\[
    y_{i,j}(t)
    \le c_\gamma
       \frac{nw_iQ_i(t)}
       {\sum_{i'\in\I(t)}w_{i'}Q_{i'}(t)},
    \qquad i\in\I(t).
\]
Consequently, if $Q(t)\ge b_2$, then
\[
\begin{split}
    \mathbb P(X_r(t)=1\mid S_r(t))
    &\ge 1-c_\gamma
       \frac{n^2w_{\max}}
       {\sum_{i'\in\I(t)}w_{i'}Q_{i'}(t)} \\
    &\ge 1-c_\gamma
       \frac{n^2w_{\max}}
       {w_{\min}Q(t)} \\
    &\ge \frac{1+\rho/n}{2}.
\end{split}
\]
Let $\mathcal F_{r-1}(t)$ denote the complete selection history before
round $r$.  Applying Lemma~\ref{lem:bernoulli-thinning} with
$p=(1+\rho/n)/2$ yields i.i.d. random variables
$Y_1(t),\ldots,Y_n(t)$ such that
\[
    Y_r(t)\sim\mathrm{Ber}\!\left(\frac{1+\rho/n}{2}\right),
    \qquad
    Y_r(t)\le X_r(t).
\]
Whenever $Y_r(t)=1$, the corresponding selected jobs are distinct.
Since service requirements are geometric with common completion
probability $\mu$ and are independent across jobs, their completion
indicators can be represented by i.i.d. random variables
$Z_r(t)\sim\mathrm{Ber}(\mu)$, independent of the thinning variables.
Therefore,
\[
    D(t)\ge\sum_{r=1}^nY_r(t)Z_r(t).
\]
The right-hand side has distribution
$\mathrm{Binom}(n,(1+\rho/n)\mu/2)$, proving the claim.
\end{proof}

\begin{lemma}[Uniform scalar total-queue bound]
\label{lem:scalar-total-queue}
For all $t\ge0$,
\begin{equation}
\label{eq:scalar-total-queue-bound}
    \mathbb E[Q(t)]
    \le
    \frac{n+\rho}{\delta}+B_Q+1.
\end{equation}
\end{lemma}

\begin{proof}
Let $Q^0$ be the occupancy of a $\mathrm{Geom}/\mathrm{Geom}/n$
queue satisfying
\[
    Q^0(t+1)
    =\bigl[Q^0(t)+A(t+1)-D^0(t)\bigr]^+,
    \qquad Q^0(0)=0,
\]
where
\[
    D^0(t)\sim\mathrm{Binom}\!\left(
       n,\frac{1+\rho/n}{2}\mu
    \right).
\]
By Lemma EC.5 in \cite{hsu2021integrated},
\begin{equation}
\label{eq:scalar-reference-mean}
    \mathbb E[Q^0(t)]\le\frac{n+\rho}{\delta}.
\end{equation}
At every slot for which $Q(t)>B_Q$, Lemma
\ref{lem:stoch-dom-total} and the usual quantile coupling allow us to
construct $D(t)$ and $D^0(t)$ so that
\[
    D(t)\ge D^0(t)
    \quad\text{a.s.}
\]
Fresh coupling randomness may be used in each slot, so the conditional
law of $D^0(t)$ remains the stated binomial law.

We claim that
\begin{equation}
\label{eq:pathwise-scalar-comparison}
    Q(t)\le Q^0(t)+B_Q+1,
    \qquad t\ge0.
\end{equation}
The claim holds at $t=0$.  Suppose it holds at time $t$.  If
$Q(t)\le B_Q$, then, because at most one job arrives in one
slot,
\[
    Q(t+1)
    \le Q(t)+A(t+1)
    \le B_Q+1
    \le Q^0(t+1)+B_Q+1.
\]
If $Q(t)>B_Q$, then the above coupling gives
$D(t)\ge D^0(t)$, and hence
\[
\begin{split}
    Q(t+1)
    &=Q(t)+A(t+1)-D(t) \\
    &\le Q^0(t)+B_Q+1+A(t+1)-D^0(t) \\
    &\le Q^0(t+1)+B_Q+1,
\end{split}
\]
where the last inequality follows from
$[x]^+\ge x$.  This proves \eqref{eq:pathwise-scalar-comparison} by
induction.  Taking expectations and using
\eqref{eq:scalar-reference-mean} proves
\eqref{eq:scalar-total-queue-bound}.
\end{proof}

\subsubsection{Weighted holding-cost analysis}

Define
\[
    s_i(t):=\sum_{j\in\J}y_{i,j}(t),
    \qquad
    W(t):=\sum_{i\in\I}w_iQ_i(t),
\]
and consider the Lyapunov function
\begin{equation}
\label{eq:weighted-Lyapunov}
    L(Q(t)):=\frac12\sum_{i\in\I}
       \frac{w_i}{\rho_i}Q_i(t)^2.
\end{equation}
We assume, as elsewhere in the paper, that $\rho_i>0$ for every class.

\begin{lemma}[One-step weighted drift]
\label{lem:one-step-weighted-drift}
For every $t\ge0$,
\begin{equation}
\label{eq:preliminary-weighted-drift}
\begin{split}
    \mathbb E[&L(Q(t+1))-L(Q(t))\mid\mathcal F_t] \\
    &\le
    \mu\sum_{i\in\I}
       \frac{w_iQ_i(t)}{\rho_i}\bigl(\rho_i-s_i(t)\bigr)
    +B_0,
\end{split}
\end{equation}
where
\begin{equation}
\label{eq:B0-definition}
    B_0
    :=\frac{\mu}{2}\sum_{i\in\I}w_i
      +\frac{\mu n^2}{2}
       \left(1+\mu c_\gamma^2\right)
       \max_{i\in\I}\frac{w_i}{\rho_i}.
\end{equation}
\end{lemma}

\begin{proof}
The queue recursion and $(a-b)^2\le a^2+b^2$ for $a,b\ge0$ give
\[
\begin{split}
    \mathbb E[&Q_i(t+1)^2-Q_i(t)^2\mid\mathcal F_t] \\
    &\le 2Q_i(t)
       \bigl(\lambda_i-\mathbb E[D_i(t)\mid\mathcal F_t]\bigr)
       +\lambda_i
       +\mathbb E[D_i(t)^2\mid\mathcal F_t].
\end{split}
\]
The departure estimates established earlier in Appendix A imply, for
$i\in\I(t)$,
\begin{equation}
\label{eq:departure-lower-recalled}
    \mathbb E[D_i(t)\mid\mathcal F_t]
    \ge
    \mu s_i(t)
    -\frac{\mu^2n^2c_\gamma^2}{2}
       \frac{w_i^2Q_i(t)}{W(t)^2},
\end{equation}
and
\begin{equation}
\label{eq:departure-second-recalled}
    \sum_{i\in\I}
       \mathbb E[D_i(t)^2\mid\mathcal F_t]
    \le n^2\mu.
\end{equation}
When $W(t)=0$, all queues are zero and the term containing $W(t)^{-2}$
is interpreted as zero.

Multiplying the preceding queue-drift inequality by
$w_i/(2\rho_i)$, summing over $i$, and using
$\lambda_i=\mu\rho_i$, we obtain
\[
\begin{split}
    \mathbb E[&L(Q(t+1))-L(Q(t))\mid\mathcal F_t] \\
    \le {}&
    \mu\sum_{i\in\I}
       \frac{w_iQ_i(t)}{\rho_i}(\rho_i-s_i(t)) \\
    &+\frac{\mu^2n^2c_\gamma^2}{2W(t)^2}
       \sum_{i\in\I(t)}
       \frac{w_i^3Q_i(t)^2}{\rho_i} \\
    &+\frac{\mu}{2}\sum_{i\in\I}w_i
     +\frac12\sum_{i\in\I}
       \frac{w_i}{\rho_i}
       \mathbb E[D_i(t)^2\mid\mathcal F_t].
\end{split}
\]
Since
\[
    \sum_{i\in\I(t)}
       \frac{w_i^3Q_i(t)^2}{\rho_i}
    \le
    \left(\max_{i\in\I}\frac{w_i}{\rho_i}\right)
    \sum_{i\in\I(t)}w_i^2Q_i(t)^2
    \le
    \left(\max_{i\in\I}\frac{w_i}{\rho_i}\right)W(t)^2,
\]
and, by \eqref{eq:departure-second-recalled},
\[
    \frac12\sum_{i\in\I}
       \frac{w_i}{\rho_i}
       \mathbb E[D_i(t)^2\mid\mathcal F_t]
    \le
    \frac{n^2\mu}{2}
       \max_{i\in\I}\frac{w_i}{\rho_i},
\]
we obtain \eqref{eq:preliminary-weighted-drift} with $B_0$ given by
\eqref{eq:B0-definition}.
\end{proof}

\begin{lemma}[Service slack induced by proportional fairness]
\label{lem:allocation-service-slack}
Let
\[
    \kappa:=\frac{n+\rho}{2\rho}>1.
\]
Then, for every $t\ge0$,
\begin{equation}
\label{eq:allocation-service-slack}
    \sum_{i\in\I(t)}
       \frac{w_iQ_i(t)}{\rho_i}\bigl(s_i(t)-\rho_i\bigr)
    \ge
    (\kappa-1)W(t)-\kappa n(\gamma+1)V.
\end{equation}
\end{lemma}

\begin{proof}
For $i\in\I(t)$ and $j\in\J$, define the comparison allocation
\[
    \overline y_{i,j}(t)
    :=\kappa\rho_i\frac{n_j}{n},
\]
and set $\overline y_{i,j}(t)=0$ for $i\notin\I(t)$.  For every
$j\in\J$,
\[
\begin{split}
    \sum_{i\in\I(t)}\overline y_{i,j}(t)
    &=\kappa\frac{n_j}{n}
      \sum_{i\in\I(t)}\rho_i \\
    &\le \kappa\rho\frac{n_j}{n}
     =\frac{n+\rho}{2n}n_j<n_j.
\end{split}
\]
Thus $\overline y(t)$ is feasible.  Moreover,
\[
    \overline s_i(t):=\sum_{j\in\J}\overline y_{i,j}(t)
    =\kappa\rho_i,
    \qquad i\in\I(t).
\]

Optimality of $y(t)$ in the allocation problem gives
\[
\begin{split}
    \sum_{i\in\I(t)}w_iQ_i(t)
       \log\frac{s_i(t)}{\kappa\rho_i}
    &\ge
    V\sum_{i\in\I(t)}\sum_{j\in\J}
      \bigl(\widetilde r_{i,j}(t)-\gamma\bigr)
      \bigl(\overline y_{i,j}(t)-y_{i,j}(t)\bigr) \\
    &=V\sum_{i\in\I(t)}\sum_{j\in\J}
      \bigl(\gamma-\widetilde r_{i,j}(t)\bigr)
      \bigl(y_{i,j}(t)-\overline y_{i,j}(t)\bigr).
\end{split}
\]
Because $\gamma-\widetilde r_{i,j}(t)\in[\gamma-1,\gamma+1]$
and $y_{i,j}(t)\ge0$,
\[
\begin{split}
    \sum_{i\in\I(t)}w_iQ_i(t)
       \log\frac{s_i(t)}{\kappa\rho_i}
    &\ge
    -V(\gamma+1)
      \sum_{i\in\I(t)}\sum_{j\in\J}
      \overline y_{i,j}(t) \\
    &\ge -V(\gamma+1)n.
\end{split}
\]
Using $\log x\le x-1$ for $x>0$, we conclude that
\[
\begin{split}
    -V(\gamma+1)n
    &\le
    \sum_{i\in\I(t)}w_iQ_i(t)
       \left(\frac{s_i(t)}{\kappa\rho_i}-1\right).
\end{split}
\]
Multiplying by $\kappa$ and rearranging yields
\[
    \sum_{i\in\I(t)}
       \frac{w_iQ_i(t)}{\rho_i}s_i(t)
    \ge
    \kappa W(t)-\kappa n(\gamma+1)V.
\]
Subtracting $W(t)$ from both sides proves
\eqref{eq:allocation-service-slack}.
\end{proof}
\begin{equation}
\label{eq:Bdr-definition}
\begin{split}
    B_{\mathrm{dr}}(V)
    := {}& \frac{\mu}{2}\sum_{i\in\I}w_i
    +\frac{\mu n^2}{2}
       \left(1+\mu c_\gamma^2\right)
       \max_{i\in\I}\frac{w_i}{\rho_i} \\
    &\quad +\mu\kappa n(\gamma+1)V .
\end{split}
\end{equation}

Combining Lemmas~\ref{lem:one-step-weighted-drift} and
\ref{lem:allocation-service-slack} gives
\begin{equation}
\label{eq:main-weighted-drift}
    \mathbb E[L(Q(t+1))-L(Q(t))\mid\mathcal F_t]
    \le -\epsilon W(t)+B_{\mathrm{dr}}(V),
\end{equation}
where
\[
    \epsilon=\mu(\kappa-1)=\frac{\mu\delta}{2\rho}
\]
and $B_{\mathrm{dr}}(V)$ is defined in
\eqref{eq:Bdr-definition}. Let \[\alpha_{c,w}=\max_{i\in \I}\frac{c_i}{w_i}.\]

We first derive the quantitative time-average bound.  Taking
expectations in \eqref{eq:main-weighted-drift}, summing from $t=0$ to
$N-1$, and using $L(Q(N))\ge0$, we obtain
\[
    \epsilon\sum_{t=0}^{N-1}\mathbb E[W(t)]
    \le
    NB_{\mathrm{dr}}(V)+\mathbb E[L(Q(0))].
\]
Since
\[
    \sum_{i\in\I}c_iQ_i(t)
    \le
    \alpha_{c,w}\sum_{i\in\I}w_iQ_i(t)
    =\alpha_{c,w}W(t),
\]
it follows that
\[
    \frac1N\sum_{t=0}^{N-1}H(t;c)
    \le
    \alpha_{c,w}
    \left(
       \frac{B_{\mathrm{dr}}(V)}{\epsilon}
       +\frac{\mathbb E[L(Q(0))]}{\epsilon N}
    \right)=  O\left( \max_{i\in \I}\frac{c_i}{w_i}\frac{1}{\delta}(V+Iw_{\max}) \right)
    .
\]
%
%
%

%
%
%
%
%
%
%
%
%
%
%
%
%
%
%
%
%
%
%
%
%
%
%
%

%
%
%
%
%
%
%
%
%
%
%
%
%
%
%
%
%
%
%
%
%
%
%
%
%
%
%
%
%
%

%
%
%
%
%
%
%
%
%
%
%
%
%
%
%
%
%
%
%
%
%
%
%
%
%
%
%
%
%
%
%
%
%
%
%
%
%
%
%

%
%
%
%
%
%
%
%
%
%
%
%
%

\subsection{Comparison with the unstructured-reward setting of
\cite{hsu2021integrated}.}\label{app:com}

The queue-control component of our algorithm is inspired by the
utility-guided dynamic assignment framework of \cite{hsu2021integrated}.
The main difference is in the learning component. In
\cite{hsu2021integrated}, the mean payoffs are unstructured and are learned
separately for each client--server pair using UCB-type estimates. In contrast,
our model assumes that the assignment rewards share a common bilinear
parameter, so that observations from one job--server pair can inform the
estimates of other pairs.

We now explain how the regret bound of \cite{hsu2021integrated} translates
into our cumulative-regret notation. Let \(N_{\rm H}\) denote the average
number of tasks per client in \cite{hsu2021integrated}, to distinguish it
from the horizon \(T\). Their Theorem~2 bounds the time-averaged payoff gap
as
\[
    \overline \Delta_{\rm H}(T)
    \le
    \frac{\beta_1}{V}
    +
    \beta_2
    \sqrt{\frac{\log N_{\rm H}}{N_{\rm H}}}
    +
    \beta_3
    \frac{N_{\rm H}(V+1)}{T},
\]
where \(\beta_1,\beta_2,\beta_3\) are constants depending on the system
parameters. Since our regret is cumulative over the horizon, the corresponding
cumulative regret is
\[
    R_{\rm UGDA\text{-}OL}(T)
    :=
    T\overline \Delta_{\rm H}(T).
\]
Thus,
\[
    R_{\rm UGDA\text{-}OL}(T)
    \le
    \frac{\beta_1 T}{V}
    +
    \beta_2 T
    \sqrt{\frac{\log N_{\rm H}}{N_{\rm H}}}
    +
    \beta_3 N_{\rm H}(V+1).
\]

For comparison with our setting, we treat \(N_{\rm H}\) and the basic service
parameters as fixed constants. Then the second term is linear in \(T\). More
precisely, the coefficient \(\beta_2\) contains a term that scales linearly
with the number \(J\) of server classes, and hence
\[
    \beta_2 T
    \sqrt{\frac{\log N_{\rm H}}{N_{\rm H}}}
    =
    \widetilde O(JT).
\]
This is the unstructured-learning cost: because payoffs are learned separately
for different client--server pairs, the learning loss does not vanish with the
horizon after converting the average payoff gap to cumulative regret.

The remaining terms correspond to the usual control/backlog tradeoff. Using
the backlog bound in \cite{hsu2021integrated}, the finite-horizon backlog term
can be written, with the slack \(\delta\) made explicit, as
\[
    O\!\left(V+\frac{1}{\delta}\right),
\]
after multiplying the time-averaged bound by \(T\). Therefore, under the same
scaling convention as in our regret bound, we obtain
\[
    R_{\rm UGDA\text{-}OL}(T)
    =
    \widetilde O\!\left(
        \frac{IT}{V}
        +
        JT
        +
        V
        +
        \frac{1}{\delta}
    \right),
\]
where the term \(IT/V\) denotes the control term arising from the
Lyapunov-drift analysis, expressed in our notation. Choosing
\[
    V_T=\Theta(\sqrt{IT})
\]
before the algorithm starts yields
\[
    R_{\rm UGDA\text{-}OL}(T)
    =
    \widetilde O\!\left(
        \sqrt{IT}
        +
        JT
        +
        \frac{1}{\delta}
    \right).
\]
Hence, compared with the unstructured-reward setting of
\cite{hsu2021integrated}, our bilinear model replaces the linear-in-\(T\)
unstructured learning term \(JT\) by the bilinear-bandit learning term
\(\widetilde O(d^2\sqrt T)\).

%
%
%
%
%
%

%
%
%
%
%

%

%
%
%
%
%
%

\subsection{Distributed Allocation Algorithms}
\label{app:distribute}

\begin{figure}[t]
\centering
\includegraphics[width=0.8\linewidth]{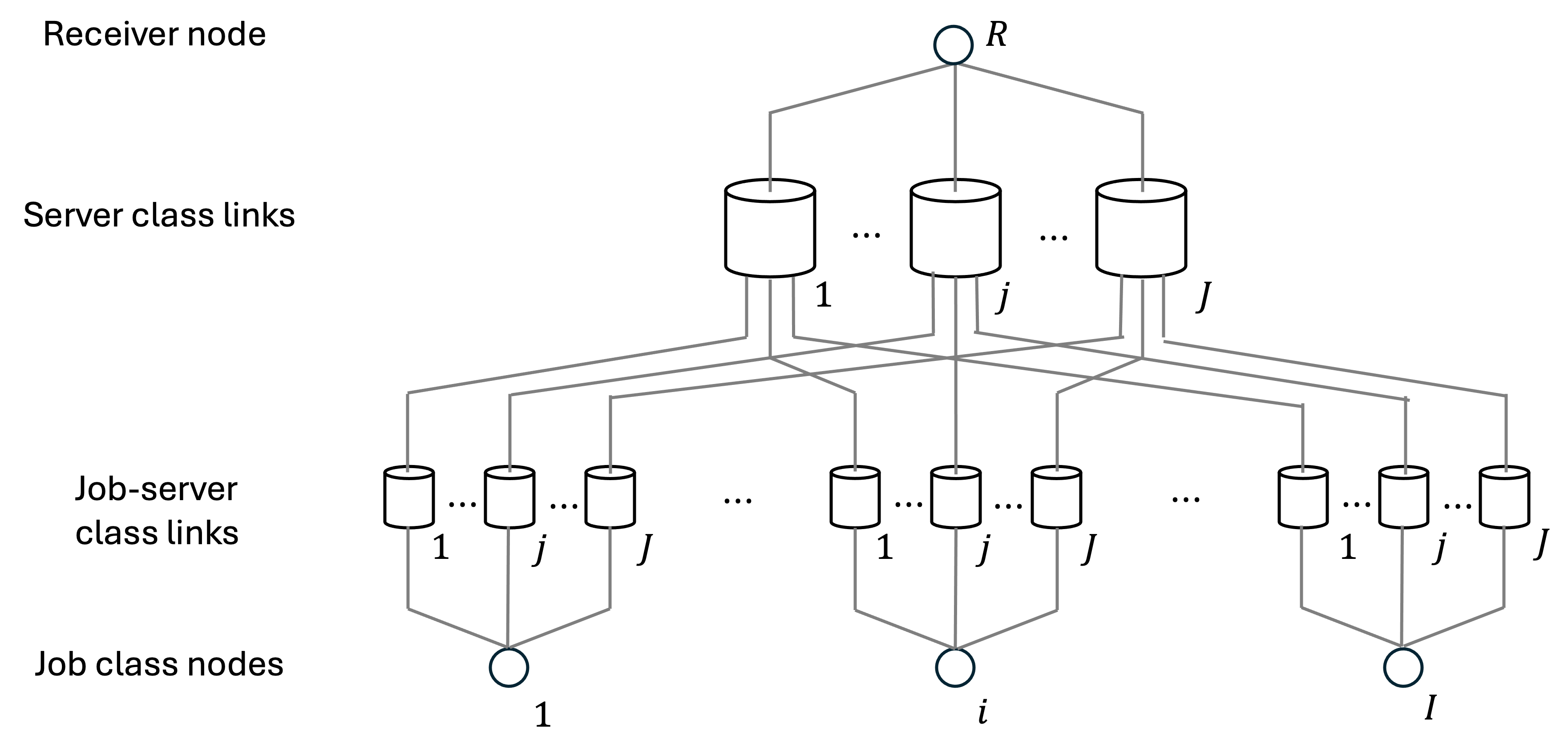}
\caption{Distributed computation of the expected allocation corresponds to a joint routing and rate control problem with a job class node as a source and a virtual receiver node. A routing path $(i,j)$ originates at job node $i$, traverses a link used exclusively by this path with a marginal price $\gamma - \tilde{r}_{i,j}$ and then traverses link $j$ corresponding to a job server node, and finally terminates at the receiver node.} 
\label{fig:bilinear2}
\end{figure}

The scheduling algorithm defined in Algorithm~\ref{alg:Alg1} can be run by a dedicated compute node in a centralized computation implementation. For large-scale systems with many jobs and servers, it is of interest to consider distributed scheduling algorithms, where computations are performed by compute nodes representing jobs or servers. The part of the algorithm concerned with the computation of mean reward estimators can be easily distributed. However, the part concerned with the computation of expected allocations of jobs to servers requires solving the convex optimization problem (\ref{eq:weighted}) in each time step.

In this section, we discuss distributed iterative algorithms for approximately computing these expected allocations, where iterative updates are performed by compute nodes representing jobs or servers. These iterative updates follow certain projected gradient descent-type algorithms with feedback delays due to distributed computation. We provide sufficient conditions for the exponential-rate convergence of these iterative algorithms. The definition of the algorithms and their convergence analysis exploits a relation with the joint routing and rate control problem addressed in \cite{kv05}.

In what follows, we consider an arbitrary time step $t$ and omit reference to $t$ in our notation. For simplicity, we consider $w_i=1$ for $i\in\I$. We write $\Q$ in lieu of $\Q(t)$ and $\I_+$ in lieu of $\I(t)$.  With a slight abuse of notation, we let $y_{i,j}(r)$ denote the value of allocation for a job-server class combination $(i,j)$ at iteration $r$, for $i\in \I_+$ and $j\in \J$, and let $y(r) = (y_{i,j}(r): i\in \I_+, j\in \J)$. We consider iterative updates of allocations under the assumption that the set of jobs and mean reward estimators are fixed. This is a standard assumption when studying such iterative allocation updates in network resource allocation problems. It is a limit case of an operational regime where such iterative updates are run at a faster timescale than the timescale at which the set of jobs and mean reward estimates change. 

We refer to the node maintaining state for a job class $i\in \I$ as \emph{job-node} $i$, and the node maintaining state for a server class $j\in \J$ as \emph{server-node} $j$. Let $\tau_{(i,j)}$ denote the round-trip delay for the $(i,j)$ job-server class, defined as the sum of the delay for communicating information from job-node $i$ to server-node $j$, denoted by $\tau_{i,j}$, and the delay for communicating information in the reverse direction, denoted by $\tau_{j,i}$.

\subsection{Allocation Computed by Job Nodes} We consider distributed computation where each job-node $i\in \I_+$ computes $y_i(t) :=  (y_{i,j}(r): j\in \J)$ for $r\geq 0$ by using the following iterative updates:
\begin{equation}
y_{i,j}(r+1) = y_{i,j}(r) + \alpha_{i,j}\left(1-\frac{\lambda_{i,j}(r)}{u'_i(y^\dag_i(r))}\right)_{y_{i,j}(r)}^+ 
\label{equ:xdiff0}
\end{equation}
where 
\begin{align*}
\lambda_{i,j}(r):= & p_j\left(y^\S_j(r-\tau_{j,i})\right) + \gamma - \tilde{r}_{i,j},\\
y^\dag_i(r):= & \sum_{j\in \J} y_{i,j}(r-\tau_{(i,j)}),\\ 
y^\S_j(r):= & \sum_{i'\in \I_+}y_{i',j}(r-\tau_{i',j}),\\
u_i(y) :=& \frac{1}{V}|\Q_i|\log(y)
\end{align*}
with $\alpha_{i,j} > 0$ being a step size parameter, $p_j$ being a non-negative continuously differentiable function with strictly positive derivative, and $(b)_a^+ = b$ if $a> 0$ and $(b)_a^+ = \max\{b,0\}$ if $a = 0$. Note that $y_i(r)$ can be interpreted as allocation to job-node $i$ acknowledged via feedback from server-nodes $j\in \J$.

We study the convergence properties of the above iterative method in continuous time by considering the system of delay differential equations, for $i\in \I_+$ and $j\in \J$,
\begin{equation}
\frac{d}{dr}y_{i,j}(r) = \alpha_{i,j}\left(1-\frac{\lambda_{i,j}(r)}{u_i'(y^\dag_i(r))}\right)_{y_{i,j}(r)}^+.
\label{equ:xdiff}
\end{equation}
Let $C_j(z) = \int_0^z p_j(u)du$ for $j\in \J$. 
An allocation $y= (y_{i,j} : i\in \I_+, j\in \J)$ is said to be an \emph{equilibrium point} of (\ref{equ:xdiff}) if it is a solution of the convex optimization problem:
\begin{equation}
\hspace{-2mm}\begin{array}{rl}
\!\!\mathrm{maximize}\!\! & \!\! \!\!\sum_{i\in \I_+}\!\! u_i(y^\dag_i) \!- \!\sum_{j\in \J} (C_j(y^\S_j)\!+\!\sum_{i\in \I_+}\!\!(\gamma \!- \! \tilde{r}_{i,j}) y_{i,j})\\
\mathrm{subject\: to}\!\! &\!\! y^\dag_i = \sum_{j\in \J} y_{i,j}, \hbox{ for all } i\in \I_+\\
&\!\!  y^\S_j = \sum_{i\in \I_+} y_{i,j}, \hbox{ for all } j\in \J\\
\mathrm{over} \!\! &\!\!  y_{i,j}\geq 0, \hbox{ for all } i\in \I_+, j\in \J.
\end{array}
\label{equ:opt2}
\end{equation}

The objective function of (\ref{equ:opt2}) is maximized at $y$ if, and only if, for all $i\in \I_+$ and $j\in \J$,
\begin{eqnarray}
&& y_{i,j}\geq 0,\label{equ:oc1}\\
&& u_i'(y^\dag_i)-p_j(y^\S_j) - (\gamma - \tilde{r}_{i,j}) \geq 0, \hbox{ and }\label{equ:oc2}\\ 
&& y_{i,j}\left(u_i'(y^\dag_i)-p_j(y^\S_j) - (\gamma - \tilde{r}_{i,j})\right) = 0.\label{equ:oc3}
\end{eqnarray}

A point $y$ satisfying (\ref{equ:oc1})-(\ref{equ:oc3}) is said to be an \emph{interior point} if either (\ref{equ:oc1}) or (\ref{equ:oc2}) holds with strict inequalities. 

The iterative updates (\ref{equ:xdiff0}) are of gradient descent type as, when all feedback communication delays are equal to zero, we have $1-\lambda_{i,j}(r)/u_i'(y^\dag_i(r)) = (1/u_i'(y^\dag_i(r)))\partial f(y(r))/\partial y_{i,j}(r)$ where $f$ is the objective of the optimization problem (\ref{equ:opt2}).  

The system of delay differential equations (\ref{equ:xdiff}) and the optimization problem (\ref{equ:opt2}) correspond to a special instance of a joint routing and rate control problem formulation studied in \cite{kv05}, where $(i,j)$ is the index of a route, $i$ is the index of a source, $j$ is the index of a link, and each route $(i,j)$ passes through link $j$ with a cost function $C_j$. Additionally, there is a link with a fixed price per unit flow $\gamma - \tilde{r}_{i,j}\geq 0$ that is used exclusively by route $(i,j)$. See Figure~\ref{fig:bilinear2} for a graphical illustration.

\subsection{Allocation Computed by Server Nodes} Another distributed algorithm for computing a maximizer $y$ of the optimization problem (\ref{equ:opt2}) is defined by letting each server-node $j\in \J$ compute values $(y_{i,j}(r): i\in \I_+, r\geq 0)$ using iterative updates with the associated system of delay differential equations, for $i\in \I_+$ and $j\in \J$,
\begin{equation}
\frac{d}{dr}y_{i,j}(r) = \alpha_{i,j}\left(1-\frac{\lambda_{i,j}(r)}{u'_i(y^\dag_i(r-\tau_{i,j}))}\right)^+_{y_{i,j}(r)}, 
\label{equ:xdiff2}
\end{equation}
where $
\lambda_{i,j}(r)= p_j\left(y^{\S}_j(r)\right) + \gamma - \tilde{r}_{i,j}$,
$y^\dag_i(r) = \sum_{j'\in \J} y_{i,j'}(r-\tau_{j',i}),$ and $y^\S_j(r) = \sum_{i'\in \I_+}y_{i',j}(r-\tau_{(i',j)}).$ Here, $y^\S_j(r)$ represents the total allocation given by server class $j$, acknowledged to be received by job-nodes $i\in \I_+$ via feedback sent to server-node $j$.

\subsection{Stability Condition} 
We provide a condition that ensures convergence of $(y^\dag(r), y^\S(r))$ to a unique point corresponding to the solution of optimization problem (\ref{equ:opt2}) as $r$ approaches infinity. Here, $(y^\dag(r), y^\S(r))$ evolves according to either the system of delay differential equations (\ref{equ:xdiff}) or (\ref{equ:xdiff2}). We define $\tau_{\max}$ as an upper bound on the round-trip delay for each job-server class combination, i.e., $\tau_{(i,j)}\leq \tau_{\max}$ for all $i\in\I_+$ and $j\in \J$.

\begin{theorem} Assume that $y^*$ is an interior equilibrium point, ${y^*}^\dag :=( \sum_{j\in \J}y_{i,j}^*: i\in \I_+)$, ${y^*}^\S := (\sum_{i\in \I_+}y_{i,j}^*: j\in \J)$, and that the following condition holds: for all $i\in \I_+$ and $j\in \J$,
\begin{equation}
\alpha_{i,j}\tau_{(i,j)}\left(1+\frac{p_j'({y_j^*}^\S){y_j^*}^\S}{p_j({y_j^*}^\S) + \gamma - \tilde{r}_{i,j}}\right)<\frac{\pi}{2}.
\label{equ:cond}
\end{equation}
Then, there exists a neighborhood $\mathcal{N}$ of $y^*$ such that for any initial trajectory $y(-\tau_{\max}),\ldots, y(0)$ lying within $\mathcal{N}$, $(y_i^\dag (r): i\in \I_+)$ and $(y_j^\S(r): j\in\J)$, evolving according to (\ref{equ:xdiff}) (resp. according to (\ref{equ:xdiff2})), converge exponentially fast, as $r$ goes to infinity, to the unique points ${y^*}^\dag$ and ${y^*}^\S$, respectively. 
\label{prop:cond}
\end{theorem}

For the system of delay differential equations (\ref{equ:xdiff}), the result in Theorem~\ref{prop:cond} follows from Theorem~2 in \cite{kv05}, as it corresponds to a special instance of the joint routing and rate control problem considered therein. The proof relies on the application of a generalized Nyquist stability criterion to a linearized system of delay differential equations. Similarly, for the system of delay differential equations (\ref{equ:xdiff2}), the proof follows the same approach, as detailed in the Appendix. 

Theorem~\ref{prop:cond} provides a sufficient condition for the exponentially fast convergence of $(y^\dag(r), y^\S(r))$ as $r$ approaches infinity. Intuitively, this condition requires the step size $\alpha_{i,j}$ to be sufficiently small relative to the reciprocal of the round-trip delay $\tau_{(i,j)}$, and involves terms related to the function $p_j$ and its derivative, as well as the marginal price $\gamma - \tilde{r}_{i,j}$, for each job-server class pair $(i,j)$. Since $\gamma - \tilde{r}_{i,j} \geq 0$, we can strengthen the sufficient condition by replacing $\gamma - \tilde{r}_{i,j}$ in (\ref{equ:cond}) with $0$. 

For concreteness, we discuss the sufficient condition from Theorem~\ref{prop:cond} for specific classes of penalty functions $C_j$.

First, consider the penalty functions $C_j$ such that $C_j'(z) = p_j(z) = (z/n_j)^{\beta_j}$, where $\beta_j >0$ for all $j\in \J$. Intuitively, the larger the value of parameter $\beta_j$, the closer the penalty function $C_j$ to the hard capacity constraint. From Theorem~\ref{prop:cond}, we derive the following sufficient stability condition:
$$
\alpha_{i,j}\tau_{(i,j)} < \frac{\pi}{2}\frac{1}{1+\beta_j}, \quad \text{for all } i\in \I_+, j\in \J.
$$
This condition implies that the step size $\alpha_{i,j}$ must be smaller than a constant factor of $1/\tau_{(i,j)}$, with the magnitude of this factor decreasing as $\beta_j$ increases.

Second, we can accommodate a broader set of penalty functions $C_j$ such that $p_j'(z)z \leq \beta_j p_j(z)$ and $p_j(z)\leq \gamma_j$ for all $z\geq 0$, where $\beta_j > 0$, $\gamma_j > 0$, and $\gamma - \tilde{r}_{i,j} \geq c$ for all $i\in \I_+$, $j\in \J$, where $c\geq 0$. Then, we have the sufficient stability condition: for all $i\in \I_+$ and $j\in \J$,
$$
\alpha_{i,j}\tau_{(i,j)} < \frac{\pi}{2}\frac{\gamma_j + c}{\gamma_j + \gamma_j \beta_j + c}. 
$$

Finally, consider the penalty function $C_j$ with $p_j(z) = (z/n_j)/(1-z/n_j)$, for $z\geq 0$, for all $j\in \J$. This function has a vertical asymptote at $z=n_j$. %
Assume $y^*$ satisfies ${y_j^*}^\S/n_j \leq 1-\epsilon$ for all $j\in \J$, where $\epsilon \in (0,1)$. Then, the following sufficient stability condition holds: for all $i\in \I_+$ and $j\in \J$,
$$
\alpha_{i,j}\tau_{(i,j)} < \frac{\pi}{4}\epsilon.
$$

\subsection{Proof of Theorem~\ref{prop:cond}}
\label{sec:cond}

For the system of delay differential equations (\ref{equ:xdiff}), the result in the theorem follows from Theorem~2 in \cite{kv05}. The same proof steps can be followed to establish the result in the theorem for the system of delay differential equations (\ref{equ:xdiff2}), which we explain in this section. 

Let $y_{i,j}(r) = y_{i,j} + u_{i,j}(r)$, $y_i^\dag(r) = y_i^\dag + v_i(r)$, and $y_j^\S(r) = y_j^\S + w_j(r)$. Then, by linearizing the system (\ref{equ:xdiff2}) about $y$, we obtain
$$
\frac{d}{dr} u_{i,j}(r) = -\frac{\alpha_{i,j}y_{i,j}}{\lambda_{i,j}}\left((-u_i''(y_i^\dag)) v_i(r-\tau_{i,j}) + p_j'(y_j^\S) w_j(r)\right)
$$
with
$$
v_i(r) = \sum_{j\in \J} u_{i,j}(r-\tau_{j,i})
$$
and
$$
w_j(r) = \sum_{i\in \I_+} u_{i,j}(r-\tau_{(i,j)})
$$
where $\lambda_{i,j} := p_j(y_j^\S) + \gamma - \hat{r}_{i,j}$.

With a slight abuse of notation, let $u_{i,j}(\omega)$, $v_i(\omega)$, $w_j(\omega)$ denote the Laplace transforms of $u_{i,j}(r)$, $v_i(r)$, and $w_j(r)$, respectively. Then, we have
$$
\omega u_{i,j}(\omega) = -\frac{\alpha_{i,j}y_{i,j}}{\lambda_{i,j}}\left((-u_i''(y_i^\dag)) e^{-\omega \tau_{i,j}}v_i(\omega) + p_j'(y_j^\S) w_j(\omega)\right)
$$
$$
v_i(\omega) = \sum_{j\in \J} e^{-\omega \tau_{j,i}} u_{i,j}(\omega)
$$
and
$$
w_j(\omega) = \sum_{i\in \I_+} e^{-\omega \tau_{(i,j)}}u_{i,j}(\omega).
$$

From these equations, we have
$$
\left(\begin{array}{c} {v}(\omega)\\ {w}(\omega)\end{array}\right) = - P^{-1} R(-\omega)^\top X(\omega) R(\omega) P \left(\begin{array}{c} {v}(\omega)\\ {w}(\omega)\end{array}\right)
$$
where $P$ is the $(|\I_+|+J) \times (|\I_+|+J)$ diagonal matrix with diagonal elements $P_{i,i} = \sqrt{-u_i''(y_i^\dag)}$ and $P_{j,j} = \sqrt{p_j'}$, $X(\omega)$ is the $|\I_+|J \times |\I_+|J$ diagonal matrix with diagonal elements $X(\omega)_{(i,j),(i,j)} = e^{-\omega \tau_{(i,j)}}/(\omega \tau_{(i,j)})$, and $R(\omega)$ is the $|\I_+|J \times (|\I_+|+J)$ matrix with elements
$$
R_{(i,j),i}(\omega) = \sqrt{\frac{\alpha_{i,j}y_{i,j}}{\lambda_{i,j}}\tau_{(i,j)}(-u_i''(y_i))}
$$
and
$$
R_{(i,j),j}(\omega) = \sqrt{\frac{\alpha_{i,j}y_{i,j}}{\lambda_{i,j}}\tau_{(i,j)}p_j'(y_j^\S)}e^{\omega \tau_{i,j}}.
$$
The matrix $G(\omega)=P^{-1}R(-\omega)^\top X(\omega) R(\omega) P$ is called the return ratio for $({v},{w})$. By the generalized Nyquist stability criterion, it is sufficient to prove that the eigenvalues of $G(\omega)$ do not encircle the point $-1$ for $w= i\theta$, $\theta \in \reals$, in order for $({v}(r), {w}(r))$ to converge to $0$ exponentially fast as $r$ goes to infinity.

If $\lambda$ is an eigenevalue of $G(i\theta)$, then we can find a unit vector ${\nu}$ such that
$$
\lambda = {\nu}^* R(i\theta)^* X(i\theta) R(i\theta){\nu}.
$$
Since $X$ is diagonal, we have
$$
\lambda = \sum_{(i,j)} |(R(i\theta){\nu})_{(i,j)}|^2\frac{e^{-i\theta \tau_{(i,j)}}}{i\theta \tau_{(i,j)}}.
$$
Hence, it follows that $\lambda = K {\xi}$ where $K = ||R(i\theta)\nu||^2$ and ${\xi}$ lies in the convex hull of the points
$$
\left\{\frac{e^{-i\theta \tau_{(i,j)}}}{i\theta \tau_{(i,j)}}: i
\in \I_+, j\in \J\right\}.
$$
This convex hull includes the point $-2/\pi$ on its boundary but contains no point on the real axis to the left of $-2/\pi$. Hence, if $\lambda$ is real then $\lambda \geq (-2/\pi)K$. It remains to show that $K < \pi/2$. 

Let $\rho(\cdot)$ denote the spectral radius and $||\cdot ||_\infty$ denote the maximum row sum matrix norm. Let $Q$ be the $(|\I_+|+J)\times (|\I_+|+J)$ diagonal matrix with diagonal elements $Q_{i,i} = y_i^\dag \sqrt{-u_i''(y_i^\dag)}$ and $Q_{j,j} = y_j^\S \sqrt{p_j'(y_j^\S)}$. Then,
\begin{eqnarray*}
K &=& {\nu}^* R(i\theta)^* R(i\theta){\nu}\\
&\leq & \rho(R(i\theta)^* R(i\theta))\\
&=& \rho(Q^{-1}R(i\theta)^* R(i\theta)Q)\\
&\leq & ||Q^{-1}R(i\theta)^* R(i\theta)Q||_\infty.
\end{eqnarray*}

We first consider rows of $Q^{-1}R(i\theta)^* R(i\theta)Q$ corresponding to $i\in \I_+$. Let us fix an arbitrary $i\in \I_+$. Note that for all $j\in \J$,
\begin{eqnarray*}
(Q^{-1}R(i\theta)^* R(i\theta)Q)_{i,j} &=& \frac{Q_{j,j}}{Q_{i,i}}R(i\theta)_{(i,j),i}R(i\theta)_{(i,j),j}\\
&=& \frac{y_{i,j}}{y_i^\dag} \frac{\alpha_{i,j}}{\lambda_{i,j}}\tau_{(i,j)} p_j'(y_j^\S) y_j^\S e^{i\theta \tau_{(i,j)}},
\end{eqnarray*}
\begin{eqnarray*}
(Q^{-1}R(i\theta)^* R(i\theta)Q)_{i,i} &=& \sum_{j\in \J} R(i\theta)^*_{(i,j),i}R(i\theta)_{(i,j),i}\\
&=& \sum_{j\in \J} \frac{y_{i,j}}{y_i^\dag} \frac{\alpha_{i,j}}{\lambda_{i,j}}\tau_{(i,j)} (-u_i''(y_i^\dag)) y_i^\dag,
\end{eqnarray*}
and $(Q^{-1}R(i\theta)^* R(i\theta)Q)_{i,i'}$ for $i,i'\in \I_+$ such that $i\neq i'$. It follows
\begin{eqnarray*}
&& \sum_{j\in \J}|(Q^{-1}R(i\theta)^* R(i\theta)Q)_{i,j}| + \sum_{i'\in \I_+} |(Q^{-1}R(i\theta)^* R(i\theta)Q)_{i,i'}|\\
&=& \sum_{j\in \J} \frac{y_{i,j}}{y_i^\dag}\left\{\frac{\alpha_{i,j}}{\lambda_{i,j}}\tau_{(i,j)}((-u_i''(y_i^\dag))y_i^\dag + p_j'(y_j^\S)y_j^\S)\right\}\\
&=& \sum_{j\in \J} \frac{y_{i,j}}{y_i^\dag}\left\{\alpha_{i,j}\tau_{(i,j)}\left(1 + \frac{p_j'(y_j^\S)y_j^\S}{p_j(y_j^\S) + \gamma - \hat{r}_{i,j}}\right)\right\}\\
&<& \frac{\pi}{2}
\end{eqnarray*}
where the last equation holds because $u'_i(y_i^\dag) = \lambda_{i,j}$, $-u''_i(y_i^\dag) y_i^\dag = u'_i(y_i^\dag)$, and $\lambda_{i,j} = p_j(y_j^\S) + \gamma - \hat{r}_{i,j}$, and the last inequality is by condition (\ref{equ:cond}). 

It remains to consider rows of $Q^{-1}R(i\theta)^* R(i\theta)Q$ corresponding to $j\in \J$. By similar arguments, we can show that for every $j\in \J$,
\begin{eqnarray*}
&& \sum_{i\in \I_+}|(Q^{-1}R(i\theta)^* R(i\theta)Q)_{j,i}| + \sum_{j'\in \J} |(Q^{-1}R(i\theta)^* R(i\theta)Q)_{j,j'}|\\
&=& \sum_{i\in \I_+}  \frac{y_{i,j}}{y_j^\S}\left\{\alpha_{i,j}\tau_{(i,j)}\left(1 + \frac{p_j'(y_j^\S)y_j^\S}{p_j(y_j^\S) + \gamma - \hat{r}_{i,j}}\right)\right\}\\
&<& \frac{\pi}{2}.
\end{eqnarray*}

\subsection{Experiments using Real-World Data}\label{app:real}

In this section, we present numerical results evaluating our proposed algorithm for scheduling servers in cluster computing systems. We conduct this evaluation using the dataset \emph{cluster-data-2019}, which contains information about jobs and servers in the Google Borg cluster system. This dataset is publicly available, and details about it can be found at \url{https://github.com/google/cluster-data} and in references \cite{borg1, borg2}. The dataset includes information about various entities, such as \emph{machines}, \emph{collections}, and \emph{instances}. Machines are servers with different CPU characteristics and memory capacities, while collections refer to jobs submitted to the cluster, each consisting of one or more tasks known as instances.

For our experiments, we utilized data collected over a time interval from the beginning of the trace to 5,000 seconds into the trace. The dataset comprises 9,526 machines that were active before the start of the measurement interval, along with 71 enqueued collections. Each machine's data includes information about CPU and memory capacity, while each collection's data includes CPU and memory request sizes for each instance. We leveraged this information to construct feature vectors for collections and machines.

To represent each collection, we calculated the average CPU and memory request sizes of its instances. Using these averages, we employed the K-means clustering algorithm to cluster collections into five different classes, with each class represented by the average values of CPU and memory request sizes. For machines, we identified 12 different classes based on their CPU and memory capacities. Additionally, we included the inverse values of CPU and memory (request) capacities, resulting in feature vectors of dimension $d=4$.

The dataset also contains information about the average number of cycles per instruction (CPI) for each instance assigned to a machine. The inverse CPI for each instance-machine combination reflects the efficiency of instance execution on the machine, depending on the computing and machine characteristics. We used these inverse CPIs to define stochastic rewards for assignments. For reward sampling, we drew samples from a Gaussian distribution with computed mean and variance derived from observed rewards. More detailed explanations of our experiments are provided in Appendix~\ref{app:real_exp}.

\begin{figure}[t]
\begin{center}
      \includegraphics[width=0.35\linewidth]{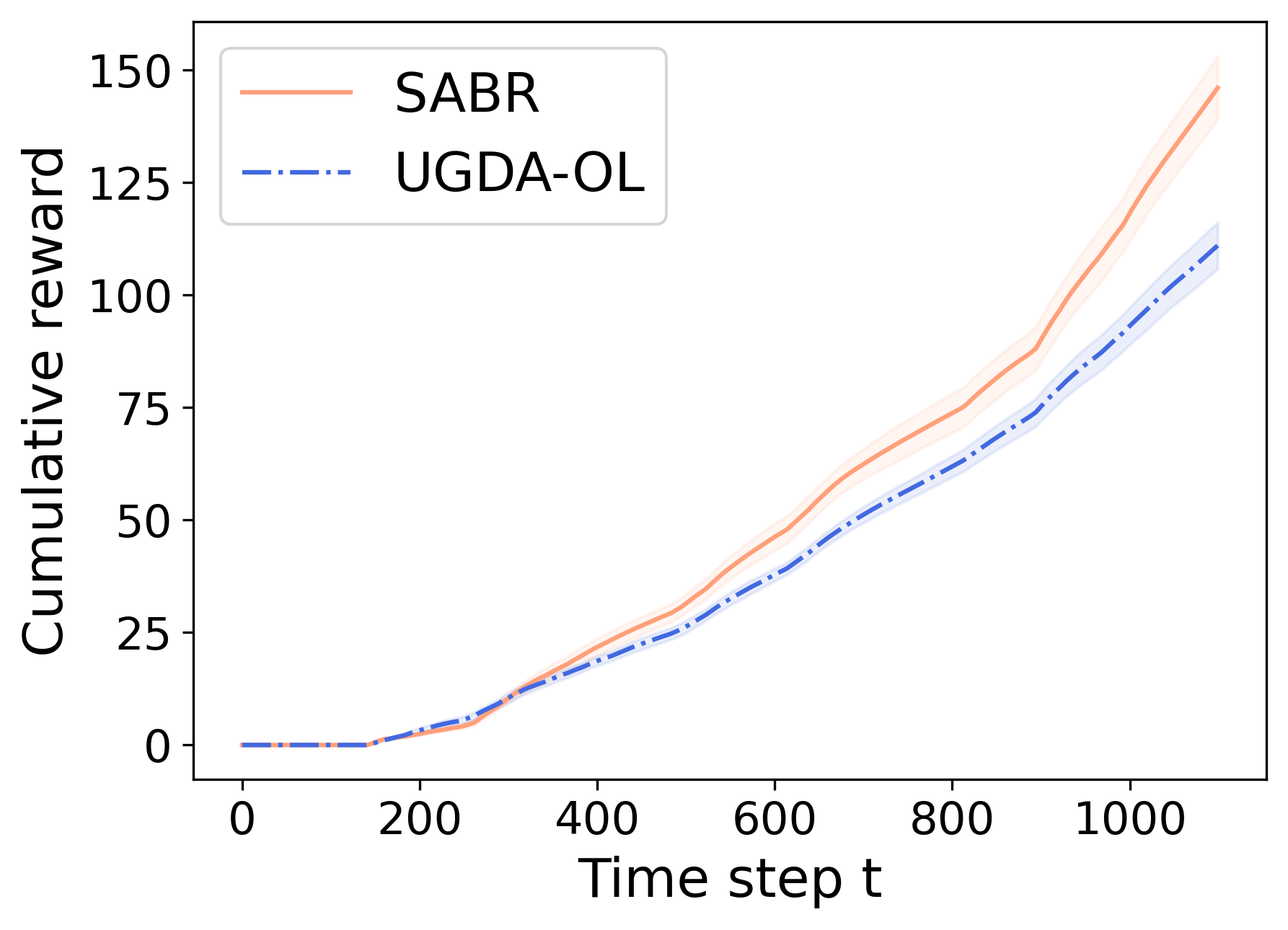}\hspace*{1.5cm}
    \includegraphics[width=0.35\linewidth]{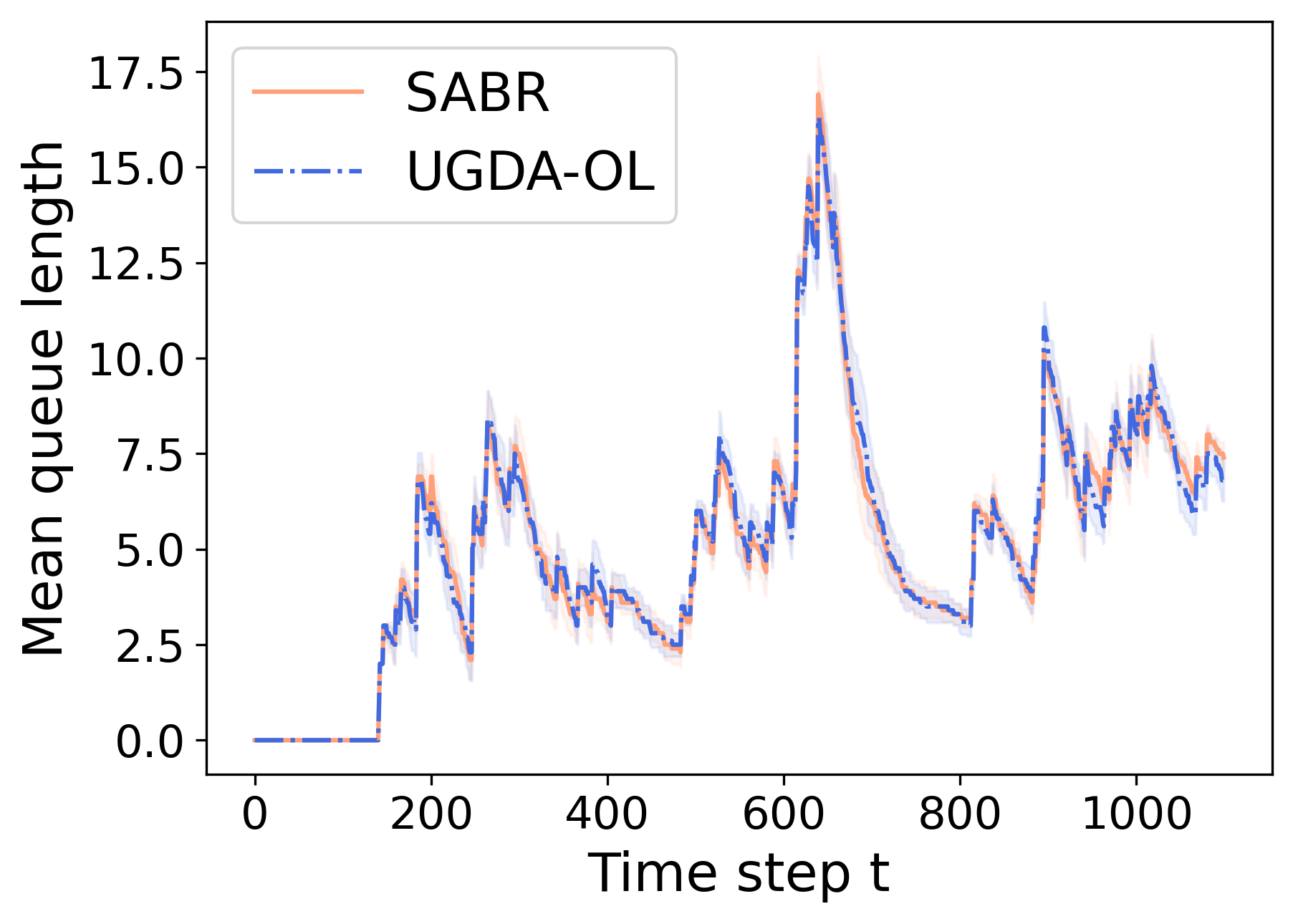}
\end{center}
\caption{Performance of \texttt{SABR} and \texttt{UGDA-OL} over time steps: (left) cumulative reward and (right) mean queue length. 
}
\label{fig:real_data}
\end{figure}
\begin{figure}[t]
\begin{center}
      \includegraphics[width=0.35\linewidth]{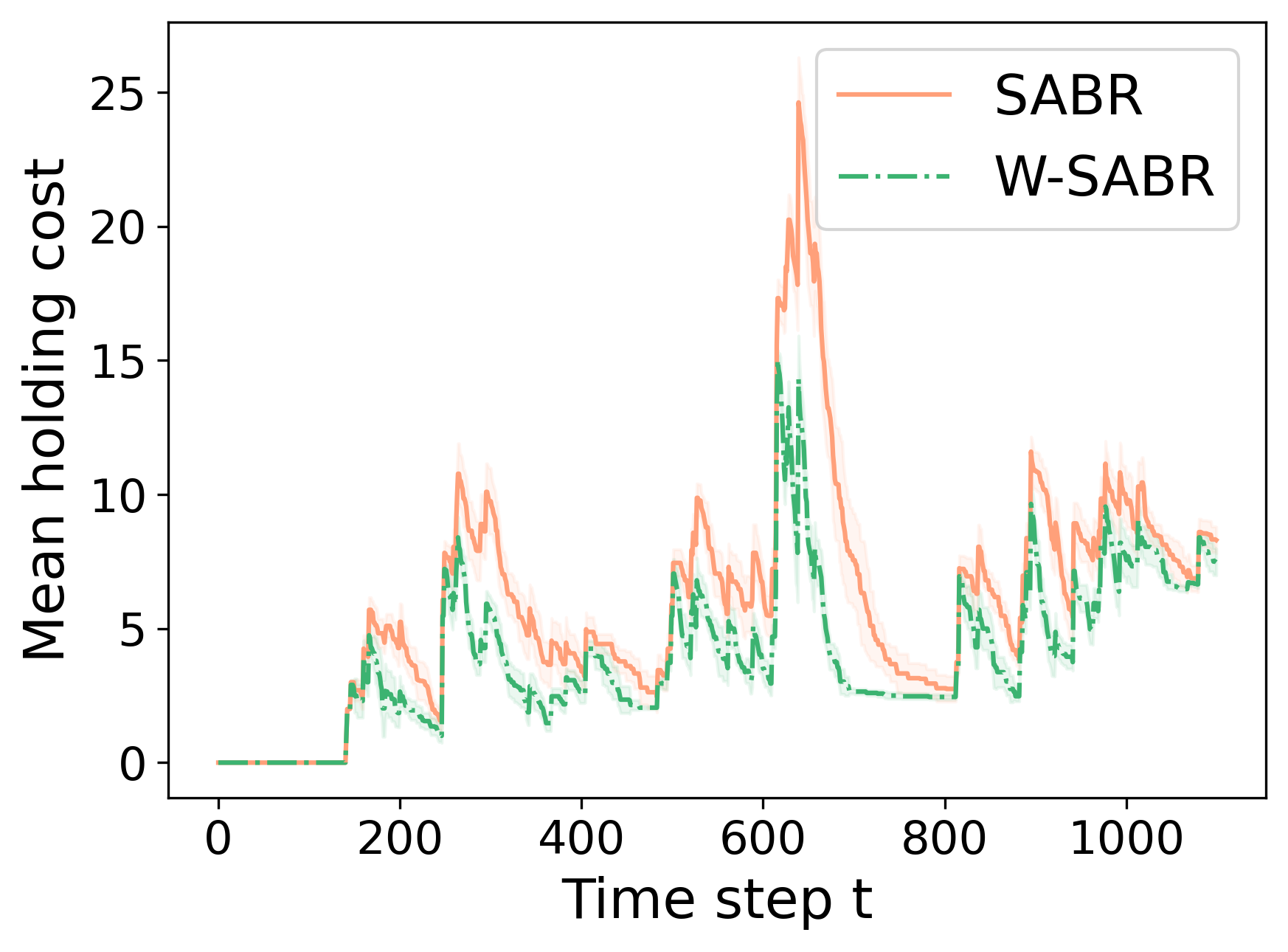}\hspace*{1.5cm}
      \includegraphics[width=0.35\linewidth]{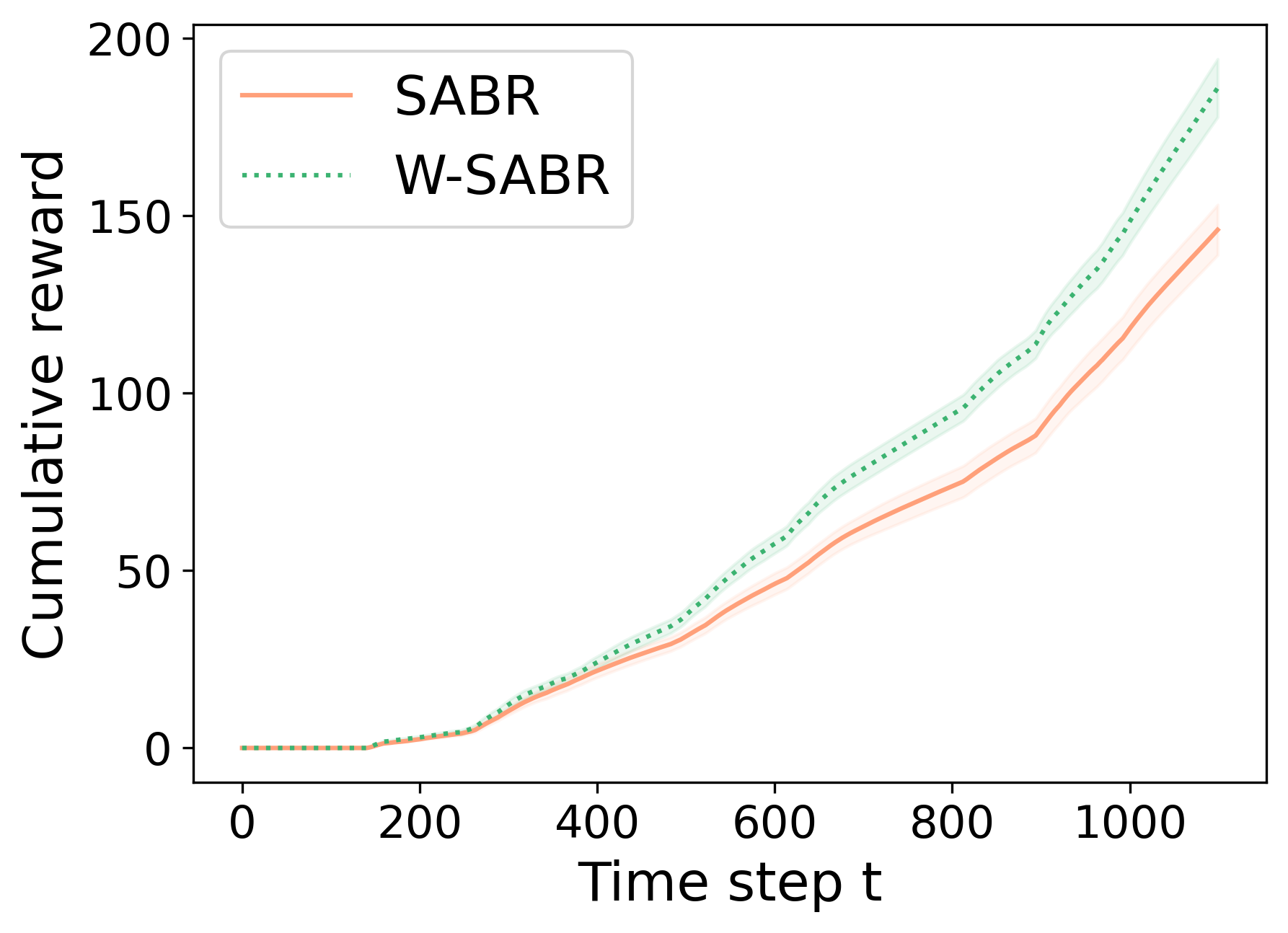}
\end{center}
\caption{ Mean holding cost and cumulative reward of \texttt{SABR} and \texttt{W-SABR} over time steps.
}
\label{fig:real_hold_cost}
\end{figure}

\red{\paragraph{Connection to the cluster trace}
In the Google Borg trace experiment, collection features are constructed from
CPU and memory requests, while machine features are constructed from CPU and
memory capacities and their inverse values. The observed inverse CPI depends
jointly on the collection and machine characteristics, making it a natural
example of a compatibility reward that can be approximated by bilinear
feature interactions.}

We execute scheduling algorithms in discrete time steps, each spanning a 5-second interval of real time, resulting in $T = 1,100$ time steps. All instances assigned to machines within a time step are assumed to be fully processed during that interval. Each machine can handle at most one instance per time step.

In Figure~\ref{fig:real_data}, we compare the performance of our algorithm \texttt{SABR} with \texttt{UGDA-OL} in terms of cumulative rewards and mean queue lengths at different time steps. We observe that \texttt{SABR} outperforms \texttt{UGDA-OL} in cumulative rewards, while both algorithms exhibit comparable mean queue lengths.

Next, we examine how the mean holding cost varies with time steps for \texttt{SABR} and \texttt{W-SABR}. For the job holding costs across five different job classes, we set $c_i=7/4$ for two high-priority job classes, $c_i=1$ for one medium-priority job class, and $c_i=1/4$ for the remaining low-priority job classes. In \texttt{W-SABR}, we set $w_i=c_i$, while in \texttt{SABR}, we set $w_i=1$. Figure~\ref{fig:real_hold_cost} shows that \texttt{W-SABR} exhibits better mean holding costs and cumulative rewards than \texttt{SABR} in most time steps.

\subsubsection{Details for the Experiments using Real-world Data}\label{app:real_exp}

\begin{figure}[t]
      \centering
      \includegraphics[width=0.35\linewidth]{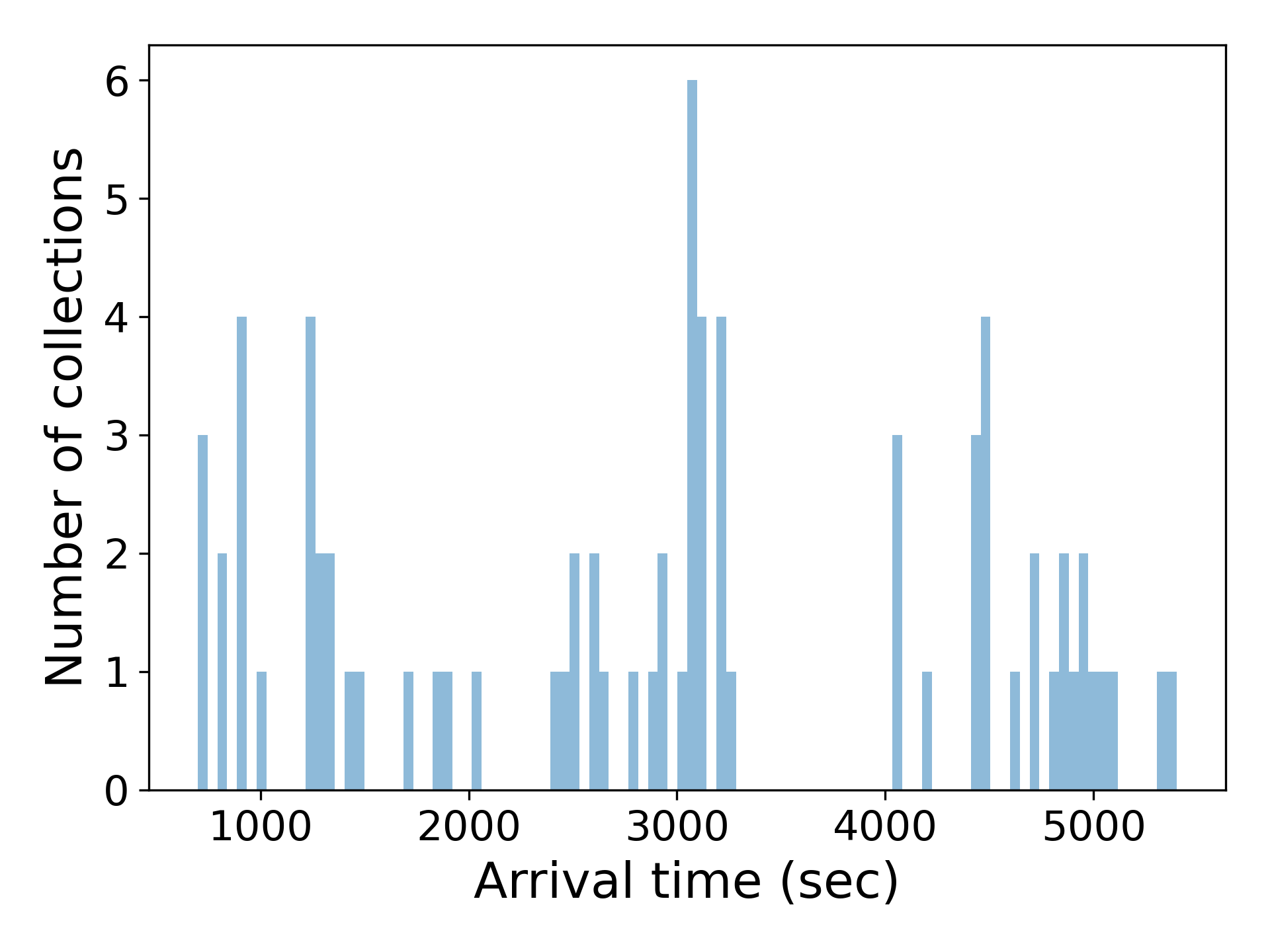}\hspace*{0.5cm}
      \includegraphics[width=0.35\linewidth]{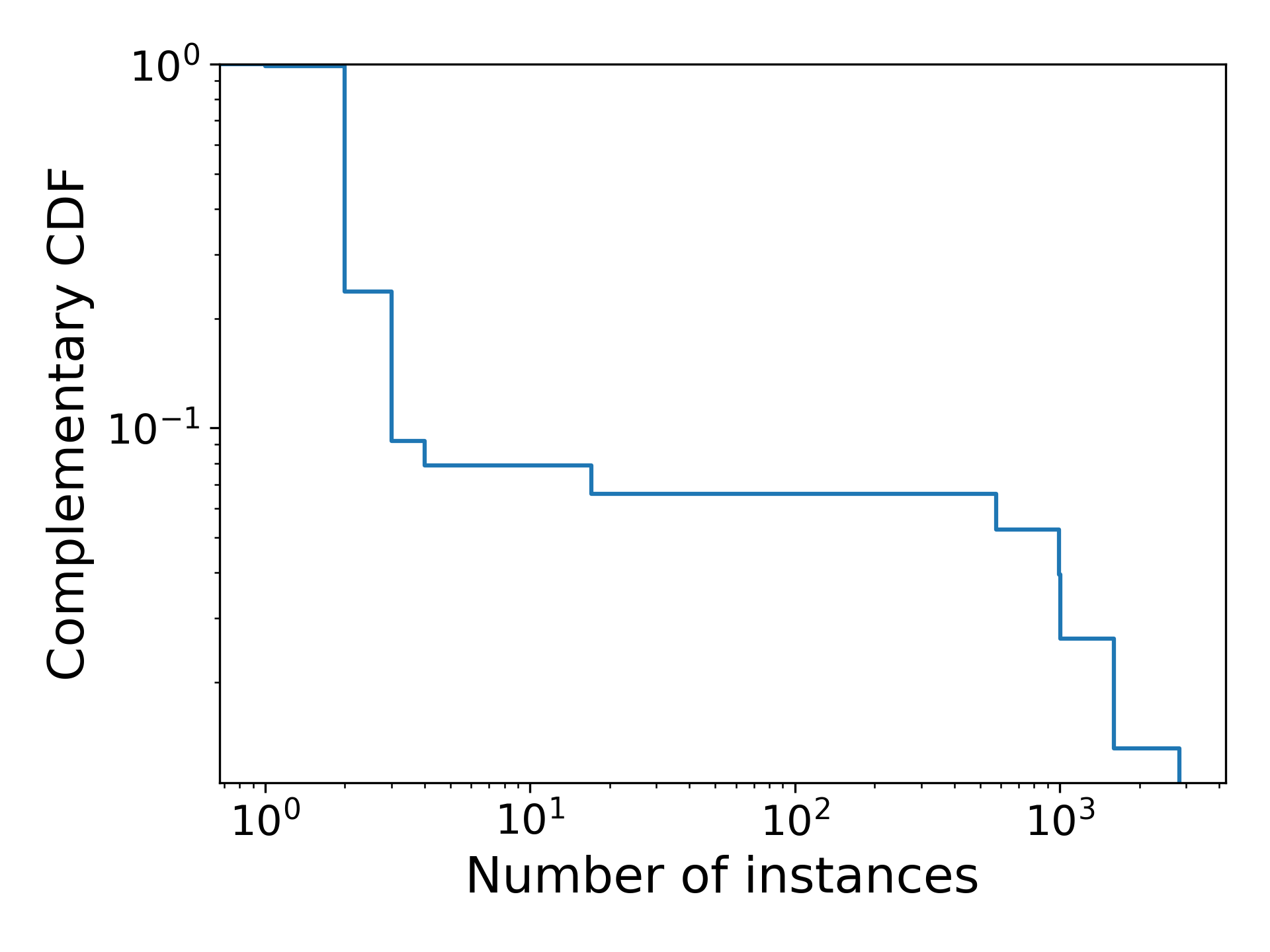}
\caption{Characteristics of the input workload: (left) arrival of collections over time, (right) complementary cumulative distribution function of the number of instances per collection.}
\label{fig:workload}
\end{figure}

\begin{figure}[t]
\begin{center}
      \includegraphics[width=0.35\linewidth]{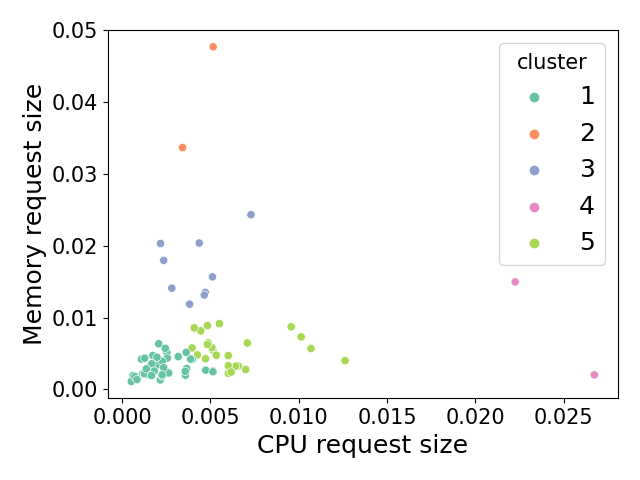}\hspace*{0.5cm}
    \includegraphics[width=0.35\linewidth]{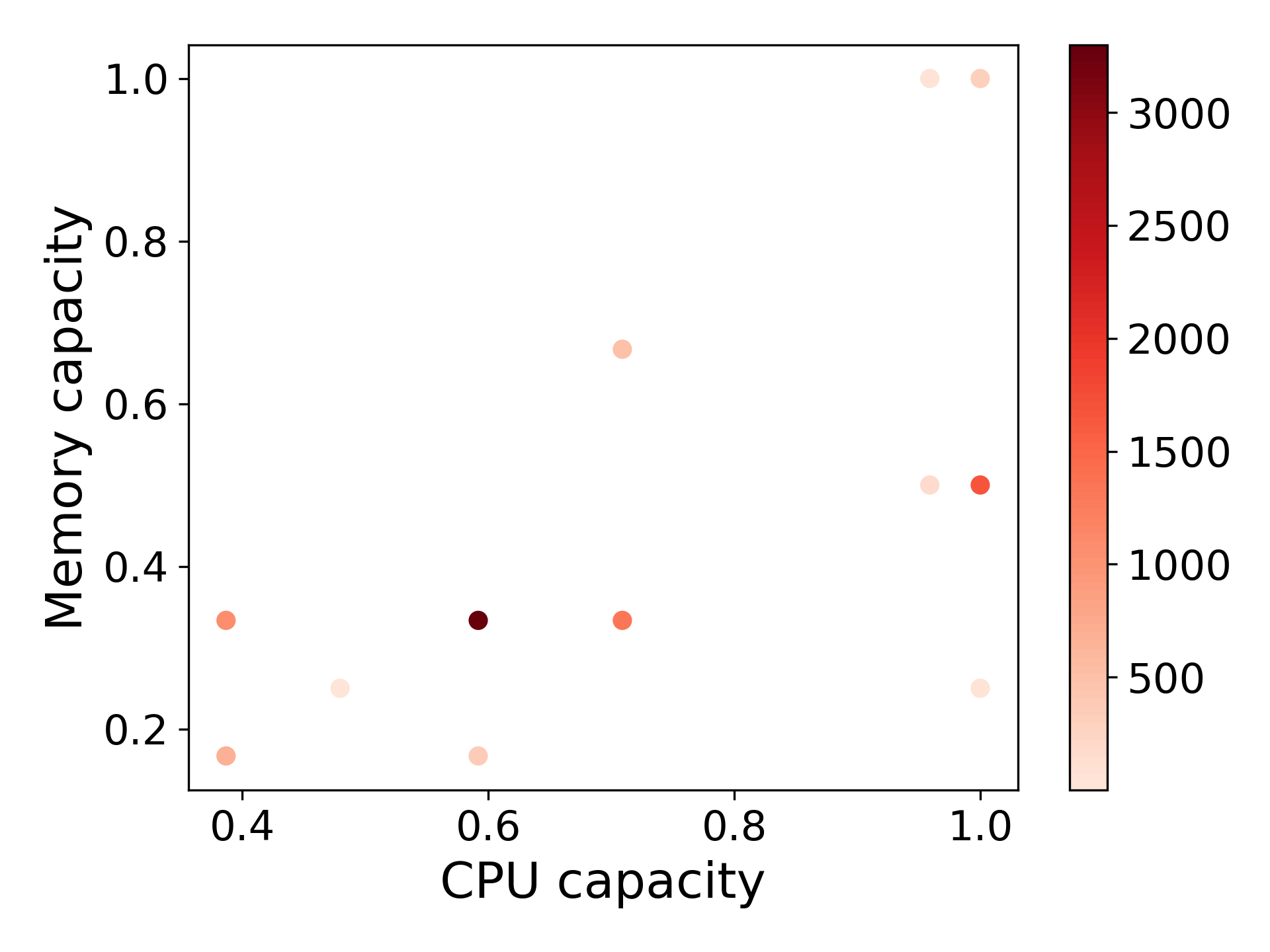}
\end{center}
\caption{Feature vectors for collections and machines: (left) individual collections classified by K-means clustering with $K=5$, and (right) machine types with counts. %
}
\label{fig:featureVec}
\end{figure}

In this section, we provide details about the experiments using the dataset \emph{cluster-data-2019}.

\paragraph{Basic Information about the Workload} For our experiments, we used the data collected over a time interval from the beginning of the trace to 5,000 seconds. The dataset comprises 9,526 machines that were active before the measurement interval commenced, with 71 collections enqueued during this period. The arrival pattern of collections over time is depicted in Figure~\ref{fig:workload} (left). Additionally, Figure~\ref{fig:workload} (right) illustrates the distribution of the number of instances per collection. It is notable that this distribution exhibits a heavy skew, with numerous collections comprising only a few instances, and a small fraction of collections comprising a significant number of instances.

\paragraph{Features of Jobs and Servers}
The dataset provides information about the CPU and memory capacity of each machine, as well as the CPU and memory request size for each instance. We utilize this information to construct feature vectors for both collections and machines.

For collections, we initially represent each collection by averaging the CPU and memory request sizes of its instances. We then employ the K-means clustering algorithm to group collections into five distinct classes based on these representations. The resulting representations of collections and their clustering into classes are depicted in Figure~\ref{fig:featureVec} (left). Each class of collections is characterized by the average CPU and memory request size of instances within that class.

Regarding machines, we identify 12 different classes based on their CPU and memory capacities, as illustrated in Figure~\ref{fig:featureVec} (right). In addition to considering CPU and memory capacity values, we also incorporate their inverse values, resulting in feature vectors of dimension $d=4$. This feature engineering approach is adopted to capture inverse relationships, which are crucial when utilizing a bilinear model.

\subsubsection{Rewards of Assignments}
The dataset provides information regarding the average number of cycles per instruction (CPI) for each assignment of an instance to a machine. The inverse CPI computed for an instance-machine pair indicates how efficiently the instance is executed by the machine. This performance metric relies on both the characteristics of the computing task and the machine itself.

We utilize the inverse CPIs to define the rewards of assignments. This is achieved by computing the mean and variance of observed rewards for each combination of a collection class and a machine class. For any combination where no assignments are observed in the data, we set the mean reward to zero.

In our simulations, we generate samples of rewards from a Gaussian distribution with means and variances set to the computed values.

%
%

%

%

%

%
%
%
%
%
%
 
%
%
%
%
%
%

%
%

%
%

%
%

%
%

%
%

%

%
%

%
%
%

%
%

%
%
%

%

%

%

\end{document}